\documentclass[runningheads]{llncs}

\usepackage{eccv}

\usepackage{eccvabbrv}

\usepackage{graphicx}
\usepackage{booktabs}
\usepackage{mathdots}
\usepackage{amsmath}

\usepackage[accsupp]{axessibility}  

\usepackage{hyperref}

\usepackage{orcidlink}

\usepackage{xcolor}
\usepackage{amsmath}
\usepackage{amsthm}
\usepackage{amsfonts}
\usepackage{amssymb}
\usepackage{stmaryrd}
\usepackage{graphicx} 

\usepackage{tikz}
\usepackage{pgfplots}
\definecolor{VeryPink}{rgb}{1.0, 0.0, 0.706}

\newcommand{\RR}{\mathbb{R}}
\newcommand{\CC}{\mathbb{C}}
\newcommand{\PP}{\mathbb{P}}
\renewcommand{\P}{\mathcal{P}}

\theoremstyle{definition}

\begin{document}

\title{Single-View Rolling-Shutter SfM} 


\author{Sofía Errázuriz Muñoz\inst{1}\orcidlink{0009-0007-1781-7489} \and
Kim Kiehn\inst{1}\orcidlink{0009-0001-0400-483X} \and
Petr Hruby\inst{1}\orcidlink{0009-0004-0344-3330} \and
Kathlén Kohn\inst{1,2}\orcidlink{0000-0002-4627-8812} }

\authorrunning{S. Errázuriz Muñoz et al.}



\institute{KTH Royal Institute of Technology, Sweden\and
Digital Futures, Sweden\\
\email{sofiaerr@kth.se}
}
\maketitle

\begin{abstract}
  Rolling-shutter (RS) cameras are ubiquitous, but RS SfM (structure-from-motion) has not been fully solved yet. This work suggests an approach to remedy this: We characterize RS single-view geometry of observed world points or lines. Exploiting this geometry, we describe which motion and scene parameters can be recovered from a single RS image and systematically derive minimal reconstruction problems. We evaluate several representative cases with proof-of-concept solvers, highlighting both feasibility and practical limitations. 
\end{abstract}

\vspace{-1cm}

\begin{figure}
    \centering
    \begin{tabular}{c c c}
     \includegraphics[width=0.251\linewidth]{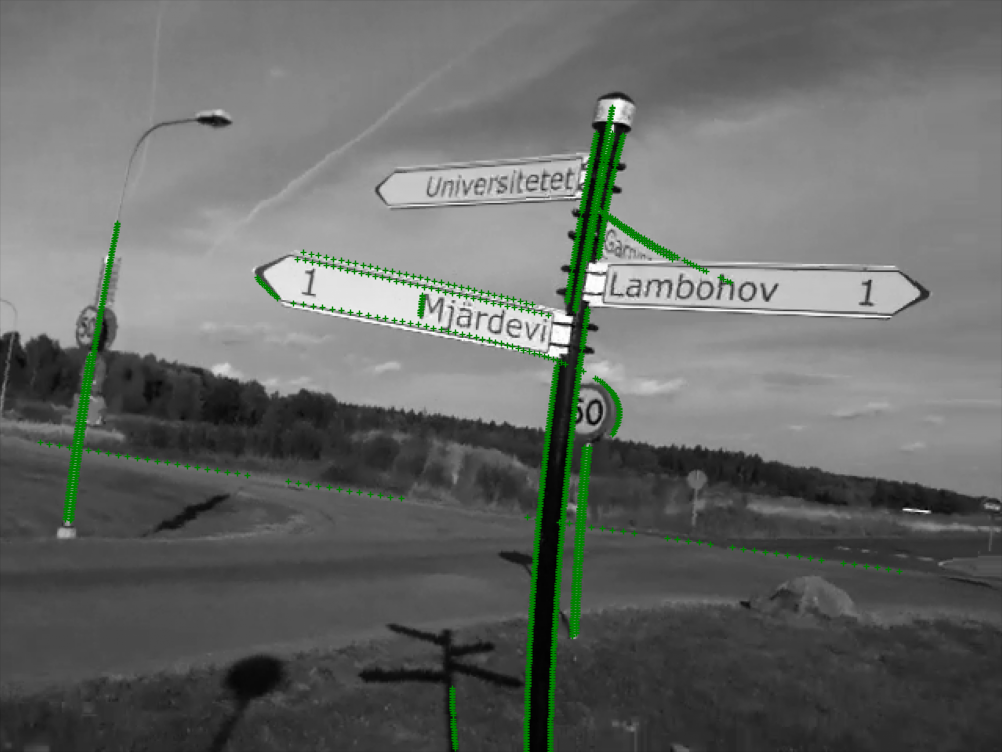} &
       \includegraphics[width=0.335\linewidth]{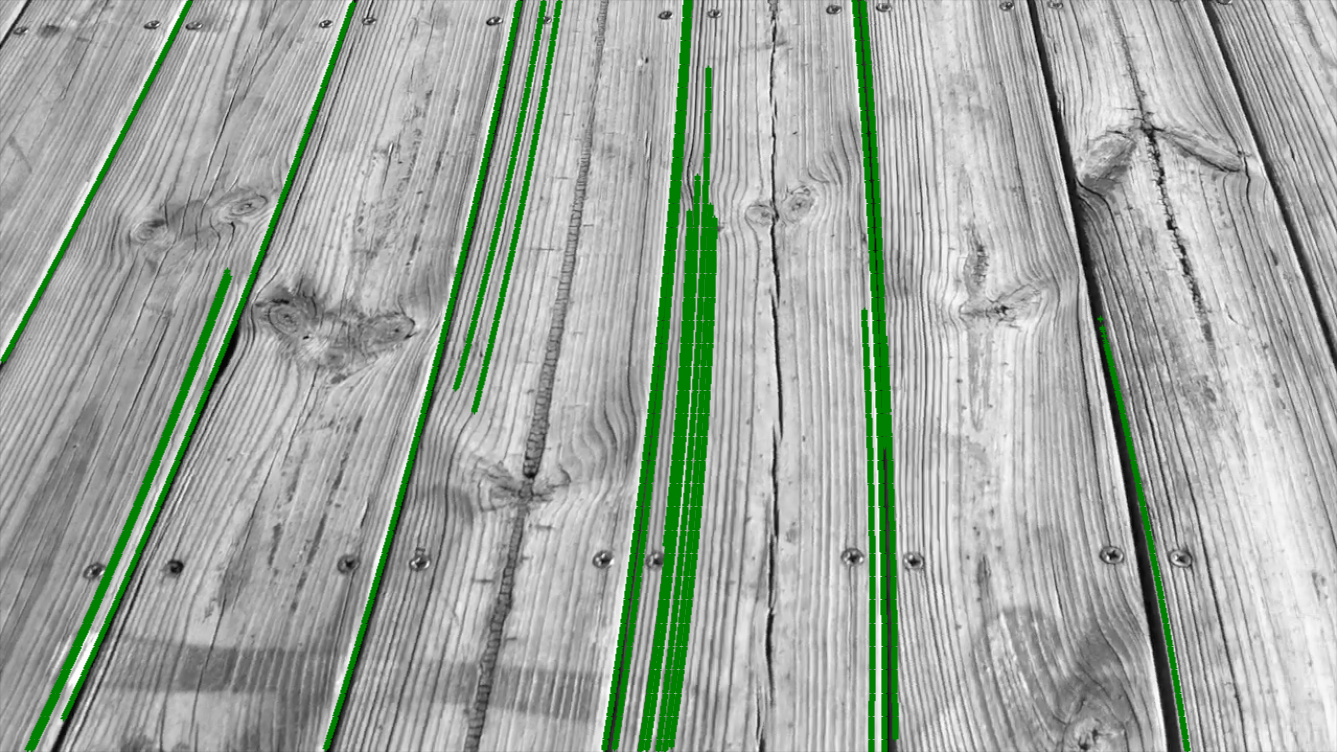}  &
       \includegraphics[width=0.35\linewidth]{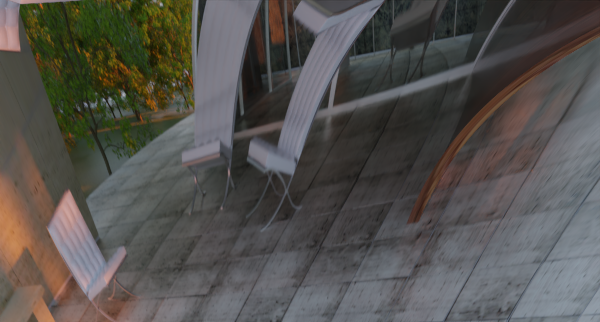}\\
    \end{tabular}
    
    \caption{Examples of RS images: \textit{left:} iPhone 3GS sequence from \cite{DBLP:conf/cvpr/ForssenR10}, \textit{middle}: sequence 06 from \cite{DBLP:conf/cvpr/HedborgFFR12}, \textit{right:} synthetic image generated by \cite{DBLP:conf/eccv/SeiskariYKRTKRS24}.}
    \label{fig:teaser}
\end{figure}

\vspace{-1.1cm}

\section{Introduction} \label{sec:intro}
Rolling shutter (RS) cameras \cite{DBLP:journals/corr/abs-cs-0503076} lead consumer and smartphone markets due to low cost, high resolution, and fast frame rates.
In contrast to global shutter (GS) sensors, RS cameras capture images line by line, yielding image distortions if the camera is moving during capture: 
1) RS images can show the same world point several times;
2) in general, RS cameras map straight world lines to non-linear image curves; see Fig.~\ref{fig:teaser}.
These distortions  turn Structure from Motion (SfM) and camera pose estimation into difficult problems.
Yet, efficient solvers for both problems are desired and essential building blocks in many applications involving moving cameras, e.g., assisted driving, augmented reality, self-navigating~robots.

There are many well-established solvers for SfM with GS cameras \cite{DBLP:conf/cvpr/NisterNB04, DBLP:journals/ivc/SteweniusNKS08, STEWENIUS2006284, Hartley1997}.
However, due to the RS image distortions described above, those solvers do not give accurate results when applied to moving RS cameras with non-negligible movement.
Efficient solvers for moving RS cameras exist only in restricted scenarios, \eg, 
\cite{DBLP:journals/ijautcomp/FanDH23,DBLP:conf/iccv/SaurerKBP13,9020067,albl2020two,DBLP:conf/cvpr/DaiLK16} (or for absolute pose \cite{albl2015r6p,magerand2012global,saurer2015minimal,albl2016rolling,albl2019rolling,kukelova2020minimal,kukelova2018linear,bapat2018rolling, hedborg2011structure}).

\smallskip \noindent \textbf{Main results.}
This article suggests an approach towards general RS SfM solvers, by systematically exploring single-view RS SfM for scenes consisting of points or lines. We consider fully-calibrated RS cameras without radial distortion, whose center and orientation vary polynomially during image capture (using Cayley's parametrization for rotations). Our results are threefold:

1) We provide \emph{foundational theory describing general RS image capture} (Sec. \ref{sec:images}). 
While the RS image distortions of points appearing several times and lines being shown as non-linear curves are well-known phenomena, we \emph{formally characterize} them for arbitrary degrees of the camera  motion model: 
On the one hand, we prove (to the best of our knowledge, for the first time) how often a world point is typically projected by the camera.
This is a fundamental invariant of the camera, referred to as its \emph{order} \cite{hahn2024order}. 
On the other hand, we determine the type of plane curves that arise as images of world lines, and we give computational evidence that our constraints on these curves are the only linear ones.

2) From our RS image theory, \emph{we derive which camera and 3D scene information is -- in principle -- possible to reconstruct from a single RS image}, exploiting the RS distortion of world lines into non-linear image curves (Sec. \ref{sec:lines}) or the multiple projections of world points (Sec. \ref{sec:points}). 
Along the way, we systematically enumerate minimal reconstruction problems (collected in SM Table \ref{tab:minimal}) and describe RS analogs to the classical essential matrix of a GS-camera pair.

3) We explore practical limitations and capabilities of the theoretical solvers from 2),
leading to \emph{proofs of concepts for practical solvers} (Sec. \ref{sec:experiments}).

\smallskip\noindent
\textbf{Related work.}\quad
Most methods for RS relative pose assume a multi-view setting. As the general formulation is challenging and largely unsolved, existing methods rely on special assumptions, e.g., consecutive frames \cite{DBLP:conf/iccv/SaurerKBP13, DBLP:conf/iccv/ZhuangCL17, DBLP:conf/cvpr/GaoGS0K24, DBLP:conf/cvpr/0001GLK25, linear_pure_trans, hruby2026asynchronous}, planar scenes \cite{9020067}, IMU \cite{DBLP:journals/corr/abs-1712-00184,DBLP:journals/corr/abs-1904-06770,hruby2025single},  special cameras \cite{DBLP:conf/icip/WangFD20, DBLP:journals/ivc/FanDZW22, DBLP:conf/wacv/McGriffMAD24}. Moreover, \cite{DBLP:journals/corr/LeeY17a} estimates  RS relative pose with a standard GS solver and refines the pose with optimization, \cite{DBLP:conf/cvpr/DaiLK16} studies non-minimal problems for RS relative pose with pure translation, and \cite{hahn2024order} studies RS cameras of order 1 and lists minimal problems for linear such cameras.

Alternatively, 3D structure and camera motion can be estimated from a single image. This was first explored by \cite{rengarajan2016bows,purkait2017rolling,lao2018robustified} in the context of RS rectification. These papers assume pure rotation, and use the information about the line curvature to recover angular velocity. While \cite{rengarajan2016bows,purkait2017rolling} are optimization-based, \cite{lao2018robustified} uses algebraic constraints from image curves to design a non-minimal solver.
\cite{DBLP:conf/wacv/PurkaitZ18} uses a similar approach to address a special motion and a known IMU.
A recent work~\cite{hruby2025single} is theoretically capable of single-view pose estimation but, in practice, relies on an IMU, and performs best in the multi-view setting.
The solvers in \cite{linear_pure_trans, hruby2026asynchronous} are capable of estimating the pose from  points that have been projected multiple times, although originally designed for the multi-view case.

The existing methods differ in the way they model camera motion. \cite{lao2018robustified, hruby2026asynchronous} replace exact rotation with an approximation, leading to simpler minimal problems. 
\cite{hruby2026asynchronous} estimates the pose without an explicit model, representing the trajectory accurately, but the resulting pose estimation is challenging. \cite{rengarajan2016bows} models motion with splines, which are accurate but non-polynomial and thus not suitable for algebraic solvers. We use the Cayley parameterization, similarly to \cite{hahn2024order}, which is an exact and polynomial formulation suitable for algebraic solvers.

\smallskip \noindent
\textbf{Notation and concepts.}
We consider cameras that, given scenes in projective 3-space $\PP^3$, project them on a projective image plane $\PP^2$.
We write homogeneous vectors $v \in \PP^n$  as $v = [v_1:v_2:\dots:v_n:v_0]^\top$ with $v_0$ being the homogenization variable.
$\RR[v]_d$ denotes the vector space of polynomials that are homogeneous of degree $d$ in the entries of $v$.
It is isomorphic to the vector space  $\RR[v_1,\ldots,v_n]_{\leq d}$  of polynomials of degree at most $d$ in the variables $v_1, \ldots, v_n$. We identify planes  $\Sigma$ in $\PP^3$ with points $\Sigma^\vee$  in the dual space $(\PP^3)^*$.
The Grassmannian $\mathrm{Gr}(1, \PP^3)$ is the set of lines in $\PP^3$.
An \emph{algebraic variety} is a solution set of polynomial equations.
The \emph{Zariski closure} of a set is the smallest variety containing it.
Additional concepts and lemmas from algebraic geometry are in SM Sec. \ref{sec:AG}.

A \emph{minimal problem} is a reconstruction problem that, given generic (e.g., sufficiently noisy) data, has a finite and positive number of solutions (over the complex numbers $\mathbb{C}$).
For an algebraic minimal problem, the number of solutions is the same for almost all input data. 
That number is a measurement for the algebraic complexity of the minimal problem,  called its \emph{degree}.
We give a systematic overview of minimal problems in RS SfM from world points or lines.
To find all minimal problems in a given setting, we follow the ideas from \cite{DBLP:journals/pami/DuffKLP24}:
Every minimal problem must be \emph{balanced}, which means that the degrees of freedom (DoF) of the parameters to reconstruct must equal the number of independent constraints imposed by the generic data.
Hence, we first narrow down the list of potentially minimal problems by characterizing all balanced ones, and afterwards determine which of those are indeed minimal. For a more formal overview of this strategy, see SM Section \ref{sec:AG}.
Most proofs are in SM Section \ref{sec:Proofs}.

\section{Camera Model}
A camera equipped with a rolling shutter (RS) sensor captures a picture by scanning the projection plane along a pencil of parallel image lines. In this article, we assume that these rolling lines are parallel to the $y$-axis and scan with constant speed. 
We denote by $C(x) \in \mathbb{R}^3$
the center of the camera
and by $R(x) \in \mathrm{SO}(3)$ its orientation
when the scanning line is at position $x \in \mathbb{R}$ on the projection plane. We will often simply call $x$ itself the scanline, or view $x=[v:t]$ as a point in a 1D projective space with default affine chart  $t = 1$.
We assume that $C(x)$ is a polynomial function in $x$, and that $R(x) = q2r(A(x))$, where $A(x) \in \RR^3$ is a polynomial function in $x$ and $q2r: \RR^3 \to \mathrm{SO}(3)$ is the Cayley transform, which assigns to a quaternion $(\alpha,\beta,\gamma,1)$ the rotation matrix
\[
\frac{1}{1 + \alpha^2 + \beta^2 + \gamma^2}
\left[\begin{smallmatrix}  
1 + \alpha^2 - \beta^2 - \gamma^2 &
2(\alpha \beta - \gamma) &
2(\alpha \gamma + \beta) \\
2(\alpha \beta + \gamma) &
1 - \alpha^2 + \beta^2 - \gamma^2 &
2(\beta \gamma - \alpha) \\
2(\alpha \gamma - \beta) &
2(\beta \gamma + \alpha) &
1 - \alpha^2 - \beta^2 + \gamma^2
\end{smallmatrix}\right].
\]
Thus, the \emph{space of RS cameras with degree-$d$ motion and degree-$\delta$ quaternion}~is
\begin{equation} \label{eq:cameraModel}
    \mathcal{P}_{d,\delta} := \{ (C,A) \in (\RR[x]_{\leq d})^3 \times (\RR[x]_{\leq \delta})^3\}.
\end{equation}

The camera matrix for the scanline at position $x$ is $P(x) = R(x) \begin{bmatrix} I | -C(x) \end{bmatrix}$.
Since taking a picture with that camera matrix is a linear map $\PP^3 \to \PP^2$ between projective spaces, we may omit the normalizing factor  $\frac{1}{1 + \alpha^2 + \beta^2 + \gamma^2}$ in the Cayley transform,  although technically without the normalization the image of $q2r(A(x))$ is not an element of $ \mathrm{SO}(3)$.

 The world points observed by the scanline at position $x$ form a plane $\Sigma(x)$. 
To compute that plane, we represent the scanline in the image plane $\PP^2$ as a point in the dual projective plane $(\PP^2)^\ast$, i.e., via the coefficient vector  $r(x) = \begin{bmatrix} t & 0 & -v \end{bmatrix}$ of its defining equation (see SM Example \ref{ex:scanlineDual}). Then, the coefficient vector of the defining equation of the plane $\Sigma(x) \subseteq \PP^3$ is
$\Sigma^\vee(x) = r(x) P(x) \in (\PP^3)^\ast$.

 \smallskip
\noindent
\textbf{Taking pictures of a world point $X \in \PP^3$.}
For any fixed RS camera, almost every world point $X$ appears the same number of times on the complex projective projection plane. 
That number is the \emph{order} of the camera \cite{hahn2024order}.
It equals the number of scanlines $x \in \PP^1$ satisfying $X \in \Sigma(x)$.  The latter condition holds iff
\begin{equation}\label{eq:orderPoint}
    0=\Sigma^\vee(x) \cdot  X = r(x) P(x) X.
\end{equation}
To take a picture of any point $X\in \PP^3$, we can first determine the scanlines $x_i$ observing $X$ by solving \eqref{eq:orderPoint} and afterwards compute the image points $P(x_i)X$.

\smallskip
\noindent
\textbf{Taking pictures of a world line $L \subseteq \PP^3$.}
 Let $L_1, L_2$ be distinct points on $L$. To obtain the projection $u(x)$ of $L$ onto the scanline $x$ with pose $R(x), C(x)$, we first compute the dual vector $l(x) \in (\PP^2)^*$ of the image line that we obtain from $L$ via the  GS camera with camera matrix $P(x)$, i.e.
   $l(x) = (R(x) (L_1 - C(x))) \times (R(x) (L_2 - C(x))) = R(x) (L_1 \times L_2 + (L_2 - L_1) \times C(x)).$
    Recall that the Pl\"ucker coordinates $[\Delta: q]$ of $L$ are equal to $\Delta = L_2 - L_1$ and $q = L_1 \times L_2$, and so we have $l(x) = R(x) (q + \Delta \times C(x))$.
Now, the projection $u(x) \in \PP^2$ of the line $L$ onto the scanline $x$ is the intersection of the scanline with the GS image line:
    \begin{equation}
        \label{eq:lineImageAtTimeX}
        u(x) = r(x) \times l(x) = r(x) \times ( R(x) (q + \Delta \times C(x)) ).
    \end{equation}

Finally, we derive constraints imposed by image points that are known to come from a world line $L$. Fix the point $u_1 = [x_1 \ y_1 \ 1]^\top \in \PP^2$. If this point is a projection of some point on $L$, it has to lie on $l(x_1)$, giving the  constraint
$
    u_1^\top R(x_1) (q + \Delta \times C(x_1))=0.$
Taking $q = L_1 \times \Delta$, this can take the form
\begin{equation}
    \label{eq:point_line_constraint}
    u_1^\top R(x_1) [\Delta]_{\times} (-L_1 + C(x_1))=0,
\end{equation}
which corresponds to the constraint used in \cite{hruby2025single}. Omitting the normalization factor $\frac{1}{1 + \alpha^2 + \beta^2 + \gamma^2}$ as discussed above yields a polynomial constraint.

In the following chapters, we are going to use these formulas to study the images of points and lines taken by a RS camera, and to develop methods for estimating the motion parameters $C(x), A(x)$ from point and line projections.

\section{RS images}\label{sec:images}
We begin by describing how the images of 3D scenes consisting of points and lines taken by RS cameras are expected to look. 
In Theorem \ref{thm:order}, we determine how often a RS camera typically observes a world point (over $\mathbb{C}$ and on the whole image plane). 
Although world points are generally projected more than once on the same RS image, the image of a world line is a single irreducible curve. 
Theorem \ref{thm:imageCurve} describes what those curves typically look like.

\begin{theorem}\label{thm:order}
    The order of almost all RS cameras in $\P_{d,\delta}$  is 
    $1 + d + 2 \delta$.
\end{theorem}

\begin{proof}
We homogenize \eqref{eq:orderPoint} into a homogeneous polynomial in $[v:t]$. The order of the RS camera is equal to the degree of that polynomial for generic $X$, as the images of the point $X$ are given as $P(x_i)X$, where the $x_i$ are the solutions to \eqref{eq:orderPoint}.
The degree of the polynomial is $\deg \Sigma^\vee \leq \deg r + \deg C + \deg R = 1 + d + 2 \delta$. It remains to show that, for a generic camera, the 3 entries of $r(x)P(x)$ do not have any common factor, since otherwise the actual degree of $\Sigma^\vee$ would be lower. 
In SM Section \ref{sec:Proofs}, we provide these remaining details.
\qed
\end{proof}

\begin{theorem}\label{thm:imageCurve}
    For almost all RS cameras in $\P_{d,\delta}$  and almost all world lines $L$, the image of $L$  is a rational irreducible curve of degree $1 + d + 2 \delta$ that passes $d + 2 \delta$ many times through the point at infinity at the end of the y-axis. 
\end{theorem}
\begin{proof}
    Let $L$ be a generic 3D line with Pl\"ucker coordinates $[\Delta: q]$. 
    By \eqref{eq:lineImageAtTimeX},  
     the image of the line $L$ by the RS camera is an algebraic curve with parameterization $
        \mathbb{P}^1 \to \PP^2,  
        x \mapsto u(x) = \begin{bmatrix}
            u_1(x) :u_2(x):u_0(x)
        \end{bmatrix}^\top$. 
    Since the image curve can be parameterized by $\PP^1$, it is an irreducible rational curve.

    From \eqref{eq:lineImageAtTimeX} in homogeneous coordinates $x=[v:t]$, we see that the entries $u_i(x)$ are polynomials of degree $1+d+2\delta$, and that $u_0(x)$ has $t$ as a factor, i.e., $u_0(x) = t \cdot \tilde{u}$ for some polynomial $\tilde{u}$ of degree $d+2 \delta$.
    Moreover, $u_1(x) = v \cdot \tilde u$, verifying that the x-coordinate $\frac{u_1(x)}{u_0(x)}$ of the point $u(x)$ is indeed $ \frac{v}{t}$ (which equals $x \in \RR$ in the affine chart $t=1$, as desired).
    The y-coordinate of $u(x)$ is y$ = \frac{u_2(x)}{u_0(x)}$.

    As in the proof of Theorem~\ref{thm:order}, it remains to show that $u_1,u_2,u_0$ do not have any common factor for generic RS cameras. We show this in SM Section \ref{sec:Proofs}.
    This has two consequences: First, since the map $x \mapsto u(x)$ is almost everywhere invertible (more concretely, over all finite points), the degree of the image curve is indeed $\deg u_2 = 1+d+2\delta$.
    Second, the curve passes $\deg \tilde u = d + 2 \delta$ many times through the point $[0:1:0]$ at infinity, namely once for each root of $\tilde u$.
    \qed
\end{proof}

\section{Line SfM} \label{sec:lines}

In this section, we analyze  single-view SfM for RS cameras that observe world lines. 
Following Theorem \ref{thm:imageCurve}, we denote by $\mathcal{H}_D$ the set of rational curves in $\PP^2$ that have degree $D$ and pass $D-1$ many times through the point $[0:1:0]$.
These curves are precisely those that can be parametrized as 
\begin{equation} \label{eq:projectiveCurveParam}
    [v:t] \mapsto [\ell_1(v:t) f(v:t) \;: \; g(v:t) \;:\; \ell_2(v:t) f(v:t)],
\end{equation}
where the $\ell_i$ are linearly independent linear forms encoding the relation between the scanline coordinate and the time unit, $\deg f = D-1$, and $\deg g = D$. 
By acting with automorphisms of $\PP^1$, we may assume without loss of generality that $\ell_1(v:t) = v$ and $\ell_2(v:t) = t$, rescaling time to make it equal to the scanline coordinate.
Thus, in the affine charts where $t=1$ and the last coordinate of \eqref{eq:projectiveCurveParam} is $1$, the curve, denoted as $u(x)$, is parametrized by
\begin{equation}\label{eq:AffineCurveParam}
    x \mapsto \left(x, \frac{g(x)}{f(x)}\right);
\end{equation}
cf. also the proof of Theorem \ref{thm:imageCurve}.
Recall that projective planar curves can be identified with their defining polynomial (up to scaling), so that the set of degree-$D$ curves becomes a Zariski dense subset of the projective space
$\PP (\CC[\text{x,y,z}]_D)$.

\begin{proposition} \label{prop:H}
    $\dim \mathcal{H}_D = 2D$.
    Moreover, the Zariski closure of $\mathcal{H}_D$ inside $\PP (\CC[\,\text{x,y,z}\,]_D)$ is a projective subspace (i.e., it is defined by linear constraints). 
\end{proposition}
\begin{proof}
    The polynomials $g$ and $f$ in \eqref{eq:AffineCurveParam} have $D+1$ and $D$ coefficients, respectively. 
    The only way how different polynomials parametrize the same curve is if they differ from $g$ and $f$ by a common scalar factor. This proves the dimension claim.
    Moreover, in the affine chart $\text{z}=1$, \eqref{eq:AffineCurveParam} implies that the curves in $\mathcal{H}_D$ are precisely those whose defining equation is of the form $\text{y} \cdot f(\text{x}) = g(\text{x})$.
    These are exactly those equations where monomials of degree larger than $1$ in y do not appear. \qed
\end{proof}

\noindent
We expect the image curves to not satisfy any further general linear constraints (in addition to being contained in $\mathcal{H}_D$). We  verified this computationally for  small choices of camera hyperparameters and number of lines (see SM):
\begin{proposition} \label{prop:spanImLine}
    Let  $d\leq 5 ,\delta \leq 5,\ell \leq 6$.
    The linear span of the set of images obtained by some RS camera in $\mathcal{P}_{d,\delta}$ observing $\ell$ world lines equals $\mathcal{H}_{1+d+2\delta}^\ell$. 
\end{proposition}

 \noindent
 SfM reconstruction is possible only up to global rotation, translation, and scaling of 3-space \cite[Rem. 1]{hahn2024order}.
Thus, we may assume  for the scanline at position $x=0$, that the camera center is the origin (i.e., $C(0) = 0$) with identity rotation (i.e., $A(0) = 0$) and  one of the remaining coefficients of $C(x)$ is fixed to be $1$ (whenever $d>0$).
We hope to recover the remaining $3 \delta + 3d -1$ camera parameters.

However, from the image of a single line we cannot reconstruct the part of the camera motion in the same direction as the line. More generally:
\begin{proposition} \label{prop:parallel}
    For a 3D scene  of parallel lines with direction vector~$\Delta$, two RS cameras with center motions $C_1, C_2 \in (\RR[x]_{\leq d})^3$ satisfying $\Delta \times C_1  = \Delta \times C_2$ and
    with the same rotation $A_1 = A_2 \in (\RR[x]_{\leq \delta})^3$ produce the same image.
    Thus, there are additional $d$ camera parameters that  cannot be recovered from the image. 
\end{proposition}
\begin{proof}
    The first claim follows directly from \eqref{eq:lineImageAtTimeX}, which shows that the only part of the center motion $C(x)$ that affects the image of a line is the cross product of the line's direction vector with $C(x)$.
    For the parameter count, we may rotate  scene and  camera such that $\Delta = [1\,0\,0]^\top$ and translate until $C(0)=[0\,0\,0]^\top$.
    Then, $\Delta \times C(x)$ depends only on the second and third entry of $C(x)$, and so the $d$ coefficients of the polynomial in the first entry of $C(x)$ cannot be recovered. \qed
\end{proof}

\subsection{Pure rotation} \label{sec:pure_rotation}
Here, we consider the case $d=0$ and $\delta>0$. 
Henceforth, all proofs are in SM Sec. \ref{sec:Proofs}.
After translating, we assume $C(x) = [0\,0\,0]^\top$.
Then,  \eqref{eq:lineImageAtTimeX} implies that the image of a line $L$ depends only on the plane spanned by $L$ and the camera center.

\begin{proposition} \label{prop:D0ImageCurvePlane}
Fix a generic RS camera in $\P_{0,\delta}$. The set of image curves of world lines is a two-dimensional plane in $\mathcal{H}_{1+2\delta}$. 
That plane uniquely determines the camera (up to global rotation and translation).
\end{proposition}

\noindent
Note that this plane encodes a single RS camera observing lines, in the same way as an essential matrix encodes two global shutter cameras observing points.

\begin{example}[$\delta=1$] \label{ex:curve_d0del1}
For a RS camera centered at $[0\,0\,0]^\top $ and with rotation given by $R(x) = q2r(\alpha x,\beta x,\gamma x)$, the image of a world line is a curve of the form 
\begin{equation}\label{eq:curveD0Delta1}
        y = \frac{Z_3x^3+Z_2x^2+Z_1x+Z_0}{Y_2x^2+Y_1x+Y_0},
    \end{equation}
    whose coefficient vector $(Z_3,\ldots,Y_0)$ is in the row span of the matrix
    \begin{align}
        \left[\begin{smallmatrix}
            \frac{\alpha^3-\alpha\beta^2+\alpha\gamma^2}{c} & \alpha^2-\beta^2+\gamma^2+2\beta
            & \frac{-2\alpha\beta+\alpha}{\gamma} & 1 &
            \frac{-2\alpha^2\beta-2\beta\gamma^2}{\gamma} & 0 & 0 \\
            \frac{-\alpha^2+\beta^2+\gamma^2}{2\gamma} & -\alpha & \frac{2\beta-1}{2\gamma} & 0 & \frac{\alpha\beta}{\gamma} & 1 & 0 \\
            -2\alpha\beta & -2\beta\gamma+2\gamma & -2\alpha & 0 & -\alpha^2+\beta^2-\gamma^2 & 0 & 1
        \end{smallmatrix}\right]. \label{eq:matrixRepresentingPlane}
    \end{align}
    This matrix is a unique representation of the plane in $\mathcal{H}_3$ spanned by its rows. 
    Therefore, given the plane, we see from \eqref{eq:matrixRepresentingPlane} that the parameters $\alpha,\beta,\gamma$ are uniquely determined.
    Since the plane encodes all image curves of lines observed by the fixed camera above, \eqref{eq:matrixRepresentingPlane} is the analog of the essential matrix of a pair of GS cameras observing points.
    
    The union over all such planes (for varying $\alpha,\beta,\gamma$) is the Zariski closure of the set of curves that are images of lines of some RS camera in $\P_{0,1}$. This union is a hypersurface in $\mathcal{H}_{3}$. 
    It consists of those curves \eqref{eq:curveD0Delta1}
    whose coefficients satisfy one homogeneous polynomial equation $IC(Z_0,Z_1,Z_2,Z_3,Y_0,Y_1,Y_2)=0$ of degree $10$, which is shown in SM Section~\ref{sec:longEquation} (in the affine chart $Z_0 = -1$).   
    \hfill $\diamondsuit$
\end{example}

\begin{proposition} \label{prop:D0ImageCurveVariety} Let $1 \leq \delta \leq 3$.
    For varying RS cameras in $\P_{0,\delta}$ and world lines, the set of image curves has dimension $3\delta + 2$. A generic such image curve uniquely determines the camera in $\P_{0,\delta}$ (up to global rotation and translation) and the world plane through the camera center producing the image curve. 
\end{proposition}

\noindent
\textbf{Practical reconstruction.}\quad
Now, we discuss practical methods for recovering the pose $A(x)$ from projections of lines into a single image for $\delta=1$. The space $\mathcal{H}_3$ is six-dimensional, so a curve $u(x) \in \mathcal{H}_3$ can be uniquely determined by six points on it, using \eqref{eq:curveD0Delta1}. By \Cref{prop:D0ImageCurveVariety}, if a projection curve $u(x)$ is measured perfectly, it satisfies the internal constraint $IC$ in \Cref{ex:curve_d0del1} and uniquely determines the pose $A(x)$ and the world plane~$q$.

In the presence of noise, however, the measured curve generically deviates from the form of \Cref{ex:curve_d0del1}, and cannot be directly used to recover the pose. Instead, we can sample five points $u_i \in \RR^2, i \in \{1,\ldots,5\}$ on $u(x)$, fixing five DoF of the curve. The last DoF is constrained by $IC$. Since $IC$ has degree 10, there are generically 10 (complex) curves of the form in \Cref{ex:curve_d0del1} passing through $u_1,\ldots,u_5$, each corresponding to a unique  pair $(A(x), q)$. This gives 10 possible 3D configurations $(A(x), q)$ for given points $u_1,...,u_5$.

Instead of finding the curve coefficients $Z_0,\ldots,Z_3,Y_0,\ldots,Y_2$, we
parameterize this problem directly by collecting equations of the form \eqref{eq:point_line_constraint} for each $u_i \in \RR^2, i \in \{1,\ldots,5\}$.
Using Gr\"obner bases (GB), we verified that these equations generate a problem with 10 solutions, and we created a solver for this problem using a homotopy continuation method  MiNuS \cite{fabbri2020minus}.

However, experimental evaluation (see Sec. \ref{sec:experiments}) shows that, while this solver works well on noise-free synthetic data, it deteriorates under noise, likely due to the combination of low angular speed and short observed curve segments, where high-frequency features of the curves are dominated by noise. To remedy this, we consider reconstruction problems that limit the number of constraints per line to 3 or 4. 
While these problems require more lines to estimate the motion, they do not rely on high angular velocity and long image curves.

Let $\ell$ denote the number of lines and $N_i$ denote the number of points on the image of line $i$. These points $u_{i,j}$ (for $i \in \{1,...,\ell\}, j \in\{1,...,N_i\}$) give $\sum_{i=1}^\ell N_i$ constraints in the form \eqref{eq:curveD0Delta1}. Comparing the number of constraints to the number of unknowns, with $3$ DoF for the angular velocity and $2$ DoF for the world plane of each line, we see that the balanced problems are those that satisfy
\begin{equation} \label{eq:balancedPureRot}
    3+2\cdot \ell = \textstyle{\sum_{i=1}^\ell} N_i.
\end{equation}
Note that $N_i>2$ is needed for the points on the $i$-th image curve to constrain the camera parameters.
Then, \eqref{eq:balancedPureRot} holds for $\ell=1,N_1=5$ (as discussed above), or
$\ell=2, N=(4,3)$, or 
$\ell=3,N=(3,3,3)$. 
GB verification confirms that both latter problems are indeed minimal, with $54$ solutions for $\ell=3$ and $30$ solutions for $\ell=2$. We have again constructed the minimal solvers using MiNuS \cite{fabbri2020minus}. These solvers show improved noise robustness compared to the single-line solver (see Sec. \ref{sec:experiments}).
All three minimal problems discussed  can be found in SM Table~\ref{tab:minimal}.

\subsection{Pure center motion} \label{sec:pureTranslation}
Here, we assume $\delta=0, d>0$ and, after global rotation, that $R(x)=I$.

\begin{proposition} \label{prop:Delta0ImageCurveFixedCamera}
Fix a generic RS camera in $\P_{d,0}$.
For $d>1$, the camera is a linear isomorphism between $\mathrm{Gr}(1, \mathbb{P}^3)$ and the subvariety of $\mathcal{H}_{1+d}$ that consists of all image curves of world lines; 
 that subvariety uniquely determines the camera (up to global rotation, translation, scaling).
For $d=1$, the image curves in $\mathcal{H}_2$ are generically in one-to-one correspondence with  world~lines.
\end{proposition}

\noindent
In particular, in the case $d=1$, we can easily triangulate lines from image curves. 
However, in that case, line images do not impose any constraints on the camera.
Similarly to the previous section, for $d>1$, the $\mathrm{Gr}(1,\mathbb{P}^3)$-isomorphic subvariety of $\mathcal{H}_{1+d}$ plays the analog role to the essential matrix.

\begin{example}[$d=1$] \label{ex:D1Delta0}
    We fix a non-rotating camera with  center motion $C(x) = [a_{1}\,b_{1}\,c_{1}]^\top x$ such that $c_{1}\neq 0$. Given the Pl\"ucker coordinates $[\Delta: q]$ of a world line, the coefficients $Y_{0},Y_{1},Z_{0},Z_{1},Z_{2}$ of the image curve $y = \frac{Z_{0}+Z_{1}x+Z_{2}x^2}{Y_{0}+Y_{1}x}$ are \begin{equation}
    \label{eq:linearCamMapD1Delta0}
     \left[ \begin{smallmatrix}
          0 & 1 & 0 & 0 & 0 & 0 \\
          0 & 0 & 0 & -c_{1} & 0 & a_{1} \\
          0 & 0 & -1 & 0 & 0 & 0 \\
          -1 & 0 & 0 & -b_{1} & a_{1} & 0 \\
          0 & 0 & 0 & 0 & -c_{1} & b_{1}
        \end{smallmatrix} \right] \cdot \left[ \begin{smallmatrix}
            q \\ \Delta
        \end{smallmatrix} \right].
    \end{equation}
    Conversely, given a generic coefficient vector, we can uniquely triangulate the world line; its Pl\"ucker coordinates are
    {\small \begin{gather*}
 q_1 = \frac{-a_{1}Z_{2}+b_{1}Y_{1}-c_{1}Z_{1}}{c_{1}}, \,
 \Delta_1 = \frac{-a_{1}Y_{0}Z_{2}+b_{1}Y_{0}Y_{1}-c_{1}Y_{1}Z_{0}}{\Xi},\\ 
  q_2 = Y_{0}, \, \Delta_2 = \frac{-a_{1}^2Z_{2}^2+2a_{1}b_{1}Y_{1}Z_{2}-a_{1}c_{1}Z_{1}Z_{2}-b_{1}^2Y_{1}^2+b_{1}c_{1}Y_{1}Z_{1}-c_{1}^2Z_{0}Z_{2}}{c_{1}\Xi}, \\
 q_3 = -Z_{0}, \, \Delta_3 = \frac{a_{1}Y_{1}Z_{2}-b_{1}Y_{1}^2-c_{1}Y_{0}Z_{2}+c_{1}Y_{1}Z_{1}}{\Xi},
    \end{gather*}}
    where $\Xi:=-b_{1}c_{1}Y_0-a_{1}b_{1}Y_1+c_{1}^2Z_0+a_{1}c_{1}Z_1+a_{1}^2Z_2$.

If $c_{1}=0$, the camera motion is parallel to the projection plane. By \cite[Prop. 10]{hahn2024order}, this is equivalent to that the camera sees a generic world point only once (instead of twice).
In this case, the linear map \eqref{eq:linearCamMapD1Delta0} has rank 4 (instead of 5), so images of world lines are not arbitrary curves in $\mathcal{H}_2$. 
In fact, by \cite[Lemma 15]{hahn2024order}, the image curves pass not only through the point $[0:1:0]$ (which is the defining property of $\mathcal{H}_2$) but also through $[a_1:b_1:0]$.
Thus, given a curve in $\mathcal{H}_2$, there are either no or infinitely many triangulating world lines.
\hfill $\diamondsuit$
\end{example}

\begin{example}[$d=2$]
    The set of image curves in $\mathcal{H}_3$ under varying cameras $C(x)=[a_1\,b_1\,c_1]^\top x + [a_2\,b_2\,c_2]^\top x^2$ is defined by one linear and one quadratic polynomial: 
    \vspace{-0.4cm}
    {\small 
    \begin{gather*}
     \left(a_2b_2c_1-b_2c_1^{2}-a_1b_2c_2+b_1c_1c_2\right)Y_0+\left(b_2c_1c_2-b_1c_2^{2}\right)Y_1\\+\left(-a_2^{2}c_1+a_2c_1^{2}+a_1a_2c_2-a_1c_1c_2\right)Z_0+\left(-a_2c_1c_2+a_1c_2^{2}\right)Z_1 =0, \quad \text{ and }
   \\
       \left(a_1a_2b_2-a_1b_2c_1+a_1b_1c_2\right)Y_0^{2}+\left(-a_2^{2}b_2+b_2c_1^{2}-a_2b_1c_2+2\,a_1b_2c_2-b_1c_1c_2\right)Y_0Y_1+\\\left(-a_2b_2c_2-b_2c_1c_2+b_1c_2^{2}\right)Y_1^{2}+b_2c_1c_2Y_0Y_2-b_2c_2^{2}Y_1Y_2+\\\left(-a_1a_2^{2}+a_1a_2c_1-a_1^{2}c_2\right)Y_0Z_0+\left(a_2^{3}-a_2c_1^{2}+a_1c_1c_2\right)Y_1Z_0+\\\left(-a_2c_1c_2+a_1c_2^{2}\right)Y_2Z_0+\left(a_2^{2}c_2+a_2c_1c_2-a_1c_2^{2}\right)Y_1Z_1-\\+ c_2^{3}Y_1Z_3+a_1a_2c_2Y_0Z_1-a_1c_2^{2}Y_0Z_2+a_{2}c_2^{2}Y_1Z_2-c_{1}c_2^{2}Y_0Z_3=0.
    \end{gather*}}  
    
    \noindent
    From the coefficients of the quadratic polynomial, we can recover the camera up to global scaling: Setting $c_2=1$, all other camera parameters are determined. 
    \hfill $\diamondsuit$
\end{example}

\begin{proposition} \label{prop:pureTranslationDominant}
    Almost every curve in $\mathcal{H}_{1+d}$ is the image of some world line observed by some camera in $\mathcal{P}_{d,0}$. 
    In fact, after modding out the global scaling, translation, and rotation, there is a $(d+1)$-dimensional family of world lines and cameras in $\mathcal{P}_{d,0}$ yielding the same generic image curve. 
\end{proposition}

\noindent
Note that out of the ambiguities with $d+1$ DoF in Proposition~\ref{prop:pureTranslationDominant}, $d$ DoF are explained by Proposition~\ref{prop:parallel}, which says that the camera movement in the same direction as the observed line cannot be reconstructed. 

\medskip

\noindent
\textbf{Practical reconstruction.}\quad 
\emph{\underline{For $d=1$}} and a fixed camera, projection curves are generically in one-to-one correspondence to the world lines (\Cref{prop:Delta0ImageCurveFixedCamera}), and impose no constraints on the motion $C(x)$.
However, for lines in a special configuration, \eg parallel or coplanar, the 3D structure has fewer than $4$ DoF per line, and parts of the motion can still be estimated. We describe minimal solvers based on point measurements on image curves, analogously to Section~\ref{sec:pure_rotation}:

\textit{Parallel lines.}\quad The lines share an unknown direction $\Delta$. Assuming $\Delta_1 \neq 0$, each 3D line is parameterized by a point $L_i = [0 \ L_{yi} \ L_{zi}] \in \RR^3$ on the line. The camera center $C(x)$ is $x \cdot [a_1 \ b_1 \ c_1]^\top$. By \Cref{prop:parallel}, any $[a_1 \ b_1 \ c_1]^\top + \alpha \Delta, \alpha \in \RR$ produces the same image. We fix this ambiguity by setting $a_1=0$, and fix the scale with $L_{z1}=1$. With $\ell$ lines, structure and motion have $3 + 2\ell$ DoF.

A projection of line $i$ is parametrized by $N_i \leq 4 = \dim \mathcal{H}_2$ points $u_{i,j}$ (for $j \in\{1,...,N_i\}$, which give $\sum_{i=1}^\ell N_i$ constraints of the form \eqref{eq:curveD0Delta1}. Hence, the balanced problems must satisfy $3 + 2\ell = \sum_{i=1}^\ell N_i$ (as in \eqref{eq:balancedPureRot}), which holds for $\ell=2, N=(4,3)$, or $\ell=3,N=(3,3,3)$. Gr\"obner basis (GB) computations verify that both are minimal with 2 solutions for $\ell=2$ and  $5$ solutions for $\ell=3$.

\textit{Parallel and coplanar lines.}\quad The lines share direction $\Delta$ and lie in a plane $\Pi$. We fix the scale by assuming $P_0:=[0 \ 0 \ 1]^\top \in \Pi$. A basis of $\Pi$ is given by the vectors $\Delta$ and $P_d = [0 \ 1 \ P_{zd}]^\top$, so each 3D line is parameterized by $\lambda_i \in \RR$, such that $L_i = P_0 + \lambda_i P_d$. The velocity is parameterized as above, so 3D structure and motion have $5+\ell$ DoF. With $N_i > 1$ points on the image of line $i$, we get $\sum_{i=1}^\ell N_i$ constraints of the form \eqref{eq:curveD0Delta1}. Assuming $\ell > 2$ (as two parallel lines are always coplanar), the balanced problems satisfying $5+\ell = \sum_{i=1}^\ell N_i$ are 
$N=(4,2,2)$, 
$N=(3,3,2)$,
$N=(3,2,2,2)$, and
$N=(2,2,2,2,2)$.
GB verifies that these problems are minimal with 2, 4, 6, and 10 solutions.

\textit{\underline{For $d \in \{ 2,3\}$}} and generic lines,  projection curves live in $\mathcal{H}_3$ resp. $\mathcal{H}_4$. By \Cref{prop:pureTranslationDominant}, these curves have no internal constraints and can be uniquely determined by $6$ resp. $8$ points. After fixing global scale, rotation and translation, the 3D structure and motion have $5+4\ell$ resp. $8+4 \ell$ DoF.
Sampling $N_i >4$ points on the image of line $i$ (to impose constraints on the camera) yields
$\sum_{i=1}^\ell N_i$ constraints of the form \eqref{eq:curveD0Delta1}. 
For $d=2$, the balanced problems satisfy $5+4 \ell = \sum N_i$, where $4< N_i\leq 6$. These are  
$N=(6,6,5)$, $N=(6,5,5,5)$, and $N=(5,5,5,5,5)$.
HC computation
shows that all are minimal of degree $30$, $131$, and $680$, respectively.
For $d=3$, there are 15 balanced problems (satisfying $8+4 \ell= \sum N_i, 4<N_i \leq 8$); see SM Example~\ref{ex:minProblsD3Delta0}. 
The problem with the lowest number of lines is $\ell=2, N=(8,8)$. It is minimal of degree 25.

All minimal problems discovered here can be found in SM Table \ref{tab:minimal}.
As in Section~\ref{sec:pure_rotation}, for representative examples of those minimal problems, we have constructed solvers  with MiNuS \cite{fabbri2020minus}, which are evaluated in Section~\ref{sec:experiments}.

\vspace{-0.2cm}
\subsection{Center motion and rotation}
In this section, we assume $d >0$ and $\delta>0$.  

\begin{proposition} \label{prop:ImageCurveFixedCamera}
Fix a generic RS camera in $\P_{d,\delta}$.
The camera is a linear isomorphism between $\mathrm{Gr}(1, \mathbb{P}^3)$ and the subvariety of $\mathcal{H}_{1+d+2\delta}$ that consists of all image curves of world lines. 
For $(d,\delta)\in \{(1,1),(1,2),(2,1)\}$, that subvariety uniquely determines the camera (up to global rotation, translation, and scaling).
\end{proposition}

\noindent
As before, the subvariety is the analog to the essential matrix.
We now analyze which image curves can be obtained for varying cameras. Recall from Prop.~\ref{prop:parallel} that the image curve is not affected by those parts of the camera motion that follow the direction of the observed line. We call this the \emph{line direction~ambiguity}.

\begin{proposition} \label{prop:delta1VaryingCam}
    Let $\delta = 1$.
    Almost every curve in $\mathcal{H}_{3+d}$ is the image of some world line observed by some camera in $\mathcal{P}_{d,1}$. In fact, after modding out the global scaling, translation, rotation, and the line direction ambiguity, there are finitely many world lines and cameras in $\mathcal{P}_{d,1}$ yielding the same generic image curve. The degrees of these minimal problems are:\begin{tabular}{c|ccccc}
        $d$  & 1 & 2 & 3 & 4 & 5 \\
        \# solutions  & 20 & 104 & 320 & 760 & 1540 
    \end{tabular}
\end{proposition}

\begin{proposition} \label{prop:delta2VaryingCam}
    Let $\delta=2$ and $1\leq d\leq3$. For varying RS cameras in $\mathcal{P}_{d,\delta}$ and world lines, the set of image curves has dimension $3\delta + 2d + 3$. A generic such image curve uniquely determines the camera in $\mathcal{P}_{d,\delta}$ and the world line, up to global scaling, translation, rotation, and the line direction ambiguity. 
\end{proposition}

\noindent
\textbf{Practical reconstruction.}
For $\delta=1$,  minimal SfM problems for a single world line are listed in Prop. \ref{prop:delta1VaryingCam}.
For $\delta >1$, we can find minimal problems by proceeding analogously to Sec. \ref{sec:pure_rotation} - \ref{sec:pureTranslation}. 
We report them for $\delta=2,d=1$ in SM Example~\ref{ex:minProblsD1Delta2}.

\section{Point SfM} \label{sec:points}
We analyze minimal problems for single-view RS SfM that use point projections to estimate motion parameters.
These problems exploit the higher order of RS cameras. A point $X$ that is observed $k$ times $u_1,\ldots,u_k \in \PP^2$ by the same camera imposes $2k$ constraints of the form $u_i \sim R(x_i) (X - C(x_i))$, where $u_i = [x_i \, y_i \, 1]^\top$, on structure and motion.
However, these constraints are not linearly independent when $k$ equals the order of the camera, several points are observed $k$ times and the camera center is undergoing a non-linear or accelerated motion:
 \begin{lemma}\label{lem:pointRel}
        Assume $d\geq 2$. Define $o := 1+d+2\delta$ and let $(x_i^j, y_i^j)$ for $i=1,...,o$ and $j=1,...,p$ be the images of $p$ world points made by a generic camera in $\P_{d,\delta}$. Then, it holds for all $j,j'$ that $
            \sum_{i=1}^{o}x_i^j=\sum_{i=1}^{o}x_i^{j'}.
        $
    \end{lemma}

\noindent
By this lemma, the set of (complex) images of $p$ world points observed by RS cameras with $d>1$ satisfies $p-1$ linear constraints. 
Analogously to Prop.~\ref{prop:spanImLine} in the case of world lines, we expect there not to be any further linear constraints:
\begin{proposition} \label{prop:spanImagePoints} Let $d \leq 4, \delta \leq 3, p \leq 3$.
    The set of images of $p$ points taken by RS cameras in $\P_{d,\delta}$ satisfy $p-1$ linear constraints if $d \geq 2$ and no linear constraints if $d \leq 1$. 
\end{proposition}

\noindent
Hence, if $p$ points were indeed appearing order-many times on the same RS image, they would impose 
\begin{equation} \label{eq:ptConstraintsNumber}
        \begin{cases}
            2p (1+d+2 \delta) & \text{constraints if }d\leq 1\\
            2p (1+d+2 \delta)-(p-1) &  \text{constraints if }d\geq 2
        \end{cases}
    \end{equation}
on structure and motion.
Now, structure and motion have $3d+3\delta-1+3p$ DoF whenever $d>0$. 
For $d=0$, we cannot reconstruct the world points, but only the lines spanned by them and the camera center (this is analogous to Sec. \ref{sec:pure_rotation} where only the plane spanned by camera center and world line could be recovered); and so structure and motion have $3 \delta +2p$ DoF. 
\begin{theorem} \label{thm:point_minimal_problems}
Consider a generic RS camera in $\mathcal{P}_{d,\delta}$ observing $p$ world points.
The balanced SfM problems using all $1+d+2\delta$ image points per world point are: 
\\$\bullet$ $p=2,d=2,\delta=0$:  this is minimal of degree 1;
\\$\bullet$ $p=2, d=1, \delta=0$:  this is minimal of degree 1;
\\$\bullet$ $p=1, d=\delta$:  minimal for $d \leq 5$, with degrees $48$ for $d=1$ and $9609$ for $d=2$. 
\end{theorem}

\noindent
As it might be relatively rare to observe all order-many image points of the same world point on a given RS image, we next consider the case where $p$ world points are observed (at least) twice each, yielding $4p$ many constraints:
\begin{theorem} \label{thm:point_minimal_p2}
Consider a generic RS camera in $\mathcal{P}_{d,\delta}$ observing $p$ world points.
The balanced SfM problems using 2 image points per world point are:
\\$\bullet$ $d >0, 3(d+\delta)-1=p$: this is minimal of degree 1 whenever $\delta=0$,  minimal \\\hspace*{1.2mm} of degree $598$ for $d=\delta=1$, and minimal of degree $2993$ for $d=2,\delta=1$;
\\$\bullet$ $d=0, 3 \delta=2p$: this is minimal of degree $1136$ for $\delta=2, p=3$.
\end{theorem}

\noindent
Single-view RS SfM using point projections is related to motion estimation from asynchronous tracks \cite{linear_pure_trans,hruby2026asynchronous}: The approach in \cite{hruby2026asynchronous} estimates the camera motion $R(t), C(t)$ from projections $u_{i,j}$ of $p$ unknown 3D points $X_i$ at $n$ distinct known timestamps $t_{i,j}$, satisfying $u_{i,j} \sim R(t_{i,j}) (X_i - C(t_{i,j}))$. This formulation is mainly used for event cameras and multi-view RS videos, assuming a special camera motion. Single-view RS SfM is a special case, with $t_{i,j}=x_{i,j}$.
Our work differs from \cite{hruby2026asynchronous} in rotation parameterization: \cite{hruby2026asynchronous} uses a polynomial approximation of constant angular velocity, while we use the exact Cayley model.

For $\delta=0$, no rotation is used, and our problem becomes equivalent to \cite{linear_pure_trans,hruby2026asynchronous}. For $d=1,\delta=0$, the camera order is $2$, and \cite{hruby2026asynchronous} shows that estimating velocity $V \in \RR^3$ reduces to fitting the essential matrix $[V]_{\times}$ to point pairs $u_{i,1}, u_{i,2}$. As $[V]_{\times}$ has $2$ DoF up to scale, two point pairs are needed to linearly estimate $V$, which is consistent with \Cref{thm:point_minimal_problems}.

\section{Experiments}\label{sec:experiments}

SM Table \ref{tab:minimal} lists all minimal problems discovered in Section~\ref{sec:lines}.
Here, we present an experimental evaluation of some resulting minimal solvers. Although our work is primarily theoretical, the experiments illustrate which aspects of the theory can be applied in practice. All solvers are implemented with homotopy continuation MiNuS \cite{fabbri2020minus} and are not optimized for runtime. 
Solvers with $d=0$ are labeled $\delta\!M(N)$, where $M=\deg(A)$ and $N$ lists points per image curve  (e.g. $N=3^2 2$ means 2 lines with 3 image points each and 1 line with 2 image points), followed by $P$ for parallel or $PC$ for parallel and coplanar lines. Solvers with $\delta=0$ have the analog label $dD(N)$.
We also include the LAAA solver \cite{lao2018robustified}, which solves the approximation of the $d=0,\delta=1$~problem.

\smallskip
\noindent
\textbf{Numerical stability.}\quad We sample the camera velocity $V_{GT}$ or Cayley parameters $A_{GT}$, and for $d \geq 2$ also higher-order motion parameters $V_{2,GT}$ and $V_{3,GT}$ from a 3D Gaussian distribution with $\mu=0,\sigma=0.2$. 3D lines are randomly sampled and projected at $N$ random scanlines $x$. For solvers assuming parallel lines, the lines share direction $\Delta_{GT}$; for coplanar solvers, they are also constrained to a plane. We consider lines projecting within the camera frame with focal length $f=768px$ and image size $w=640px$, $h=480px$.

The solvers estimate motion ($V$ or $A$) and, if applicable, line direction $\Delta$. For $\delta=0$, we measure the velocity error $\angle(V,V_{GT})$ and the line error $\angle(\Delta,\Delta_{GT})$. For $d=0$, we measure the axis error $\angle(A,A_{GT})$ and the relative norm error $\lVert A-A_{GT}\rVert/\lVert A_{GT} \rVert$. Results (Fig.~\ref{fig:stability_tests}) show that our solvers are mostly stable. A few samples have large errors, likely due to homotopy failures; as this is a proof of concept, more stable solvers are left as future work. The LAAA solver \cite{lao2018robustified} shows higher errors, as it solves an approximation rather than the exact problem.

\begin{figure}
    \centering
    \input{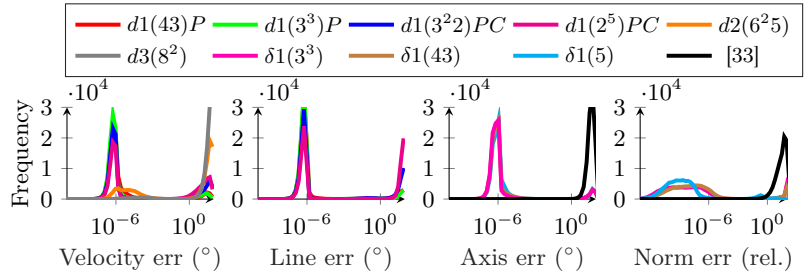}
    \vspace{-0.3cm}
    \caption{\textit{Noiseless stability.} Histogram of the velocity and line direction errors (for $\delta=0$) and the axis and norm errors (for $d=0$), calculated from $10^5$ noiseless samples.}
    \label{fig:stability_tests}
\end{figure}

\noindent
\textbf{Synthetic tests with noise.}\quad We sample problems like before, and add Gaussian noise with $\sigma=1px$ to each point measurement. The results are shown in Fig.~\ref{fig:noise_tests}. As expected, the solvers are less robust than multi-view methods, due to weaker constraints, yet, a substantial portion of solutions is accurate enough to initialize SfM. All solvers with $d=1,\delta=0$ behave similarly: about $23\%$ of samples have velocity error below $20^\circ$ and $45\%$ below $40^\circ$. The line error behaves similarly. Solvers with $d \geq 2$ have larger velocity errors. For $d=0,\delta=1$, $\delta1(3^3)$ is the most robust, with axis error below $40^\circ$ in about $33 \%$ of the cases.

\begin{figure}
    \centering
    \begin{tikzpicture}

\begin{axis}[%
width=0.16\columnwidth,
height=0.10\columnwidth,
at={(0in,0.389in)},
scale only axis,
xmin=0,
xmax=50,
xlabel style={font=\color{white!15!black}},
xlabel={Velocity err ($^\circ$)},
ymin=0,
ymax=1,
ymode=normal,
yminorticks=true,
ytick={0,0.5,1},
axis lines = left,
axis background/.style={fill=white},
title style={font=\bfseries},
ylabel style={yshift=-0.15in},
ylabel={Recall},
legend style={at={(2.5,1.1)}, anchor=south, legend cell align=left, align=left, draw=white!15!black, font=\footnotesize,nodes={scale=0.9, transform shape}},
legend columns=5
]


\addplot [color=red,line width=1.5pt, mark options={solid, red}]
  table[row sep=crcr]{%
0 0\\
1 0.01262\\
2 0.02418\\
3 0.0366\\
4 0.04869\\
5 0.06077\\
6 0.07298\\
7 0.08448\\
8 0.09587\\
9 0.10739\\
10 0.11904\\
11 0.12965\\
12 0.14102\\
13 0.15216\\
14 0.16386\\
15 0.17487\\
16 0.18607\\
17 0.19817\\
18 0.20907\\
19 0.21946\\
20 0.23113\\
21 0.24246\\
22 0.25352\\
23 0.26375\\
24 0.27441\\
25 0.28553\\
26 0.29647\\
27 0.3081\\
28 0.31887\\
29 0.33038\\
30 0.34154\\
31 0.35236\\
32 0.36332\\
33 0.37441\\
34 0.38554\\
35 0.39713\\
36 0.40808\\
37 0.41881\\
38 0.43016\\
39 0.44019\\
40 0.45124\\
41 0.46189\\
42 0.47281\\
43 0.48369\\
44 0.49505\\
45 0.50666\\
46 0.51779\\
47 0.52859\\
48 0.53914\\
49 0.55068\\
50 0.56237\\
51 0.57318\\
52 0.58376\\
53 0.59518\\
54 0.6061\\
55 0.61679\\
56 0.62729\\
57 0.63818\\
58 0.64879\\
59 0.65991\\
60 0.67111\\
61 0.68215\\
62 0.69385\\
63 0.70424\\
64 0.7153\\
65 0.72582\\
66 0.73636\\
67 0.74723\\
68 0.75855\\
69 0.76898\\
70 0.77964\\
71 0.79093\\
72 0.80204\\
73 0.81265\\
74 0.82355\\
75 0.83431\\
76 0.84543\\
77 0.85624\\
78 0.86739\\
79 0.87842\\
80 0.88954\\
81 0.90013\\
82 0.91064\\
83 0.92222\\
84 0.93369\\
85 0.94465\\
86 0.95549\\
87 0.96637\\
88 0.97737\\
89 0.98841\\
90 1\\
};
\addlegendentry{$d1(43)P$}

\addplot [color=green,line width=1.5pt, mark options={solid, red}]
  table[row sep=crcr]{%
0 0\\
1 0.01162\\
2 0.02305\\
3 0.03502\\
4 0.04706\\
5 0.05896\\
6 0.07127\\
7 0.08346\\
8 0.09558\\
9 0.10748\\
10 0.11904\\
11 0.13093\\
12 0.14304\\
13 0.15419\\
14 0.16593\\
15 0.17763\\
16 0.1899\\
17 0.2017\\
18 0.21289\\
19 0.22475\\
20 0.23609\\
21 0.24758\\
22 0.25931\\
23 0.27069\\
24 0.28147\\
25 0.29324\\
26 0.30435\\
27 0.31527\\
28 0.3263\\
29 0.33774\\
30 0.34949\\
31 0.36126\\
32 0.37269\\
33 0.38336\\
34 0.39438\\
35 0.40539\\
36 0.41631\\
37 0.42744\\
38 0.43833\\
39 0.44904\\
40 0.4608\\
41 0.47224\\
42 0.48287\\
43 0.49383\\
44 0.50472\\
45 0.51587\\
46 0.52694\\
47 0.53755\\
48 0.54867\\
49 0.55922\\
50 0.57035\\
51 0.58136\\
52 0.59195\\
53 0.60283\\
54 0.61393\\
55 0.62492\\
56 0.63621\\
57 0.647\\
58 0.65761\\
59 0.66832\\
60 0.67927\\
61 0.68995\\
62 0.70091\\
63 0.71162\\
64 0.72208\\
65 0.73257\\
66 0.74351\\
67 0.75326\\
68 0.764\\
69 0.77481\\
70 0.78538\\
71 0.79654\\
72 0.80737\\
73 0.81819\\
74 0.82897\\
75 0.84001\\
76 0.85032\\
77 0.86125\\
78 0.87186\\
79 0.88257\\
80 0.89341\\
81 0.90381\\
82 0.91512\\
83 0.92603\\
84 0.93668\\
85 0.94678\\
86 0.95772\\
87 0.96771\\
88 0.97883\\
89 0.98898\\
90 1\\
};
\addlegendentry{$d1(3^3)P$}

\addplot [color=blue,line width=1.5pt, mark options={solid, red}]
  table[row sep=crcr]{%
0 0\\
1 0.01149\\
2 0.02391\\
3 0.03582\\
4 0.04747\\
5 0.05877\\
6 0.07076\\
7 0.08311\\
8 0.09437\\
9 0.10647\\
10 0.11828\\
11 0.13038\\
12 0.14173\\
13 0.15292\\
14 0.16492\\
15 0.17666\\
16 0.18829\\
17 0.20035\\
18 0.21188\\
19 0.22325\\
20 0.23426\\
21 0.24568\\
22 0.25782\\
23 0.26902\\
24 0.28056\\
25 0.29194\\
26 0.30302\\
27 0.31445\\
28 0.32618\\
29 0.33734\\
30 0.34867\\
31 0.35996\\
32 0.37153\\
33 0.38295\\
34 0.3945\\
35 0.40522\\
36 0.41676\\
37 0.4277\\
38 0.43829\\
39 0.44916\\
40 0.4599\\
41 0.47106\\
42 0.48209\\
43 0.49327\\
44 0.50465\\
45 0.51576\\
46 0.52676\\
47 0.53729\\
48 0.54844\\
49 0.55948\\
50 0.57116\\
51 0.5819\\
52 0.59243\\
53 0.6034\\
54 0.61432\\
55 0.62569\\
56 0.6365\\
57 0.64769\\
58 0.6583\\
59 0.66884\\
60 0.67973\\
61 0.68958\\
62 0.70032\\
63 0.71074\\
64 0.72136\\
65 0.73165\\
66 0.74227\\
67 0.75305\\
68 0.76402\\
69 0.7746\\
70 0.78545\\
71 0.79617\\
72 0.80658\\
73 0.81731\\
74 0.82771\\
75 0.83836\\
76 0.84899\\
77 0.85978\\
78 0.87057\\
79 0.88139\\
80 0.89167\\
81 0.90208\\
82 0.91258\\
83 0.92345\\
84 0.93437\\
85 0.94496\\
86 0.9557\\
87 0.96673\\
88 0.97772\\
89 0.9885\\
90 1\\
};
\addlegendentry{$d1(3^22)PC$}

\addplot [color=magenta,line width=1.5pt, mark options={solid, red}]
  table[row sep=crcr]{%
0 0\\
1 0.01209\\
2 0.02457\\
3 0.03689\\
4 0.0491\\
5 0.06163\\
6 0.07307\\
7 0.08457\\
8 0.09621\\
9 0.10867\\
10 0.12077\\
11 0.13298\\
12 0.14476\\
13 0.1561\\
14 0.16815\\
15 0.17994\\
16 0.19148\\
17 0.20265\\
18 0.21405\\
19 0.22549\\
20 0.23685\\
21 0.24789\\
22 0.25981\\
23 0.27102\\
24 0.28212\\
25 0.29352\\
26 0.30594\\
27 0.3174\\
28 0.32824\\
29 0.33904\\
30 0.35062\\
31 0.3625\\
32 0.37315\\
33 0.38459\\
34 0.39555\\
35 0.40639\\
36 0.41755\\
37 0.42882\\
38 0.44006\\
39 0.45112\\
40 0.4618\\
41 0.47314\\
42 0.48445\\
43 0.49545\\
44 0.50689\\
45 0.5171\\
46 0.52787\\
47 0.53894\\
48 0.54942\\
49 0.56065\\
50 0.57122\\
51 0.5824\\
52 0.59288\\
53 0.60403\\
54 0.61544\\
55 0.62618\\
56 0.63705\\
57 0.64763\\
58 0.65874\\
59 0.66944\\
60 0.68038\\
61 0.6908\\
62 0.70228\\
63 0.7136\\
64 0.72451\\
65 0.73545\\
66 0.7467\\
67 0.75746\\
68 0.76846\\
69 0.77917\\
70 0.78884\\
71 0.79921\\
72 0.8099\\
73 0.82028\\
74 0.83081\\
75 0.84084\\
76 0.85191\\
77 0.86258\\
78 0.87338\\
79 0.88369\\
80 0.89436\\
81 0.90501\\
82 0.91584\\
83 0.92619\\
84 0.93658\\
85 0.94718\\
86 0.95811\\
87 0.96827\\
88 0.97871\\
89 0.98923\\
90 1\\
};
\addlegendentry{$d1(2^5)PC$}

\addplot [color=orange,line width=1.5pt, mark options={solid, red}]
  table[row sep=crcr]{%
0 0\\
1 0.0004\\
2 0.0008\\
3 0.0013\\
4 0.0023\\
5 0.0041\\
6 0.006\\
7 0.0082\\
8 0.01\\
9 0.0121\\
10 0.0147\\
11 0.0178\\
12 0.0214\\
13 0.0247\\
14 0.0288\\
15 0.0331\\
16 0.0384\\
17 0.0427\\
18 0.0473\\
19 0.0529\\
20 0.0582\\
21 0.0638\\
22 0.0694\\
23 0.0757\\
24 0.0838\\
25 0.0924\\
26 0.0995\\
27 0.1077\\
28 0.1163\\
29 0.1239\\
30 0.132\\
31 0.1404\\
32 0.1497\\
33 0.1585\\
34 0.1676\\
35 0.1763\\
36 0.1872\\
37 0.198\\
38 0.208\\
39 0.22\\
40 0.2318\\
41 0.2439\\
42 0.2576\\
43 0.2692\\
44 0.2811\\
45 0.2942\\
46 0.307\\
47 0.3207\\
48 0.3332\\
49 0.3454\\
50 0.3595\\
51 0.3726\\
52 0.3855\\
53 0.3985\\
54 0.4132\\
55 0.4252\\
56 0.4412\\
57 0.4565\\
58 0.4711\\
59 0.4852\\
60 0.5015\\
61 0.5159\\
62 0.5305\\
63 0.5471\\
64 0.5628\\
65 0.5786\\
66 0.5956\\
67 0.61\\
68 0.6279\\
69 0.6443\\
70 0.6597\\
71 0.6772\\
72 0.6924\\
73 0.7102\\
74 0.7249\\
75 0.7407\\
76 0.7577\\
77 0.7755\\
78 0.7911\\
79 0.8093\\
80 0.8269\\
81 0.8437\\
82 0.8617\\
83 0.8801\\
84 0.8983\\
85 0.9156\\
86 0.9338\\
87 0.9501\\
88 0.9665\\
89 0.9821\\
90 1\\
};
\addlegendentry{$d2(6^25)$}

\addplot [color=gray,line width=1.5pt, mark options={solid, red}]
  table[row sep=crcr]{%
0 0\\
1 0.0002\\
2 0.0006\\
3 0.0012\\
4 0.0022\\
5 0.0033\\
6 0.0043\\
7 0.0071\\
8 0.0105\\
9 0.0136\\
10 0.0164\\
11 0.0204\\
12 0.0249\\
13 0.0279\\
14 0.0335\\
15 0.0371\\
16 0.0437\\
17 0.0505\\
18 0.0551\\
19 0.0609\\
20 0.0668\\
21 0.0734\\
22 0.0797\\
23 0.0854\\
24 0.0925\\
25 0.1006\\
26 0.1096\\
27 0.1156\\
28 0.124\\
29 0.1316\\
30 0.14\\
31 0.149\\
32 0.1584\\
33 0.1684\\
34 0.1796\\
35 0.1898\\
36 0.198\\
37 0.2108\\
38 0.2217\\
39 0.233\\
40 0.2466\\
41 0.2592\\
42 0.2706\\
43 0.2807\\
44 0.292\\
45 0.3028\\
46 0.316\\
47 0.3276\\
48 0.3411\\
49 0.3555\\
50 0.3676\\
51 0.3797\\
52 0.3946\\
53 0.4096\\
54 0.4223\\
55 0.4364\\
56 0.4507\\
57 0.4659\\
58 0.4791\\
59 0.4942\\
60 0.5106\\
61 0.5243\\
62 0.5386\\
63 0.5554\\
64 0.5706\\
65 0.5853\\
66 0.5989\\
67 0.6139\\
68 0.6286\\
69 0.6447\\
70 0.6619\\
71 0.6795\\
72 0.6974\\
73 0.7156\\
74 0.7349\\
75 0.7491\\
76 0.7656\\
77 0.7823\\
78 0.7969\\
79 0.8121\\
80 0.8281\\
81 0.8444\\
82 0.8636\\
83 0.8794\\
84 0.8991\\
85 0.9174\\
86 0.9343\\
87 0.9494\\
88 0.9667\\
89 0.9833\\
90 1\\
};
\addlegendentry{$d3(8^2)$}

\addplot [color=VeryPink,line width=1.5pt, mark options={solid, red}]
  table[row sep=crcr]{%
0 0\\
};
\addlegendentry{$\delta1(3^3)$}

\addplot [color=brown,line width=1.5pt, mark options={solid, red}]
  table[row sep=crcr]{%
0 0\\
};
\addlegendentry{$\delta1(43)$}

\addplot [color=cyan,line width=1.5pt, mark options={solid, red}]
  table[row sep=crcr]{%
0 0\\
};
\addlegendentry{$\delta1(5)$}

\addplot [color=black,line width=1.5pt, mark options={solid, red}]
  table[row sep=crcr]{%
0 0\\
};
\addlegendentry{\cite{lao2018robustified}}


\end{axis}

\input{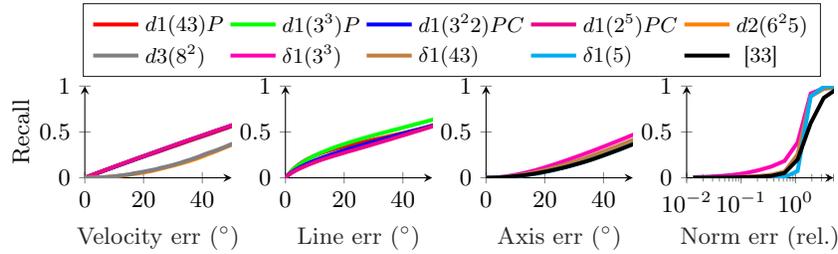}

\end{tikzpicture}%
    \vspace{-0.3cm}
    \caption{\textit{Noise test.} Recall curve of the velocity and line direction errors (for $\delta=0$) and the axis and norm errors (for $d=0$), calculated from $10^5$ samples with noise $\sigma=1px$, $||V||=0.2$ for $\delta=0$, $||A||=0.2$ for $d=0$.}
    \label{fig:noise_tests}
\end{figure}

\smallskip
\noindent
\textbf{Real-world experiments.}\quad For $d=1,\delta=0$, we use sequence 06 from \cite{DBLP:conf/cvpr/HedborgFFR12}, depicting parallel and coplanar lines with camera motion close to pure linear translation. As the sequence is depicted by a RS and GS camera moving together, we find pseudo-GT velocities and line directions by reconstructing the GS images in COLMAP \cite{schoenberger2016sfm}. For $d=0,\delta=1$, we use the iPhone 3GS sequence from \cite{DBLP:conf/cvpr/ForssenR10}. Since no GS images are available, pseudo-GT is obtained via a multi-view point-based method described in SM Sec. \ref{sec:SMexperiments}. Although less accurate than COLMAP on GS images, we believe  this provides a reasonable estimate of angular motion. Image curves are detected with a method similar to \cite{halir1998numerically} adapted to RS projections (see SM), and used as input to RANSAC \cite{ransac} with $1000$ iterations that selects the model minimizing the reprojection error (see SM). More advanced RANSAC variants \cite{barath2018graph} are left for future work.
Results (Fig.~\ref{fig:real_tests}) show that $d1(3^22)PC$, assuming parallel and coplanar lines, achieves velocity error below $20^\circ$ in about $50\%$ of images and below $40^\circ$ in $80\%$. $d1(3^3)P$ is less accurate but does not rely on coplanarity, making it more universal. Line direction estimates are accurate for all solvers. We exclude $d1(43)P$, as with our detector it consistently estimates the velocity as $[1 \ 0 \ 0]^\top$. For $d=0,\delta=1$, our $\delta1(43)$ and $\delta1(5)$ solvers perform best, outperforming LAAA \cite{lao2018robustified} in both axis and norm error. Overall, the results show that for some images with simple motion and regular scenes, single-image RS motion estimation achieves acceptable accuracy.

\begin{figure}
    \centering
    \begin{tikzpicture}

\begin{axis}[%
width=0.16\columnwidth,
height=0.10\columnwidth,
at={(0in,0.389in)},
scale only axis,
xmin=0,
xmax=50,
xlabel style={font=\color{white!15!black}},
xlabel={Velocity err ($^{\circ}$)},
ytick={0,0.5,1},
ymin=0,
ymax=1,
ymode=normal,
yminorticks=true,
axis lines = left,
axis background/.style={fill=white},
title style={font=\bfseries},
ylabel style={yshift=-0.15in},
xlabel style={yshift=0.07in},
ylabel={Recall},
legend style={at={(1.0,1.15)}, anchor=south, legend cell align=left, align=left, draw=white!15!black, font=\footnotesize,nodes={scale=0.9, transform shape}},
legend image post style={xscale=0.5},
legend columns=4
]



\addplot [color=green,line width=1.5pt, mark options={solid, red}]
  table[row sep=crcr]{%
0 0\\
2.1626625628294285 0.02\\
7.211603004504696 0.04\\
9.09691898180879 0.06\\
11.775945342969658 0.08\\
13.26801035350923 0.1\\
16.495948263158212 0.12\\
16.86559732088605 0.14\\
16.938738843756997 0.16\\
17.189245714148402 0.18\\
21.50402625762772 0.2\\
22.48715527314066 0.22\\
23.0699668131461 0.24\\
23.25153841848222 0.26\\
24.162101317263648 0.28\\
25.51903658923878 0.3\\
26.832401468881788 0.32\\
27.537056365052383 0.34\\
28.013388557548296 0.36\\
29.822949260936596 0.38\\
31.446043459148036 0.4\\
37.807619201202485 0.42\\
38.243838157967005 0.44\\
38.53833694201545 0.46\\
39.99696128156532 0.48\\
40.12799769007599 0.5\\
41.456925207298866 0.52\\
43.400089522250504 0.54\\
47.058177061870865 0.56\\
47.49563888549623 0.58\\
47.64000025869842 0.6\\
49.595351717414914 0.62\\
52.042768056531344 0.64\\
52.539102198326326 0.66\\
57.17286164687516 0.68\\
58.15991473839367 0.7\\
59.67189741674349 0.72\\
61.85246957884144 0.74\\
62.26662828304322 0.76\\
62.49828988353824 0.78\\
62.66740110567999 0.8\\
62.96879230924093 0.82\\
63.79285750752781 0.84\\
65.4302011030257 0.86\\
65.66693597250345 0.88\\
68.35815938688083 0.9\\
71.14681709091545 0.92\\
127.45718111168503 0.94\\
148.3893158202355 0.96\\
154.06769921997463 0.98\\
167.94473030024903 1.0\\
};
\addlegendentry{$d1(3^3)P$}

\addplot [color=blue,line width=1.5pt, mark options={solid, red}]
  table[row sep=crcr]{%
0 0\\
0.5025324448555653 0.02\\
0.70766468934265 0.04\\
0.9644654417976363 0.06\\
1.172668315837372 0.08\\
3.9051094208341524 0.1\\
4.122769986024683 0.12\\
4.986132431253123 0.14\\
5.2445892973500134 0.16\\
5.584836266303056 0.18\\
5.816090366320654 0.2\\
5.857142640158207 0.22\\
6.203610656569953 0.24\\
7.05285444597992 0.26\\
8.267917380901832 0.28\\
8.88479845952429 0.3\\
9.959818973933375 0.32\\
10.09450040143228 0.34\\
10.423236497654589 0.36\\
12.814191560093215 0.38\\
14.513570657249536 0.4\\
15.073996861185435 0.42\\
15.89714061526874 0.44\\
19.59387257723594 0.46\\
23.950889985676028 0.48\\
24.455361186469922 0.5\\
24.83121684317552 0.52\\
27.15422609800594 0.54\\
27.379235472029183 0.56\\
28.368876941903814 0.58\\
29.303004149314635 0.6\\
29.75585349079851 0.62\\
32.51741861057485 0.64\\
33.00203537632577 0.66\\
33.0114627769554 0.68\\
35.464646093693474 0.7\\
37.10801091274031 0.72\\
37.966013509660826 0.74\\
38.0414814172674 0.76\\
39.90315656787976 0.78\\
41.151959949045825 0.8\\
41.90344778281235 0.82\\
44.625805419467945 0.84\\
46.77316529258062 0.86\\
46.81334882128129 0.88\\
48.720396597339956 0.9\\
49.25246799453055 0.92\\
49.5317921229158 0.94\\
49.83487181985465 0.96\\
86.86970814509867 0.98\\
174.7433519697376 1.0\\
};
\addlegendentry{$d1(3^22)PC$}

\addplot [color=magenta,line width=1.5pt, mark options={solid, red}]
  table[row sep=crcr]{%
0 0\\
5.088828254569242 0.02\\
11.235590344375103 0.04\\
13.031767830932276 0.06\\
13.524721062357585 0.08\\
13.983538868683654 0.1\\
14.699145152496712 0.12\\
17.2220660313047 0.14\\
19.384144417989077 0.16\\
20.91887963963622 0.18\\
21.924654039722427 0.2\\
22.780030626970127 0.22\\
23.10043315652972 0.24\\
26.776955221085263 0.26\\
26.835955144709317 0.28\\
26.924980339324776 0.3\\
30.21576457452801 0.32\\
30.50201488692758 0.34\\
31.247105649973538 0.36\\
31.359566980927575 0.38\\
32.15946575302991 0.4\\
32.54586323498684 0.42\\
33.384821986861496 0.44\\
34.365492131591196 0.46\\
38.52285063996279 0.48\\
39.19149776139307 0.5\\
39.878254885775505 0.52\\
43.390259566572716 0.54\\
52.18788390503323 0.56\\
52.46229967399252 0.58\\
52.518067349560134 0.6\\
54.18915195471601 0.62\\
54.50988327787579 0.64\\
55.280654779482454 0.66\\
56.82944089195092 0.68\\
57.131519340926275 0.7\\
58.58108633663663 0.72\\
59.22640046308004 0.74\\
60.43626490001082 0.76\\
61.135413088120494 0.78\\
61.52689541348259 0.8\\
64.3310401416725 0.82\\
67.15942027623437 0.84\\
67.50811009941305 0.86\\
68.22986975817368 0.88\\
68.62759018953605 0.9\\
82.89338151615961 0.92\\
85.04490294388789 0.94\\
98.62216763751974 0.96\\
113.26118420318345 0.98\\
179.09883120108017 1.0\\
};
\addlegendentry{$d1(2^5)PC$}


\end{axis}

\input{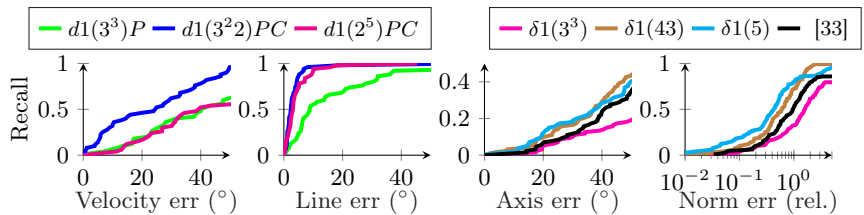}

\end{tikzpicture}%
    \vspace{-0.4cm}
    \caption{\textit{Real-world experiments.} Recall curves of the velocity and line direction errors of the proposed solvers with $\delta=0$ on the dataset \cite{DBLP:conf/cvpr/HedborgFFR12} (\textit{left}), and of the axis and norm errors of the proposed solvers with $d=0$ on the dataset \cite{DBLP:conf/cvpr/ForssenR10}.}
    \label{fig:real_tests}
\end{figure}

\section{Conclusion}
We characterized the algebraic and geometric properties of RS images of world points or lines, exploited those features to systematically explore SfM possibilities from a single view, and evaluated examples of minimal solvers. Despite its exhaustive nature, our work leaves many questions open. 
This includes the study of 1) single-view SfM from images of both world lines \emph{and} world points, or more complex primitives,2) multi-view SfM, or 3) alternative camera models (e.g., tangential / linearized rotations to model small but non-negligible rotations more accurately, or the more expressive model where $C(x)$ and $A(x)$ are rational (instead of polynomial) functions). 
Moreover, several statements in this article have been proven computationally for small choices of camera hyperparameters, but deserve more general, theoretical proofs.
We also leave the study of essential-matrix analogs for single RS cameras observing points for future work.
Finally, it should be investigated whether the minimal problems proposed in this article decompose into easier problems (cf. \cite{duff2025galois}), and can be used together within hybrid RANSAC \cite{DBLP:conf/cvpr/CamposecoCPS18} to build a general purpose, robust pipeline for RS SfM.

\section*{Acknowledgments} The authors thank Tomas Pajdla for numerous discussions on rolling-shutter cameras. 
S. E. M., P. H. and Ka. K. were supported by the Swedish Foundations’ Starting Grant \emph{Algebraic Vision} funded by the Ragnar Söderbergs stiftelse.
Kim K. and Ka. K. were supported by the Wallenberg AI, Autonomous Systems and Software Program (WASP) funded by the Knut and Alice Wallenberg Foundation.
%
%
\bibliographystyle{splncs04}
\bibliography{main}

\newpage
\appendix

\section{Algebraic Geometry Prerequisites} \label{sec:AG}

Let $\mathbb{K}$ be a field. In this work, $\mathbb{K}$ is either the real numbers $\RR$, the complex numbers $\CC$, or  -- for symbolic computations --  the rational numbers $\mathbb{Q}$ or occasionally a finite field.
As customary, we work with homogeneous coordinates on 3D world and images; that is, we work with projective spaces $\PP^n$ that are obtained from $\mathbb{K}^{n+1}\setminus \{ 0 \}$ by identifying vectors that are the same up to scalar multiplication. 
We say that $\PP^n$ is the projectivization of the vector space $\mathbb{K}^{n+1}$. 
Similarly, the \emph{dual projective space} $(\PP^n)^\ast$ is the projectivization of the dual vector space $(\mathbb{K}^{n+1})^\ast$. Note that points in  $(\PP^n)^\ast$ correspond to hyperplanes in $\PP^n$.

\begin{example} \label{ex:scanlineDual}
    The scanlines of our RS sensors are parallel to the $y$-axis. For instance, $\{ (5,y) \mid y \in \mathbb{K} \}$ is a scanline; its defining equation in homogeneous coordinates is $x-5z = 0$. 
    That line in the projective plane $\PP^2$ corresponds to the point $\begin{bmatrix}
        1:0:-5
    \end{bmatrix}$ in the dual projective plane $(\PP^2)^\ast$.
    \hfill $\diamondsuit$
\end{example}

\noindent
The projective space $\PP^n$ can also be seen as the union of $n+1$ \emph{affine charts} (i.e., $n+1$ copies of $\mathbb{K}^n$ ``glued together'' correctly; see \cite{shafarevich} for a formal treatment). 
Our default affine chart $\mathbb{K}^n \subseteq \PP^n$, as customary in computer vision, is the one where the last coordinate of points in $\PP^n$ is assumed to be $1$.

As our camera model is polynomial, algebraic varieties appear naturally in our work. 
An \emph{affine variety} in $\mathbb{K}^n$ is the solution set of a system of polynomial equations in $n$ variables (note that there are more general notions of algebraic varieties, see \cite{shafarevich}, but this definition suffices for our purposes).
A \emph{projective variety} in $\PP^n$ is the solution set of a system of \emph{homogeneous} polynomial equations.
Given a projective variety in $\PP^n$, we can obtain an affine variety in $\mathbb{K}^n$ by setting the last variable in each defining equation to $1$.
Conversely, from an affine variety in $\mathbb{K}^n$, we can build a projective variety in $\PP^n$ by introducing a new (last) variable and using that to homogenize each defining equation.

The latter process is a special instance of taking Zariski closures. 
Given any set $S$ in $\mathbb{K}^n$ or $\PP^n$, the \emph{Zariski closure} of $S$ is the smallest variety in $\mathbb{K}^n$ or $\PP^n$, respectively, that contains $S$.
For instance, for an affine variety $X$ in  $\mathbb{K}^n$, when viewing $\mathbb{K}^n$ as a subset of $\PP^n$, the Zariski closure of $X$ inside $\PP^n$ is precisely the projective variety obtained by homogenization described above.

In this work, we model the picture-taking process as an algebraic map between varieties.
An \emph{algebraic map} $f: X \to Y$ between affine varieties $X \subseteq \mathbb{K}^n$ and $Y \subseteq \mathbb{K}^m$ is given by $m$ rational functions $f_i$ in $n$ variables each.
For projective varieties $X \subseteq \PP^n$ and $Y \subseteq \PP^m$, algebraic maps $f: X \to Y$ are given by $m+1$ homogeneous polynomial functions $f_i$ in $n+1$ variables each (again, for more general notions of algebraic maps, see \cite{shafarevich}).

\begin{example} \label{ex:picMaps}
    By Theorem \ref{thm:imageCurve}, for almost all RS cameras in $\mathcal{P}_{d,\delta}$ (see \eqref{eq:cameraModel}), the picture-taking process of observing $\ell$ lines is an algebraic map between projective varieties
    \begin{equation} \label{eq:picMapFixedCam}
        \mathrm{Gr}(1, \PP^3)^\ell \to \mathcal{H}_{1+d+2\delta}^\ell,
    \end{equation}
    where the Grassmannian $\mathrm{Gr}(1, \PP^3) := \{ L \subseteq \PP^3 \mid L \text{ is line} \}$ can be seen as a projective variety inside $\PP^5$ via the Pl\"ucker coordinates and $\mathcal{H}_{1+d+2\delta}$ is a $2(1+d+2\delta)$-dimensional projective space by Proposition \ref{prop:H}. 

    When the camera is not fixed, we can model the picture-taking process of varying cameras as an algebraic map as well:
    \begin{equation}
        \label{eq:picMapVariedCam}
        \mathcal{P}_{d, \delta} \times \mathrm{Gr}(1, \PP^3)^\ell \to \mathcal{H}_{1+d+2\delta}^\ell. 
    \end{equation}
    Given an RS camera and $\ell$ lines, this map outputs the $\ell$ resulting image curves. 
    The map in \eqref{eq:picMapFixedCam} is obtained from \eqref{eq:picMapVariedCam} by specializing to a fixed camera.
    \hfill $\diamondsuit$
\end{example}

We are interested, on the one hand, in understanding all possible RS images (e.g., in understanding the images of the maps in Example \ref{ex:picMaps}) and, on the other hand, in 3D reconstruction.
Algebraically, 3D reconstruction means to compute \emph{fibers} of algebraic maps $f: X \to Y$, i.e., preimages $f^{-1}(y)$ over points $y \in Y$.

\begin{example}
    In the setting of Example \ref{ex:picMaps}, triangulation refers to computing  fibers of the map in \eqref{eq:picMapFixedCam}, while SfM aims to compute  fibers of the map \eqref{eq:picMapVariedCam}. \hfill $\diamondsuit$
\end{example}

\noindent
The dimension of the image of an algebraic map is directly linked to the dimension of its fibers. 
This  is formalized in the following theorem.
That theorem, like many other assertions in this paper, uses the concept of \emph{generic} points. 
A property is said to hold for generic points in a variety $X$ if there is a proper subvariety $Z \subsetneq X$ such that the property holds for all $x \in X \setminus Z$. 
Equivalently, we often write that the property holds \emph{for almost all} points in $X$; or simply say that it holds \emph{generically}.
Moreover, the following theorem holds for algebraic maps $f: X \to Y$ between \emph{irreducible} varieties, which are varieties that are not the union of two proper subvarieties. 
Note that the irreducibility of $X$ implies that the image $f(X)$ must be irreducible as well. 
All varieties describing structure or motion in this article are irreducible, e.g., all varieties in Example \ref{ex:picMaps} are irreducible.

\begin{theorem}[Fiber-Dimension Theorem {\cite[Chap. 1.6.3]{shafarevich}}] \label{thm:fiberDim}
    If $f: X \to Y$ is an algebraic map between irreducible varieties (over $\mathbb{C})$, then 
    \begin{align*}
        \dim X = \dim f^{-1} (f(x)) + \dim \mathrm{im}(f)
    \end{align*}
    for generic $x \in X$.
    Moreover, $\dim f^{-1} (f(x)) \geq \dim X - \dim \mathrm{im}(f)$ for all $x \in X$. 
\end{theorem}

\noindent
    This theorem is the nonlinear version of the rank-nullity theorem in linear algebra. The main difference is that for linear maps $f$, all fibers $f^{-1} (f(x))$ have the same dimension, while for polynomial / rational maps only almost all fibers have the same dimension.
A computational criterion to compute the dimension of the image of an algebraic map is by computing the rank of its Jacobian at a generic point in the domain, stated formally as follows:
\begin{lemma}[{\cite[Lemma 2.4 in Chap. 2.6]{shafarevich}}]
    \label{lem:jacobianCheck}
    If $f: X \to Y$ is an algebraic map between irreducible varieties, then  
    \begin{align*}
        \dim \mathrm{im}(f) = \mathrm{rank}  \, d_x f
    \end{align*}
    for almost all $x \in X$. Moreover, $\dim \mathrm{im}(f) \geq \mathrm{rank}  \, d_x f$ for all smooth $x \in X$. 
\end{lemma}

\noindent
Here, the smooth points on a variety $X$ are those where $X$ locally looks like a smooth manifold (for a formal definition, see \cite{shafarevich}). 
When the ground field $\mathbb{K}$ is one of $\CC, \RR, \mathbb{Q}$, then almost every point on a variety $X$ is smooth. 
At a smooth point $x$ of an irreducible variety $X$, the tangent space $T_x X$ has the same dimension as $X$. 
Given an algebraic map $f: X\to Y$, its derivative at $x$ is a linear map $d_xf: T_x X \to T_{f(x)} Y$, whose representation matrix is the Jacobian mentioned right before Lemma \ref{lem:jacobianCheck}.

We call an algebraic map \emph{birational} (onto its image) if it is injective almost everywhere, i.e., the fiber $f^{-1}(f(x))= \{ x \}$ for almost all $x \in X$. Here is a computational criterion to check for birationality:

\begin{lemma}[{\cite[Prop. 4.6.7]{ega}}]
\label{lem:birationalCheck}
    Let $f: X \to Y$ be an algebraic map between irreducible varieties (over $\mathbb{C})$. If there is some smooth point $x \in X$ such that the derivative $d_x f$ of $f$ at $x$ has full rank and $f^{-1}(f(x))= \{ x \}$, then the map $f$ is birational onto its image. 
\end{lemma}

We say that an algebraic reconstruction problem is a \emph{minimal problem} if the generic fibers of its associated algebraic map $f: X \to Y$ are zero-dimensional; i.e., for generic $y \in Y$, we require that $f^{-1}(y)$ is finite and non-empty. 
Here,  $Y$ is the variety of possible measurements and $X$ is the variety of parameters to be recovered.
Typically, $Y$ should be a linear space (in an affine chart), as in Example \ref{ex:picMaps}, because when $Y$ is defined by nonlinear constraints (e.g., the degree-10 equation in Sec. \ref{sec:longEquation}), it is often practically  infeasible to obtain exact measurements on $Y$ due to noise. 
By the Fiber-Dimension Theorem \ref{thm:fiberDim}, every minimal problem satisfies that $\dim X = \dim Y$, and we call problems with that property \emph{balanced} (as in \cite{DBLP:journals/pami/DuffKLP24}).
Moreover, given a balanced problem, we can check whether it is indeed minimal by testing one of the following two equivalent conditions:
1) $f^{-1}(y) \neq \emptyset$ for generic $y \in Y$, or
2) $\dim \mathrm{im}(f) = \dim Y$.
We can test the first condition by finding a solution (e.g., with homotopy continuation) and the second condition by the Jacobian check in Lemma \ref{lem:jacobianCheck}.

\section{Minimal Problems} \label{sec:minProbs}

\begin{table}

\newcommand{\picplaceholder}{
\begin{tikzpicture}[scale=0.6]
\draw[gray] (0,0) rectangle (2.8,1.8);
\node at (1.4,0.9) {\small picture};
\end{tikzpicture}
}

\renewcommand{\arraystretch}{1.6}


\begin{center}
\resizebox{0.8\linewidth}{!}{
\begin{tabular}{|p{2.6cm}|ccccc|}
\hline 
\rule{0pt}{2.2cm}
\raisebox{0.75cm}{$\delta=0$, $d=1$}
&
\raisebox{-0.1cm}{\tikzset{every picture/.style={line width=0.75pt}} 

\begin{tikzpicture}[x=0.4pt,y=0.4pt,yscale=-1,xscale=1]

\draw   (202.25,52) -- (348,52) -- (348,197.75) -- (202.25,197.75) -- cycle ;
\draw [color={rgb, 255:red, 208; green, 2; blue, 27 }  ,draw opacity=1 ]   (305.5,62.25) -- (215.5,158.25) ;
\draw  [fill={rgb, 255:red, 0; green, 0; blue, 0 }  ,fill opacity=1 ] (287.25,79.13) .. controls (287.25,80.44) and (288.31,81.5) .. (289.63,81.5) .. controls (290.94,81.5) and (292,80.44) .. (292,79.13) .. controls (292,77.81) and (290.94,76.75) .. (289.63,76.75) .. controls (288.31,76.75) and (287.25,77.81) .. (287.25,79.13) -- cycle ;
\draw [color={rgb, 255:red, 208; green, 2; blue, 27 }  ,draw opacity=1 ][fill={rgb, 255:red, 208; green, 2; blue, 27 }  ,fill opacity=1 ]   (332.5,88.75) -- (243,184.5) ;
\draw  [color={rgb, 255:red, 0; green, 0; blue, 0 }  ,draw opacity=1 ][fill={rgb, 255:red, 0; green, 0; blue, 0 }  ,fill opacity=1 ] (272.75,94.63) .. controls (272.75,95.94) and (273.81,97) .. (275.13,97) .. controls (276.44,97) and (277.5,95.94) .. (277.5,94.63) .. controls (277.5,93.31) and (276.44,92.25) .. (275.13,92.25) .. controls (273.81,92.25) and (272.75,93.31) .. (272.75,94.63) -- cycle ;
\draw  [color={rgb, 255:red, 0; green, 0; blue, 0 }  ,draw opacity=1 ][fill={rgb, 255:red, 0; green, 0; blue, 0 }  ,fill opacity=1 ] (256.75,112.13) .. controls (256.75,113.44) and (257.81,114.5) .. (259.13,114.5) .. controls (260.44,114.5) and (261.5,113.44) .. (261.5,112.13) .. controls (261.5,110.81) and (260.44,109.75) .. (259.13,109.75) .. controls (257.81,109.75) and (256.75,110.81) .. (256.75,112.13) -- cycle ;
\draw  [color={rgb, 255:red, 0; green, 0; blue, 0 }  ,draw opacity=1 ][fill={rgb, 255:red, 0; green, 0; blue, 0 }  ,fill opacity=1 ] (238.75,131.13) .. controls (238.75,132.44) and (239.81,133.5) .. (241.13,133.5) .. controls (242.44,133.5) and (243.5,132.44) .. (243.5,131.13) .. controls (243.5,129.81) and (242.44,128.75) .. (241.13,128.75) .. controls (239.81,128.75) and (238.75,129.81) .. (238.75,131.13) -- cycle ;
\draw  [fill={rgb, 255:red, 0; green, 0; blue, 0 }  ,fill opacity=1 ] (305.25,115.13) .. controls (305.25,116.44) and (306.31,117.5) .. (307.63,117.5) .. controls (308.94,117.5) and (310,116.44) .. (310,115.13) .. controls (310,113.81) and (308.94,112.75) .. (307.63,112.75) .. controls (306.31,112.75) and (305.25,113.81) .. (305.25,115.13) -- cycle ;
\draw  [fill={rgb, 255:red, 0; green, 0; blue, 0 }  ,fill opacity=1 ] (285.75,135.63) .. controls (285.75,136.94) and (286.81,138) .. (288.13,138) .. controls (289.44,138) and (290.5,136.94) .. (290.5,135.63) .. controls (290.5,134.31) and (289.44,133.25) .. (288.13,133.25) .. controls (286.81,133.25) and (285.75,134.31) .. (285.75,135.63) -- cycle ;
\draw  [fill={rgb, 255:red, 0; green, 0; blue, 0 }  ,fill opacity=1 ] (263.75,159.63) .. controls (263.75,160.94) and (264.81,162) .. (266.13,162) .. controls (267.44,162) and (268.5,160.94) .. (268.5,159.63) .. controls (268.5,158.31) and (267.44,157.25) .. (266.13,157.25) .. controls (264.81,157.25) and (263.75,158.31) .. (263.75,159.63) -- cycle ;

\end{tikzpicture}}
&
\raisebox{-0.1cm}{\tikzset{every picture/.style={line width=0.75pt}} 

\begin{tikzpicture}[x=0.4pt,y=0.4pt,yscale=-1,xscale=1]

\draw   (222.25,72) -- (368,72) -- (368,217.75) -- (222.25,217.75) -- cycle ;
\draw [color={rgb, 255:red, 208; green, 2; blue, 27 }  ,draw opacity=1 ]   (297.67,85) -- (236.33,178.33) ;
\draw  [fill={rgb, 255:red, 0; green, 0; blue, 0 }  ,fill opacity=1 ] (283.92,102.46) .. controls (283.92,103.77) and (284.98,104.83) .. (286.29,104.83) .. controls (287.6,104.83) and (288.67,103.77) .. (288.67,102.46) .. controls (288.67,101.15) and (287.6,100.08) .. (286.29,100.08) .. controls (284.98,100.08) and (283.92,101.15) .. (283.92,102.46) -- cycle ;
\draw [color={rgb, 255:red, 208; green, 2; blue, 27 }  ,draw opacity=1 ][fill={rgb, 255:red, 208; green, 2; blue, 27 }  ,fill opacity=1 ]   (353,119) -- (294.33,208.5) ;
\draw  [color={rgb, 255:red, 0; green, 0; blue, 0 }  ,draw opacity=1 ][fill={rgb, 255:red, 0; green, 0; blue, 0 }  ,fill opacity=1 ] (254.75,146.96) .. controls (254.75,148.27) and (255.81,149.33) .. (257.13,149.33) .. controls (258.44,149.33) and (259.5,148.27) .. (259.5,146.96) .. controls (259.5,145.65) and (258.44,144.58) .. (257.13,144.58) .. controls (255.81,144.58) and (254.75,145.65) .. (254.75,146.96) -- cycle ;
\draw  [color={rgb, 255:red, 0; green, 0; blue, 0 }  ,draw opacity=1 ][fill={rgb, 255:red, 0; green, 0; blue, 0 }  ,fill opacity=1 ] (268.75,125.46) .. controls (268.75,126.77) and (269.81,127.83) .. (271.13,127.83) .. controls (272.44,127.83) and (273.5,126.77) .. (273.5,125.46) .. controls (273.5,124.15) and (272.44,123.08) .. (271.13,123.08) .. controls (269.81,123.08) and (268.75,124.15) .. (268.75,125.46) -- cycle ;
\draw  [color={rgb, 255:red, 0; green, 0; blue, 0 }  ,draw opacity=1 ][fill={rgb, 255:red, 0; green, 0; blue, 0 }  ,fill opacity=1 ] (333.75,145.46) .. controls (333.75,146.77) and (334.81,147.83) .. (336.13,147.83) .. controls (337.44,147.83) and (338.5,146.77) .. (338.5,145.46) .. controls (338.5,144.15) and (337.44,143.08) .. (336.13,143.08) .. controls (334.81,143.08) and (333.75,144.15) .. (333.75,145.46) -- cycle ;
\draw  [fill={rgb, 255:red, 0; green, 0; blue, 0 }  ,fill opacity=1 ] (345.25,127.79) .. controls (345.25,129.1) and (346.31,130.17) .. (347.63,130.17) .. controls (348.94,130.17) and (350,129.1) .. (350,127.79) .. controls (350,126.48) and (348.94,125.42) .. (347.63,125.42) .. controls (346.31,125.42) and (345.25,126.48) .. (345.25,127.79) -- cycle ;
\draw  [fill={rgb, 255:red, 0; green, 0; blue, 0 }  ,fill opacity=1 ] (308.75,183.63) .. controls (308.75,184.94) and (309.81,186) .. (311.13,186) .. controls (312.44,186) and (313.5,184.94) .. (313.5,183.63) .. controls (313.5,182.31) and (312.44,181.25) .. (311.13,181.25) .. controls (309.81,181.25) and (308.75,182.31) .. (308.75,183.63) -- cycle ;
\draw [color={rgb, 255:red, 208; green, 2; blue, 27 }  ,draw opacity=1 ]   (334.33,89) -- (257.5,203.58) ;
\draw  [color={rgb, 255:red, 0; green, 0; blue, 0 }  ,draw opacity=1 ][fill={rgb, 255:red, 0; green, 0; blue, 0 }  ,fill opacity=1 ] (310.42,121.12) .. controls (310.42,122.44) and (311.48,123.5) .. (312.79,123.5) .. controls (314.1,123.5) and (315.17,122.44) .. (315.17,121.12) .. controls (315.17,119.81) and (314.1,118.75) .. (312.79,118.75) .. controls (311.48,118.75) and (310.42,119.81) .. (310.42,121.12) -- cycle ;
\draw  [color={rgb, 255:red, 0; green, 0; blue, 0 }  ,draw opacity=1 ][fill={rgb, 255:red, 0; green, 0; blue, 0 }  ,fill opacity=1 ] (292.75,147.25) .. controls (292.75,148.56) and (293.81,149.63) .. (295.13,149.63) .. controls (296.44,149.63) and (297.5,148.56) .. (297.5,147.25) .. controls (297.5,145.94) and (296.44,144.88) .. (295.13,144.88) .. controls (293.81,144.88) and (292.75,145.94) .. (292.75,147.25) -- cycle ;
\draw  [color={rgb, 255:red, 0; green, 0; blue, 0 }  ,draw opacity=1 ][fill={rgb, 255:red, 0; green, 0; blue, 0 }  ,fill opacity=1 ] (269.75,181.79) .. controls (269.75,183.1) and (270.81,184.17) .. (272.13,184.17) .. controls (273.44,184.17) and (274.5,183.1) .. (274.5,181.79) .. controls (274.5,180.48) and (273.44,179.42) .. (272.13,179.42) .. controls (270.81,179.42) and (269.75,180.48) .. (269.75,181.79) -- cycle ;

\end{tikzpicture}}
&
&
&
\\
\textbf{lable}
&
\underline{$d_1(4,3)P$}
&
\underline{$d_1(3^3)P$}
&
&
&
\\
\textbf{degree}
&
$2$
&
$5$
&
&
&
\\
\hline

\rule{0pt}{2.3cm}
\raisebox{0.75cm}{$\delta=0$, $d=1$}
&
\tikzset{every picture/.style={line width=0.75pt}} 

\begin{tikzpicture}[x=0.4pt,y=0.4pt,yscale=-1,xscale=1]

\draw   (242.25,92) -- (388,92) -- (388,237.75) -- (242.25,237.75) -- cycle ;
\draw [color={rgb, 255:red, 208; green, 2; blue, 27 }  ,draw opacity=1 ]   (317.67,105) -- (256.33,198.33) ;
\draw  [fill={rgb, 255:red, 0; green, 0; blue, 0 }  ,fill opacity=1 ] (303.92,122.46) .. controls (303.92,123.77) and (304.98,124.83) .. (306.29,124.83) .. controls (307.6,124.83) and (308.67,123.77) .. (308.67,122.46) .. controls (308.67,121.15) and (307.6,120.08) .. (306.29,120.08) .. controls (304.98,120.08) and (303.92,121.15) .. (303.92,122.46) -- cycle ;
\draw [color={rgb, 255:red, 208; green, 2; blue, 27 }  ,draw opacity=1 ][fill={rgb, 255:red, 208; green, 2; blue, 27 }  ,fill opacity=1 ]   (373,139) -- (314.33,228.5) ;
\draw  [color={rgb, 255:red, 0; green, 0; blue, 0 }  ,draw opacity=1 ][fill={rgb, 255:red, 0; green, 0; blue, 0 }  ,fill opacity=1 ] (274.75,166.96) .. controls (274.75,168.27) and (275.81,169.33) .. (277.13,169.33) .. controls (278.44,169.33) and (279.5,168.27) .. (279.5,166.96) .. controls (279.5,165.65) and (278.44,164.58) .. (277.13,164.58) .. controls (275.81,164.58) and (274.75,165.65) .. (274.75,166.96) -- cycle ;
\draw  [color={rgb, 255:red, 0; green, 0; blue, 0 }  ,draw opacity=1 ][fill={rgb, 255:red, 0; green, 0; blue, 0 }  ,fill opacity=1 ] (288.75,145.46) .. controls (288.75,146.77) and (289.81,147.83) .. (291.13,147.83) .. controls (292.44,147.83) and (293.5,146.77) .. (293.5,145.46) .. controls (293.5,144.15) and (292.44,143.08) .. (291.13,143.08) .. controls (289.81,143.08) and (288.75,144.15) .. (288.75,145.46) -- cycle ;
\draw  [fill={rgb, 255:red, 0; green, 0; blue, 0 }  ,fill opacity=1 ] (365.25,147.79) .. controls (365.25,149.1) and (366.31,150.17) .. (367.63,150.17) .. controls (368.94,150.17) and (370,149.1) .. (370,147.79) .. controls (370,146.48) and (368.94,145.42) .. (367.63,145.42) .. controls (366.31,145.42) and (365.25,146.48) .. (365.25,147.79) -- cycle ;
\draw  [fill={rgb, 255:red, 0; green, 0; blue, 0 }  ,fill opacity=1 ] (328.75,203.63) .. controls (328.75,204.94) and (329.81,206) .. (331.13,206) .. controls (332.44,206) and (333.5,204.94) .. (333.5,203.63) .. controls (333.5,202.31) and (332.44,201.25) .. (331.13,201.25) .. controls (329.81,201.25) and (328.75,202.31) .. (328.75,203.63) -- cycle ;
\draw [color={rgb, 255:red, 208; green, 2; blue, 27 }  ,draw opacity=1 ]   (354.33,109) -- (277.5,223.58) ;
\draw  [color={rgb, 255:red, 0; green, 0; blue, 0 }  ,draw opacity=1 ][fill={rgb, 255:red, 0; green, 0; blue, 0 }  ,fill opacity=1 ] (330.42,141.12) .. controls (330.42,142.44) and (331.48,143.5) .. (332.79,143.5) .. controls (334.1,143.5) and (335.17,142.44) .. (335.17,141.12) .. controls (335.17,139.81) and (334.1,138.75) .. (332.79,138.75) .. controls (331.48,138.75) and (330.42,139.81) .. (330.42,141.12) -- cycle ;
\draw  [color={rgb, 255:red, 0; green, 0; blue, 0 }  ,draw opacity=1 ][fill={rgb, 255:red, 0; green, 0; blue, 0 }  ,fill opacity=1 ] (312.75,167.25) .. controls (312.75,168.56) and (313.81,169.63) .. (315.13,169.63) .. controls (316.44,169.63) and (317.5,168.56) .. (317.5,167.25) .. controls (317.5,165.94) and (316.44,164.88) .. (315.13,164.88) .. controls (313.81,164.88) and (312.75,165.94) .. (312.75,167.25) -- cycle ;
\draw  [color={rgb, 255:red, 0; green, 0; blue, 0 }  ,draw opacity=1 ][fill={rgb, 255:red, 0; green, 0; blue, 0 }  ,fill opacity=1 ] (261.75,185.96) .. controls (261.75,187.27) and (262.81,188.33) .. (264.13,188.33) .. controls (265.44,188.33) and (266.5,187.27) .. (266.5,185.96) .. controls (266.5,184.65) and (265.44,183.58) .. (264.13,183.58) .. controls (262.81,183.58) and (261.75,184.65) .. (261.75,185.96) -- cycle ;

\end{tikzpicture}
&
\tikzset{every picture/.style={line width=0.75pt}} 

\begin{tikzpicture}[x=0.4pt,y=0.4pt,yscale=-1,xscale=1]

\draw   (260.25,86) -- (406,86) -- (406,231.75) -- (260.25,231.75) -- cycle ;
\draw [color={rgb, 255:red, 208; green, 2; blue, 27 }  ,draw opacity=1 ]   (335.67,99) -- (274.33,192.33) ;
\draw  [fill={rgb, 255:red, 0; green, 0; blue, 0 }  ,fill opacity=1 ] (321.92,116.46) .. controls (321.92,117.77) and (322.98,118.83) .. (324.29,118.83) .. controls (325.6,118.83) and (326.67,117.77) .. (326.67,116.46) .. controls (326.67,115.15) and (325.6,114.08) .. (324.29,114.08) .. controls (322.98,114.08) and (321.92,115.15) .. (321.92,116.46) -- cycle ;
\draw [color={rgb, 255:red, 208; green, 2; blue, 27 }  ,draw opacity=1 ][fill={rgb, 255:red, 208; green, 2; blue, 27 }  ,fill opacity=1 ]   (391,133) -- (332.33,222.5) ;
\draw  [color={rgb, 255:red, 0; green, 0; blue, 0 }  ,draw opacity=1 ][fill={rgb, 255:red, 0; green, 0; blue, 0 }  ,fill opacity=1 ] (295.95,155.76) .. controls (295.95,157.07) and (297.01,158.13) .. (298.33,158.13) .. controls (299.64,158.13) and (300.7,157.07) .. (300.7,155.76) .. controls (300.7,154.45) and (299.64,153.38) .. (298.33,153.38) .. controls (297.01,153.38) and (295.95,154.45) .. (295.95,155.76) -- cycle ;
\draw  [fill={rgb, 255:red, 0; green, 0; blue, 0 }  ,fill opacity=1 ] (371.65,159.39) .. controls (371.65,160.7) and (372.71,161.77) .. (374.03,161.77) .. controls (375.34,161.77) and (376.4,160.7) .. (376.4,159.39) .. controls (376.4,158.08) and (375.34,157.02) .. (374.03,157.02) .. controls (372.71,157.02) and (371.65,158.08) .. (371.65,159.39) -- cycle ;
\draw  [fill={rgb, 255:red, 0; green, 0; blue, 0 }  ,fill opacity=1 ] (348.75,193.63) .. controls (348.75,194.94) and (349.81,196) .. (351.13,196) .. controls (352.44,196) and (353.5,194.94) .. (353.5,193.63) .. controls (353.5,192.31) and (352.44,191.25) .. (351.13,191.25) .. controls (349.81,191.25) and (348.75,192.31) .. (348.75,193.63) -- cycle ;
\draw [color={rgb, 255:red, 208; green, 2; blue, 27 }  ,draw opacity=1 ]   (372.33,103) -- (295.5,217.58) ;
\draw  [color={rgb, 255:red, 0; green, 0; blue, 0 }  ,draw opacity=1 ][fill={rgb, 255:red, 0; green, 0; blue, 0 }  ,fill opacity=1 ] (356.82,122.72) .. controls (356.82,124.04) and (357.88,125.1) .. (359.19,125.1) .. controls (360.5,125.1) and (361.57,124.04) .. (361.57,122.72) .. controls (361.57,121.41) and (360.5,120.35) .. (359.19,120.35) .. controls (357.88,120.35) and (356.82,121.41) .. (356.82,122.72) -- cycle ;
\draw  [color={rgb, 255:red, 0; green, 0; blue, 0 }  ,draw opacity=1 ][fill={rgb, 255:red, 0; green, 0; blue, 0 }  ,fill opacity=1 ] (339.95,148.05) .. controls (339.95,149.36) and (341.01,150.42) .. (342.33,150.42) .. controls (343.64,150.42) and (344.7,149.36) .. (344.7,148.05) .. controls (344.7,146.74) and (343.64,145.67) .. (342.33,145.67) .. controls (341.01,145.67) and (339.95,146.74) .. (339.95,148.05) -- cycle ;
\draw  [color={rgb, 255:red, 0; green, 0; blue, 0 }  ,draw opacity=1 ][fill={rgb, 255:red, 0; green, 0; blue, 0 }  ,fill opacity=1 ] (279.75,179.96) .. controls (279.75,181.27) and (280.81,182.33) .. (282.13,182.33) .. controls (283.44,182.33) and (284.5,181.27) .. (284.5,179.96) .. controls (284.5,178.65) and (283.44,177.58) .. (282.13,177.58) .. controls (280.81,177.58) and (279.75,178.65) .. (279.75,179.96) -- cycle ;
\draw  [color={rgb, 255:red, 0; green, 0; blue, 0 }  ,draw opacity=1 ][fill={rgb, 255:red, 0; green, 0; blue, 0 }  ,fill opacity=1 ] (311.55,190.26) .. controls (311.55,191.57) and (312.61,192.63) .. (313.93,192.63) .. controls (315.24,192.63) and (316.3,191.57) .. (316.3,190.26) .. controls (316.3,188.95) and (315.24,187.88) .. (313.93,187.88) .. controls (312.61,187.88) and (311.55,188.95) .. (311.55,190.26) -- cycle ;

\end{tikzpicture}
&
\tikzset{every picture/.style={line width=0.75pt}} 

\begin{tikzpicture}[x=0.4pt,y=0.4pt,yscale=-1,xscale=1]

\draw   (267.25,87) -- (413,87) -- (413,232.75) -- (267.25,232.75) -- cycle ;
\draw [color={rgb, 255:red, 208; green, 2; blue, 27 }  ,draw opacity=1 ]   (339,95) -- (277.67,188.33) ;
\draw [color={rgb, 255:red, 208; green, 2; blue, 27 }  ,draw opacity=1 ][fill={rgb, 255:red, 208; green, 2; blue, 27 }  ,fill opacity=1 ]   (402.67,136) -- (344,225.5) ;
\draw [color={rgb, 255:red, 208; green, 2; blue, 27 }  ,draw opacity=1 ]   (392.33,105.67) -- (327.67,205.67) ;
\draw [color={rgb, 255:red, 208; green, 2; blue, 27 }  ,draw opacity=1 ]   (351.33,121) -- (286.67,221) ;
\draw  [color={rgb, 255:red, 0; green, 0; blue, 0 }  ,draw opacity=1 ][fill={rgb, 255:red, 0; green, 0; blue, 0 }  ,fill opacity=1 ] (329.42,105.99) .. controls (329.42,107.3) and (330.48,108.37) .. (331.79,108.37) .. controls (333.1,108.37) and (334.17,107.3) .. (334.17,105.99) .. controls (334.17,104.68) and (333.1,103.62) .. (331.79,103.62) .. controls (330.48,103.62) and (329.42,104.68) .. (329.42,105.99) -- cycle ;
\draw  [color={rgb, 255:red, 0; green, 0; blue, 0 }  ,draw opacity=1 ][fill={rgb, 255:red, 0; green, 0; blue, 0 }  ,fill opacity=1 ] (314.42,128.79) .. controls (314.42,130.1) and (315.48,131.17) .. (316.79,131.17) .. controls (318.1,131.17) and (319.17,130.1) .. (319.17,128.79) .. controls (319.17,127.48) and (318.1,126.42) .. (316.79,126.42) .. controls (315.48,126.42) and (314.42,127.48) .. (314.42,128.79) -- cycle ;
\draw  [color={rgb, 255:red, 0; green, 0; blue, 0 }  ,draw opacity=1 ][fill={rgb, 255:red, 0; green, 0; blue, 0 }  ,fill opacity=1 ] (287.42,169.99) .. controls (287.42,171.3) and (288.48,172.37) .. (289.79,172.37) .. controls (291.1,172.37) and (292.17,171.3) .. (292.17,169.99) .. controls (292.17,168.68) and (291.1,167.62) .. (289.79,167.62) .. controls (288.48,167.62) and (287.42,168.68) .. (287.42,169.99) -- cycle ;
\draw  [color={rgb, 255:red, 0; green, 0; blue, 0 }  ,draw opacity=1 ][fill={rgb, 255:red, 0; green, 0; blue, 0 }  ,fill opacity=1 ] (340.22,133.99) .. controls (340.22,135.3) and (341.28,136.37) .. (342.59,136.37) .. controls (343.9,136.37) and (344.97,135.3) .. (344.97,133.99) .. controls (344.97,132.68) and (343.9,131.62) .. (342.59,131.62) .. controls (341.28,131.62) and (340.22,132.68) .. (340.22,133.99) -- cycle ;
\draw  [color={rgb, 255:red, 0; green, 0; blue, 0 }  ,draw opacity=1 ][fill={rgb, 255:red, 0; green, 0; blue, 0 }  ,fill opacity=1 ] (319.22,166.99) .. controls (319.22,168.3) and (320.28,169.37) .. (321.59,169.37) .. controls (322.9,169.37) and (323.97,168.3) .. (323.97,166.99) .. controls (323.97,165.68) and (322.9,164.62) .. (321.59,164.62) .. controls (320.28,164.62) and (319.22,165.68) .. (319.22,166.99) -- cycle ;
\draw  [color={rgb, 255:red, 0; green, 0; blue, 0 }  ,draw opacity=1 ][fill={rgb, 255:red, 0; green, 0; blue, 0 }  ,fill opacity=1 ] (371.62,134.19) .. controls (371.62,135.5) and (372.68,136.57) .. (373.99,136.57) .. controls (375.3,136.57) and (376.37,135.5) .. (376.37,134.19) .. controls (376.37,132.88) and (375.3,131.82) .. (373.99,131.82) .. controls (372.68,131.82) and (371.62,132.88) .. (371.62,134.19) -- cycle ;
\draw  [color={rgb, 255:red, 0; green, 0; blue, 0 }  ,draw opacity=1 ][fill={rgb, 255:red, 0; green, 0; blue, 0 }  ,fill opacity=1 ] (351.02,165.59) .. controls (351.02,166.9) and (352.08,167.97) .. (353.39,167.97) .. controls (354.7,167.97) and (355.77,166.9) .. (355.77,165.59) .. controls (355.77,164.28) and (354.7,163.22) .. (353.39,163.22) .. controls (352.08,163.22) and (351.02,164.28) .. (351.02,165.59) -- cycle ;
\draw  [color={rgb, 255:red, 0; green, 0; blue, 0 }  ,draw opacity=1 ][fill={rgb, 255:red, 0; green, 0; blue, 0 }  ,fill opacity=1 ] (383.82,161.19) .. controls (383.82,162.5) and (384.88,163.57) .. (386.19,163.57) .. controls (387.5,163.57) and (388.57,162.5) .. (388.57,161.19) .. controls (388.57,159.88) and (387.5,158.82) .. (386.19,158.82) .. controls (384.88,158.82) and (383.82,159.88) .. (383.82,161.19) -- cycle ;
\draw  [color={rgb, 255:red, 0; green, 0; blue, 0 }  ,draw opacity=1 ][fill={rgb, 255:red, 0; green, 0; blue, 0 }  ,fill opacity=1 ] (355.62,204.19) .. controls (355.62,205.5) and (356.68,206.57) .. (357.99,206.57) .. controls (359.3,206.57) and (360.37,205.5) .. (360.37,204.19) .. controls (360.37,202.88) and (359.3,201.82) .. (357.99,201.82) .. controls (356.68,201.82) and (355.62,202.88) .. (355.62,204.19) -- cycle ;

\end{tikzpicture}
&
\tikzset{every picture/.style={line width=0.75pt}} 

\begin{tikzpicture}[x=0.4pt,y=0.4pt,yscale=-1,xscale=1]

\draw   (270.25,87) -- (416,87) -- (416,232.75) -- (270.25,232.75) -- cycle ;
\draw [color={rgb, 255:red, 208; green, 2; blue, 27 }  ,draw opacity=1 ]   (337.4,97) -- (278.67,185.93) ;
\draw [color={rgb, 255:red, 208; green, 2; blue, 27 }  ,draw opacity=1 ][fill={rgb, 255:red, 208; green, 2; blue, 27 }  ,fill opacity=1 ]   (410,131) -- (353.6,216.6) ;
\draw [color={rgb, 255:red, 208; green, 2; blue, 27 }  ,draw opacity=1 ]   (394.6,117.8) -- (321.87,226.07) ;
\draw [color={rgb, 255:red, 208; green, 2; blue, 27 }  ,draw opacity=1 ]   (347.53,116.6) -- (282.87,216.6) ;
\draw  [color={rgb, 255:red, 0; green, 0; blue, 0 }  ,draw opacity=1 ][fill={rgb, 255:red, 0; green, 0; blue, 0 }  ,fill opacity=1 ] (327.62,107.99) .. controls (327.62,109.3) and (328.68,110.37) .. (329.99,110.37) .. controls (331.3,110.37) and (332.37,109.3) .. (332.37,107.99) .. controls (332.37,106.68) and (331.3,105.62) .. (329.99,105.62) .. controls (328.68,105.62) and (327.62,106.68) .. (327.62,107.99) -- cycle ;
\draw [color={rgb, 255:red, 208; green, 2; blue, 27 }  ,draw opacity=1 ]   (386.6,93) -- (312.2,207) ;
\draw  [color={rgb, 255:red, 0; green, 0; blue, 0 }  ,draw opacity=1 ][fill={rgb, 255:red, 0; green, 0; blue, 0 }  ,fill opacity=1 ] (302.02,146.79) .. controls (302.02,148.1) and (303.08,149.17) .. (304.39,149.17) .. controls (305.7,149.17) and (306.77,148.1) .. (306.77,146.79) .. controls (306.77,145.48) and (305.7,144.42) .. (304.39,144.42) .. controls (303.08,144.42) and (302.02,145.48) .. (302.02,146.79) -- cycle ;
\draw  [color={rgb, 255:red, 0; green, 0; blue, 0 }  ,draw opacity=1 ][fill={rgb, 255:red, 0; green, 0; blue, 0 }  ,fill opacity=1 ] (331.62,137.99) .. controls (331.62,139.3) and (332.68,140.37) .. (333.99,140.37) .. controls (335.3,140.37) and (336.37,139.3) .. (336.37,137.99) .. controls (336.37,136.68) and (335.3,135.62) .. (333.99,135.62) .. controls (332.68,135.62) and (331.62,136.68) .. (331.62,137.99) -- cycle ;
\draw  [color={rgb, 255:red, 0; green, 0; blue, 0 }  ,draw opacity=1 ][fill={rgb, 255:red, 0; green, 0; blue, 0 }  ,fill opacity=1 ] (302.42,183.19) .. controls (302.42,184.5) and (303.48,185.57) .. (304.79,185.57) .. controls (306.1,185.57) and (307.17,184.5) .. (307.17,183.19) .. controls (307.17,181.88) and (306.1,180.82) .. (304.79,180.82) .. controls (303.48,180.82) and (302.42,181.88) .. (302.42,183.19) -- cycle ;
\draw  [color={rgb, 255:red, 0; green, 0; blue, 0 }  ,draw opacity=1 ][fill={rgb, 255:red, 0; green, 0; blue, 0 }  ,fill opacity=1 ] (336.02,166.79) .. controls (336.02,168.1) and (337.08,169.17) .. (338.39,169.17) .. controls (339.7,169.17) and (340.77,168.1) .. (340.77,166.79) .. controls (340.77,165.48) and (339.7,164.42) .. (338.39,164.42) .. controls (337.08,164.42) and (336.02,165.48) .. (336.02,166.79) -- cycle ;
\draw  [color={rgb, 255:red, 0; green, 0; blue, 0 }  ,draw opacity=1 ][fill={rgb, 255:red, 0; green, 0; blue, 0 }  ,fill opacity=1 ] (372.82,110.79) .. controls (372.82,112.1) and (373.88,113.17) .. (375.19,113.17) .. controls (376.5,113.17) and (377.57,112.1) .. (377.57,110.79) .. controls (377.57,109.48) and (376.5,108.42) .. (375.19,108.42) .. controls (373.88,108.42) and (372.82,109.48) .. (372.82,110.79) -- cycle ;
\draw  [color={rgb, 255:red, 0; green, 0; blue, 0 }  ,draw opacity=1 ][fill={rgb, 255:red, 0; green, 0; blue, 0 }  ,fill opacity=1 ] (334.42,203.59) .. controls (334.42,204.9) and (335.48,205.97) .. (336.79,205.97) .. controls (338.1,205.97) and (339.17,204.9) .. (339.17,203.59) .. controls (339.17,202.28) and (338.1,201.22) .. (336.79,201.22) .. controls (335.48,201.22) and (334.42,202.28) .. (334.42,203.59) -- cycle ;
\draw  [color={rgb, 255:red, 0; green, 0; blue, 0 }  ,draw opacity=1 ][fill={rgb, 255:red, 0; green, 0; blue, 0 }  ,fill opacity=1 ] (363.22,161.59) .. controls (363.22,162.9) and (364.28,163.97) .. (365.59,163.97) .. controls (366.9,163.97) and (367.97,162.9) .. (367.97,161.59) .. controls (367.97,160.28) and (366.9,159.22) .. (365.59,159.22) .. controls (364.28,159.22) and (363.22,160.28) .. (363.22,161.59) -- cycle ;
\draw  [color={rgb, 255:red, 0; green, 0; blue, 0 }  ,draw opacity=1 ][fill={rgb, 255:red, 0; green, 0; blue, 0 }  ,fill opacity=1 ] (375.62,179.59) .. controls (375.62,180.9) and (376.68,181.97) .. (377.99,181.97) .. controls (379.3,181.97) and (380.37,180.9) .. (380.37,179.59) .. controls (380.37,178.28) and (379.3,177.22) .. (377.99,177.22) .. controls (376.68,177.22) and (375.62,178.28) .. (375.62,179.59) -- cycle ;
\draw  [color={rgb, 255:red, 0; green, 0; blue, 0 }  ,draw opacity=1 ][fill={rgb, 255:red, 0; green, 0; blue, 0 }  ,fill opacity=1 ] (399.22,143.59) .. controls (399.22,144.9) and (400.28,145.97) .. (401.59,145.97) .. controls (402.9,145.97) and (403.97,144.9) .. (403.97,143.59) .. controls (403.97,142.28) and (402.9,141.22) .. (401.59,141.22) .. controls (400.28,141.22) and (399.22,142.28) .. (399.22,143.59) -- cycle ;

\end{tikzpicture}
&
\\
\textbf{lable}
&
$d_1(4,2^2)PC$
&
\underline{$d_1(3^2,2)PC$}
&
$d_1(3,2^3)PC$
&
\underline{$d_1(2^5)PC$}
&
\\
\textbf{degree}
&
$2$
&
$4$
&
$6$
&
$10$
&
\\
\hline

\rule{0pt}{2.3cm}
\raisebox{0.75cm}{$\delta=0$, $d=2$}
&
\tikzset{every picture/.style={line width=0.75pt}} 

\begin{tikzpicture}[x=0.4pt,y=0.4pt,yscale=-1,xscale=1]

\draw   (290.25,107) -- (436,107) -- (436,252.75) -- (290.25,252.75) -- cycle ;
\draw [color={rgb, 255:red, 208; green, 2; blue, 27 }  ,draw opacity=1 ]   (363,119) -- (299.07,210.33) ;
\draw [color={rgb, 255:red, 208; green, 2; blue, 27 }  ,draw opacity=1 ]   (379,129.4) -- (345.4,241.8) ;
\draw  [color={rgb, 255:red, 0; green, 0; blue, 0 }  ,draw opacity=1 ][fill={rgb, 255:red, 0; green, 0; blue, 0 }  ,fill opacity=1 ] (337.22,152.39) .. controls (337.22,153.7) and (338.28,154.77) .. (339.59,154.77) .. controls (340.9,154.77) and (341.97,153.7) .. (341.97,152.39) .. controls (341.97,151.08) and (340.9,150.02) .. (339.59,150.02) .. controls (338.28,150.02) and (337.22,151.08) .. (337.22,152.39) -- cycle ;
\draw [color={rgb, 255:red, 208; green, 2; blue, 27 }  ,draw opacity=1 ]   (397.4,119) -- (415.4,231.4) ;
\draw  [color={rgb, 255:red, 0; green, 0; blue, 0 }  ,draw opacity=1 ][fill={rgb, 255:red, 0; green, 0; blue, 0 }  ,fill opacity=1 ] (311.62,189.19) .. controls (311.62,190.5) and (312.68,191.57) .. (313.99,191.57) .. controls (315.3,191.57) and (316.37,190.5) .. (316.37,189.19) .. controls (316.37,187.88) and (315.3,186.82) .. (313.99,186.82) .. controls (312.68,186.82) and (311.62,187.88) .. (311.62,189.19) -- cycle ;
\draw  [color={rgb, 255:red, 0; green, 0; blue, 0 }  ,draw opacity=1 ][fill={rgb, 255:red, 0; green, 0; blue, 0 }  ,fill opacity=1 ] (324.02,171.99) .. controls (324.02,173.3) and (325.08,174.37) .. (326.39,174.37) .. controls (327.7,174.37) and (328.77,173.3) .. (328.77,171.99) .. controls (328.77,170.68) and (327.7,169.62) .. (326.39,169.62) .. controls (325.08,169.62) and (324.02,170.68) .. (324.02,171.99) -- cycle ;
\draw  [color={rgb, 255:red, 0; green, 0; blue, 0 }  ,draw opacity=1 ][fill={rgb, 255:red, 0; green, 0; blue, 0 }  ,fill opacity=1 ] (300.82,204.39) .. controls (300.82,205.7) and (301.88,206.77) .. (303.19,206.77) .. controls (304.5,206.77) and (305.57,205.7) .. (305.57,204.39) .. controls (305.57,203.08) and (304.5,202.02) .. (303.19,202.02) .. controls (301.88,202.02) and (300.82,203.08) .. (300.82,204.39) -- cycle ;
\draw  [color={rgb, 255:red, 0; green, 0; blue, 0 }  ,draw opacity=1 ][fill={rgb, 255:red, 0; green, 0; blue, 0 }  ,fill opacity=1 ] (344.02,143.19) .. controls (344.02,144.5) and (345.08,145.57) .. (346.39,145.57) .. controls (347.7,145.57) and (348.77,144.5) .. (348.77,143.19) .. controls (348.77,141.88) and (347.7,140.82) .. (346.39,140.82) .. controls (345.08,140.82) and (344.02,141.88) .. (344.02,143.19) -- cycle ;
\draw  [color={rgb, 255:red, 0; green, 0; blue, 0 }  ,draw opacity=1 ][fill={rgb, 255:red, 0; green, 0; blue, 0 }  ,fill opacity=1 ] (374.27,137.74) .. controls (374.27,139.05) and (375.33,140.12) .. (376.64,140.12) .. controls (377.95,140.12) and (379.02,139.05) .. (379.02,137.74) .. controls (379.02,136.43) and (377.95,135.37) .. (376.64,135.37) .. controls (375.33,135.37) and (374.27,136.43) .. (374.27,137.74) -- cycle ;
\draw  [color={rgb, 255:red, 0; green, 0; blue, 0 }  ,draw opacity=1 ][fill={rgb, 255:red, 0; green, 0; blue, 0 }  ,fill opacity=1 ] (358.27,191.49) .. controls (358.27,192.8) and (359.33,193.87) .. (360.64,193.87) .. controls (361.95,193.87) and (363.02,192.8) .. (363.02,191.49) .. controls (363.02,190.18) and (361.95,189.12) .. (360.64,189.12) .. controls (359.33,189.12) and (358.27,190.18) .. (358.27,191.49) -- cycle ;
\draw  [color={rgb, 255:red, 0; green, 0; blue, 0 }  ,draw opacity=1 ][fill={rgb, 255:red, 0; green, 0; blue, 0 }  ,fill opacity=1 ] (369.97,151.89) .. controls (369.97,153.2) and (371.03,154.27) .. (372.34,154.27) .. controls (373.65,154.27) and (374.72,153.2) .. (374.72,151.89) .. controls (374.72,150.58) and (373.65,149.52) .. (372.34,149.52) .. controls (371.03,149.52) and (369.97,150.58) .. (369.97,151.89) -- cycle ;
\draw  [color={rgb, 255:red, 0; green, 0; blue, 0 }  ,draw opacity=1 ][fill={rgb, 255:red, 0; green, 0; blue, 0 }  ,fill opacity=1 ] (353.42,129.19) .. controls (353.42,130.5) and (354.48,131.57) .. (355.79,131.57) .. controls (357.1,131.57) and (358.17,130.5) .. (358.17,129.19) .. controls (358.17,127.88) and (357.1,126.82) .. (355.79,126.82) .. controls (354.48,126.82) and (353.42,127.88) .. (353.42,129.19) -- cycle ;
\draw  [color={rgb, 255:red, 0; green, 0; blue, 0 }  ,draw opacity=1 ][fill={rgb, 255:red, 0; green, 0; blue, 0 }  ,fill opacity=1 ] (362.72,175.14) .. controls (362.72,176.45) and (363.78,177.52) .. (365.09,177.52) .. controls (366.4,177.52) and (367.47,176.45) .. (367.47,175.14) .. controls (367.47,173.83) and (366.4,172.77) .. (365.09,172.77) .. controls (363.78,172.77) and (362.72,173.83) .. (362.72,175.14) -- cycle ;
\draw  [color={rgb, 255:red, 0; green, 0; blue, 0 }  ,draw opacity=1 ][fill={rgb, 255:red, 0; green, 0; blue, 0 }  ,fill opacity=1 ] (350.77,215.99) .. controls (350.77,217.3) and (351.83,218.37) .. (353.14,218.37) .. controls (354.45,218.37) and (355.52,217.3) .. (355.52,215.99) .. controls (355.52,214.68) and (354.45,213.62) .. (353.14,213.62) .. controls (351.83,213.62) and (350.77,214.68) .. (350.77,215.99) -- cycle ;
\draw  [color={rgb, 255:red, 0; green, 0; blue, 0 }  ,draw opacity=1 ][fill={rgb, 255:red, 0; green, 0; blue, 0 }  ,fill opacity=1 ] (347.27,228.24) .. controls (347.27,229.55) and (348.33,230.62) .. (349.64,230.62) .. controls (350.95,230.62) and (352.02,229.55) .. (352.02,228.24) .. controls (352.02,226.93) and (350.95,225.87) .. (349.64,225.87) .. controls (348.33,225.87) and (347.27,226.93) .. (347.27,228.24) -- cycle ;
\draw  [color={rgb, 255:red, 0; green, 0; blue, 0 }  ,draw opacity=1 ][fill={rgb, 255:red, 0; green, 0; blue, 0 }  ,fill opacity=1 ] (396.77,128.24) .. controls (396.77,129.55) and (397.83,130.62) .. (399.14,130.62) .. controls (400.45,130.62) and (401.52,129.55) .. (401.52,128.24) .. controls (401.52,126.93) and (400.45,125.87) .. (399.14,125.87) .. controls (397.83,125.87) and (396.77,126.93) .. (396.77,128.24) -- cycle ;
\draw  [color={rgb, 255:red, 0; green, 0; blue, 0 }  ,draw opacity=1 ][fill={rgb, 255:red, 0; green, 0; blue, 0 }  ,fill opacity=1 ] (404.03,175.2) .. controls (404.03,176.51) and (405.09,177.58) .. (406.4,177.58) .. controls (407.71,177.58) and (408.78,176.51) .. (408.78,175.2) .. controls (408.78,173.89) and (407.71,172.83) .. (406.4,172.83) .. controls (405.09,172.83) and (404.03,173.89) .. (404.03,175.2) -- cycle ;
\draw  [color={rgb, 255:red, 0; green, 0; blue, 0 }  ,draw opacity=1 ][fill={rgb, 255:red, 0; green, 0; blue, 0 }  ,fill opacity=1 ] (399.77,148.49) .. controls (399.77,149.8) and (400.83,150.87) .. (402.14,150.87) .. controls (403.45,150.87) and (404.52,149.8) .. (404.52,148.49) .. controls (404.52,147.18) and (403.45,146.12) .. (402.14,146.12) .. controls (400.83,146.12) and (399.77,147.18) .. (399.77,148.49) -- cycle ;
\draw  [color={rgb, 255:red, 0; green, 0; blue, 0 }  ,draw opacity=1 ][fill={rgb, 255:red, 0; green, 0; blue, 0 }  ,fill opacity=1 ] (408.02,198.74) .. controls (408.02,200.05) and (409.08,201.12) .. (410.39,201.12) .. controls (411.7,201.12) and (412.77,200.05) .. (412.77,198.74) .. controls (412.77,197.43) and (411.7,196.37) .. (410.39,196.37) .. controls (409.08,196.37) and (408.02,197.43) .. (408.02,198.74) -- cycle ;
\draw  [color={rgb, 255:red, 0; green, 0; blue, 0 }  ,draw opacity=1 ][fill={rgb, 255:red, 0; green, 0; blue, 0 }  ,fill opacity=1 ] (411.02,219.24) .. controls (411.02,220.55) and (412.08,221.62) .. (413.39,221.62) .. controls (414.7,221.62) and (415.77,220.55) .. (415.77,219.24) .. controls (415.77,217.93) and (414.7,216.87) .. (413.39,216.87) .. controls (412.08,216.87) and (411.02,217.93) .. (411.02,219.24) -- cycle ;

\end{tikzpicture}
&
\input{Tikz/d2_6555_}
&
\input{Tikz/d2_55555_}
&
&
\\
\textbf{lable}
&
\underline{$d_2(6^2,5)$}
&
$d_2(6,5^3)$
&
$d_2(5^5)$
&
&
\\
\textbf{degree}
&
$30$
&
$131$
&
$680$
&
&
\\
\hline

\rule{0pt}{2.3cm}
\raisebox{0.75cm}{$\delta=0$, $d=3$}
&
\tikzset{every picture/.style={line width=0.75pt}} 

\begin{tikzpicture}[x=0.4pt,y=0.4pt,yscale=-1,xscale=1]

\draw   (300.25,91) -- (446,91) -- (446,236.75) -- (300.25,236.75) -- cycle ;
\draw [color={rgb, 255:red, 208; green, 2; blue, 27 }  ,draw opacity=1 ]   (397,115) -- (315.63,222.89) ;
\draw [color={rgb, 255:red, 208; green, 2; blue, 27 }  ,draw opacity=1 ]   (420.6,102.6) -- (400.2,220.6) ;
\draw  [color={rgb, 255:red, 0; green, 0; blue, 0 }  ,draw opacity=1 ][fill={rgb, 255:red, 0; green, 0; blue, 0 }  ,fill opacity=1 ] (386.88,125.03) .. controls (386.88,126.34) and (387.94,127.4) .. (389.26,127.4) .. controls (390.57,127.4) and (391.63,126.34) .. (391.63,125.03) .. controls (391.63,123.72) and (390.57,122.65) .. (389.26,122.65) .. controls (387.94,122.65) and (386.88,123.72) .. (386.88,125.03) -- cycle ;
\draw  [color={rgb, 255:red, 0; green, 0; blue, 0 }  ,draw opacity=1 ][fill={rgb, 255:red, 0; green, 0; blue, 0 }  ,fill opacity=1 ] (379.6,135.27) .. controls (379.6,136.58) and (380.67,137.65) .. (381.98,137.65) .. controls (383.29,137.65) and (384.35,136.58) .. (384.35,135.27) .. controls (384.35,133.96) and (383.29,132.9) .. (381.98,132.9) .. controls (380.67,132.9) and (379.6,133.96) .. (379.6,135.27) -- cycle ;
\draw  [color={rgb, 255:red, 0; green, 0; blue, 0 }  ,draw opacity=1 ][fill={rgb, 255:red, 0; green, 0; blue, 0 }  ,fill opacity=1 ] (360.5,160.34) .. controls (360.5,161.65) and (361.57,162.72) .. (362.88,162.72) .. controls (364.19,162.72) and (365.25,161.65) .. (365.25,160.34) .. controls (365.25,159.03) and (364.19,157.97) .. (362.88,157.97) .. controls (361.57,157.97) and (360.5,159.03) .. (360.5,160.34) -- cycle ;
\draw  [color={rgb, 255:red, 0; green, 0; blue, 0 }  ,draw opacity=1 ][fill={rgb, 255:red, 0; green, 0; blue, 0 }  ,fill opacity=1 ] (353.94,168.94) .. controls (353.94,170.25) and (355,171.32) .. (356.31,171.32) .. controls (357.63,171.32) and (358.69,170.25) .. (358.69,168.94) .. controls (358.69,167.63) and (357.63,166.57) .. (356.31,166.57) .. controls (355,166.57) and (353.94,167.63) .. (353.94,168.94) -- cycle ;
\draw  [color={rgb, 255:red, 0; green, 0; blue, 0 }  ,draw opacity=1 ][fill={rgb, 255:red, 0; green, 0; blue, 0 }  ,fill opacity=1 ] (369.91,147.98) .. controls (369.91,149.3) and (370.97,150.36) .. (372.28,150.36) .. controls (373.6,150.36) and (374.66,149.3) .. (374.66,147.98) .. controls (374.66,146.67) and (373.6,145.61) .. (372.28,145.61) .. controls (370.97,145.61) and (369.91,146.67) .. (369.91,147.98) -- cycle ;
\draw  [color={rgb, 255:red, 0; green, 0; blue, 0 }  ,draw opacity=1 ][fill={rgb, 255:red, 0; green, 0; blue, 0 }  ,fill opacity=1 ] (416.71,110.5) .. controls (416.71,111.81) and (417.77,112.87) .. (419.08,112.87) .. controls (420.4,112.87) and (421.46,111.81) .. (421.46,110.5) .. controls (421.46,109.19) and (420.4,108.12) .. (419.08,108.12) .. controls (417.77,108.12) and (416.71,109.19) .. (416.71,110.5) -- cycle ;
\draw  [color={rgb, 255:red, 0; green, 0; blue, 0 }  ,draw opacity=1 ][fill={rgb, 255:red, 0; green, 0; blue, 0 }  ,fill opacity=1 ] (345.63,180.34) .. controls (345.63,181.65) and (346.69,182.72) .. (348.01,182.72) .. controls (349.32,182.72) and (350.38,181.65) .. (350.38,180.34) .. controls (350.38,179.03) and (349.32,177.97) .. (348.01,177.97) .. controls (346.69,177.97) and (345.63,179.03) .. (345.63,180.34) -- cycle ;
\draw  [color={rgb, 255:red, 0; green, 0; blue, 0 }  ,draw opacity=1 ][fill={rgb, 255:red, 0; green, 0; blue, 0 }  ,fill opacity=1 ] (333.2,197.48) .. controls (333.2,198.8) and (334.27,199.86) .. (335.58,199.86) .. controls (336.89,199.86) and (337.95,198.8) .. (337.95,197.48) .. controls (337.95,196.17) and (336.89,195.11) .. (335.58,195.11) .. controls (334.27,195.11) and (333.2,196.17) .. (333.2,197.48) -- cycle ;
\draw  [color={rgb, 255:red, 0; green, 0; blue, 0 }  ,draw opacity=1 ][fill={rgb, 255:red, 0; green, 0; blue, 0 }  ,fill opacity=1 ] (321.42,212.73) .. controls (321.42,214.05) and (322.48,215.11) .. (323.79,215.11) .. controls (325.1,215.11) and (326.17,214.05) .. (326.17,212.73) .. controls (326.17,211.42) and (325.1,210.36) .. (323.79,210.36) .. controls (322.48,210.36) and (321.42,211.42) .. (321.42,212.73) -- cycle ;
\draw  [color={rgb, 255:red, 0; green, 0; blue, 0 }  ,draw opacity=1 ][fill={rgb, 255:red, 0; green, 0; blue, 0 }  ,fill opacity=1 ] (414.81,121.91) .. controls (414.81,123.22) and (415.87,124.29) .. (417.18,124.29) .. controls (418.5,124.29) and (419.56,123.22) .. (419.56,121.91) .. controls (419.56,120.6) and (418.5,119.54) .. (417.18,119.54) .. controls (415.87,119.54) and (414.81,120.6) .. (414.81,121.91) -- cycle ;
\draw  [color={rgb, 255:red, 0; green, 0; blue, 0 }  ,draw opacity=1 ][fill={rgb, 255:red, 0; green, 0; blue, 0 }  ,fill opacity=1 ] (412,138.66) .. controls (412,139.97) and (413.06,141.04) .. (414.38,141.04) .. controls (415.69,141.04) and (416.75,139.97) .. (416.75,138.66) .. controls (416.75,137.35) and (415.69,136.29) .. (414.38,136.29) .. controls (413.06,136.29) and (412,137.35) .. (412,138.66) -- cycle ;
\draw  [color={rgb, 255:red, 0; green, 0; blue, 0 }  ,draw opacity=1 ][fill={rgb, 255:red, 0; green, 0; blue, 0 }  ,fill opacity=1 ] (409.75,151.91) .. controls (409.75,153.22) and (410.81,154.29) .. (412.13,154.29) .. controls (413.44,154.29) and (414.5,153.22) .. (414.5,151.91) .. controls (414.5,150.6) and (413.44,149.54) .. (412.13,149.54) .. controls (410.81,149.54) and (409.75,150.6) .. (409.75,151.91) -- cycle ;
\draw  [color={rgb, 255:red, 0; green, 0; blue, 0 }  ,draw opacity=1 ][fill={rgb, 255:red, 0; green, 0; blue, 0 }  ,fill opacity=1 ] (406.5,170.16) .. controls (406.5,171.47) and (407.56,172.54) .. (408.88,172.54) .. controls (410.19,172.54) and (411.25,171.47) .. (411.25,170.16) .. controls (411.25,168.85) and (410.19,167.79) .. (408.88,167.79) .. controls (407.56,167.79) and (406.5,168.85) .. (406.5,170.16) -- cycle ;
\draw  [color={rgb, 255:red, 0; green, 0; blue, 0 }  ,draw opacity=1 ][fill={rgb, 255:red, 0; green, 0; blue, 0 }  ,fill opacity=1 ] (405,178.91) .. controls (405,180.22) and (406.06,181.29) .. (407.38,181.29) .. controls (408.69,181.29) and (409.75,180.22) .. (409.75,178.91) .. controls (409.75,177.6) and (408.69,176.54) .. (407.38,176.54) .. controls (406.06,176.54) and (405,177.6) .. (405,178.91) -- cycle ;
\draw  [color={rgb, 255:red, 0; green, 0; blue, 0 }  ,draw opacity=1 ][fill={rgb, 255:red, 0; green, 0; blue, 0 }  ,fill opacity=1 ] (402,196.91) .. controls (402,198.22) and (403.06,199.29) .. (404.38,199.29) .. controls (405.69,199.29) and (406.75,198.22) .. (406.75,196.91) .. controls (406.75,195.6) and (405.69,194.54) .. (404.38,194.54) .. controls (403.06,194.54) and (402,195.6) .. (402,196.91) -- cycle ;
\draw  [color={rgb, 255:red, 0; green, 0; blue, 0 }  ,draw opacity=1 ][fill={rgb, 255:red, 0; green, 0; blue, 0 }  ,fill opacity=1 ] (399.25,212.66) .. controls (399.25,213.97) and (400.31,215.04) .. (401.63,215.04) .. controls (402.94,215.04) and (404,213.97) .. (404,212.66) .. controls (404,211.35) and (402.94,210.29) .. (401.63,210.29) .. controls (400.31,210.29) and (399.25,211.35) .. (399.25,212.66) -- cycle ;

\end{tikzpicture}
&
\multicolumn{3}{c}{\raisebox{0.3cm}{14 further balanced problems; see Example \ref{ex:minProblsD3Delta0}}}
&
\\
\textbf{lable}
&
\underline{$d_3(8^2)$}
&
&
&
&
\\
\textbf{degree}
&
$25$
&
&
&
&
\\
\hline

\rule{0pt}{2.3cm}
\raisebox{0.75cm}{$\delta=1$, $d=0$}
&
\tikzset{every picture/.style={line width=0.75pt}} 

\begin{tikzpicture}[x=0.4pt,y=0.4pt,yscale=-1,xscale=1]

\draw   (306.25,87) -- (452,87) -- (452,232.75) -- (306.25,232.75) -- cycle ;
\draw [color={rgb, 255:red, 208; green, 2; blue, 27 }  ,draw opacity=1 ]   (438,98.83) -- (319.67,219.5) ;
\draw  [color={rgb, 255:red, 0; green, 0; blue, 0 }  ,draw opacity=1 ][fill={rgb, 255:red, 0; green, 0; blue, 0 }  ,fill opacity=1 ] (421.84,112.67) .. controls (421.84,113.99) and (422.9,115.05) .. (424.21,115.05) .. controls (425.52,115.05) and (426.59,113.99) .. (426.59,112.67) .. controls (426.59,111.36) and (425.52,110.3) .. (424.21,110.3) .. controls (422.9,110.3) and (421.84,111.36) .. (421.84,112.67) -- cycle ;
\draw  [color={rgb, 255:red, 0; green, 0; blue, 0 }  ,draw opacity=1 ][fill={rgb, 255:red, 0; green, 0; blue, 0 }  ,fill opacity=1 ] (401.94,133.94) .. controls (401.94,135.25) and (403,136.32) .. (404.31,136.32) .. controls (405.63,136.32) and (406.69,135.25) .. (406.69,133.94) .. controls (406.69,132.63) and (405.63,131.57) .. (404.31,131.57) .. controls (403,131.57) and (401.94,132.63) .. (401.94,133.94) -- cycle ;
\draw  [color={rgb, 255:red, 0; green, 0; blue, 0 }  ,draw opacity=1 ][fill={rgb, 255:red, 0; green, 0; blue, 0 }  ,fill opacity=1 ] (372.3,164.01) .. controls (372.3,165.32) and (373.36,166.38) .. (374.67,166.38) .. controls (375.98,166.38) and (377.05,165.32) .. (377.05,164.01) .. controls (377.05,162.7) and (375.98,161.63) .. (374.67,161.63) .. controls (373.36,161.63) and (372.3,162.7) .. (372.3,164.01) -- cycle ;
\draw  [color={rgb, 255:red, 0; green, 0; blue, 0 }  ,draw opacity=1 ][fill={rgb, 255:red, 0; green, 0; blue, 0 }  ,fill opacity=1 ] (342.2,194.15) .. controls (342.2,195.46) and (343.27,196.53) .. (344.58,196.53) .. controls (345.89,196.53) and (346.95,195.46) .. (346.95,194.15) .. controls (346.95,192.84) and (345.89,191.78) .. (344.58,191.78) .. controls (343.27,191.78) and (342.2,192.84) .. (342.2,194.15) -- cycle ;
\draw  [color={rgb, 255:red, 0; green, 0; blue, 0 }  ,draw opacity=1 ][fill={rgb, 255:red, 0; green, 0; blue, 0 }  ,fill opacity=1 ] (328.42,208.73) .. controls (328.42,210.05) and (329.48,211.11) .. (330.79,211.11) .. controls (332.1,211.11) and (333.17,210.05) .. (333.17,208.73) .. controls (333.17,207.42) and (332.1,206.36) .. (330.79,206.36) .. controls (329.48,206.36) and (328.42,207.42) .. (328.42,208.73) -- cycle ;

\end{tikzpicture}
&
\tikzset{every picture/.style={line width=0.75pt}} 

\begin{tikzpicture}[x=0.4pt,y=0.4pt,yscale=-1,xscale=1]

\draw   (308.25,77) -- (454,77) -- (454,222.75) -- (308.25,222.75) -- cycle ;
\draw [color={rgb, 255:red, 208; green, 2; blue, 27 }  ,draw opacity=1 ]   (397,91) -- (323.67,185.5) ;
\draw  [color={rgb, 255:red, 0; green, 0; blue, 0 }  ,draw opacity=1 ][fill={rgb, 255:red, 0; green, 0; blue, 0 }  ,fill opacity=1 ] (389.14,98.34) .. controls (389.14,99.65) and (390.2,100.72) .. (391.51,100.72) .. controls (392.83,100.72) and (393.89,99.65) .. (393.89,98.34) .. controls (393.89,97.03) and (392.83,95.97) .. (391.51,95.97) .. controls (390.2,95.97) and (389.14,97.03) .. (389.14,98.34) -- cycle ;
\draw  [color={rgb, 255:red, 0; green, 0; blue, 0 }  ,draw opacity=1 ][fill={rgb, 255:red, 0; green, 0; blue, 0 }  ,fill opacity=1 ] (374.3,117.21) .. controls (374.3,118.52) and (375.36,119.58) .. (376.67,119.58) .. controls (377.98,119.58) and (379.05,118.52) .. (379.05,117.21) .. controls (379.05,115.9) and (377.98,114.83) .. (376.67,114.83) .. controls (375.36,114.83) and (374.3,115.9) .. (374.3,117.21) -- cycle ;
\draw  [color={rgb, 255:red, 0; green, 0; blue, 0 }  ,draw opacity=1 ][fill={rgb, 255:red, 0; green, 0; blue, 0 }  ,fill opacity=1 ] (354.2,143.35) .. controls (354.2,144.66) and (355.27,145.73) .. (356.58,145.73) .. controls (357.89,145.73) and (358.95,144.66) .. (358.95,143.35) .. controls (358.95,142.04) and (357.89,140.98) .. (356.58,140.98) .. controls (355.27,140.98) and (354.2,142.04) .. (354.2,143.35) -- cycle ;
\draw  [color={rgb, 255:red, 0; green, 0; blue, 0 }  ,draw opacity=1 ][fill={rgb, 255:red, 0; green, 0; blue, 0 }  ,fill opacity=1 ] (335.17,167.63) .. controls (335.17,168.95) and (336.23,170.01) .. (337.54,170.01) .. controls (338.85,170.01) and (339.92,168.95) .. (339.92,167.63) .. controls (339.92,166.32) and (338.85,165.26) .. (337.54,165.26) .. controls (336.23,165.26) and (335.17,166.32) .. (335.17,167.63) -- cycle ;
\draw [color={rgb, 255:red, 208; green, 2; blue, 27 }  ,draw opacity=1 ]   (423.4,87.8) -- (399,209) ;
\draw  [color={rgb, 255:red, 0; green, 0; blue, 0 }  ,draw opacity=1 ][fill={rgb, 255:red, 0; green, 0; blue, 0 }  ,fill opacity=1 ] (417.42,106.13) .. controls (417.42,107.45) and (418.48,108.51) .. (419.79,108.51) .. controls (421.1,108.51) and (422.17,107.45) .. (422.17,106.13) .. controls (422.17,104.82) and (421.1,103.76) .. (419.79,103.76) .. controls (418.48,103.76) and (417.42,104.82) .. (417.42,106.13) -- cycle ;
\draw  [color={rgb, 255:red, 0; green, 0; blue, 0 }  ,draw opacity=1 ][fill={rgb, 255:red, 0; green, 0; blue, 0 }  ,fill opacity=1 ] (408.83,148.4) .. controls (408.83,149.71) and (409.89,150.77) .. (411.2,150.77) .. controls (412.51,150.77) and (413.58,149.71) .. (413.58,148.4) .. controls (413.58,147.09) and (412.51,146.02) .. (411.2,146.02) .. controls (409.89,146.02) and (408.83,147.09) .. (408.83,148.4) -- cycle ;
\draw  [color={rgb, 255:red, 0; green, 0; blue, 0 }  ,draw opacity=1 ][fill={rgb, 255:red, 0; green, 0; blue, 0 }  ,fill opacity=1 ] (401.67,185.38) .. controls (401.67,186.7) and (402.73,187.76) .. (404.04,187.76) .. controls (405.35,187.76) and (406.42,186.7) .. (406.42,185.38) .. controls (406.42,184.07) and (405.35,183.01) .. (404.04,183.01) .. controls (402.73,183.01) and (401.67,184.07) .. (401.67,185.38) -- cycle ;

\end{tikzpicture}
&
\tikzset{every picture/.style={line width=0.75pt}} 

\begin{tikzpicture}[x=0.4pt,y=0.4pt,yscale=-1,xscale=1]

\draw   (306.25,84) -- (452,84) -- (452,229.75) -- (306.25,229.75) -- cycle ;
\draw [color={rgb, 255:red, 208; green, 2; blue, 27 }  ,draw opacity=1 ]   (381,95.4) -- (321.67,192.5) ;
\draw  [color={rgb, 255:red, 0; green, 0; blue, 0 }  ,draw opacity=1 ][fill={rgb, 255:red, 0; green, 0; blue, 0 }  ,fill opacity=1 ] (373.14,104.54) .. controls (373.14,105.85) and (374.2,106.92) .. (375.51,106.92) .. controls (376.83,106.92) and (377.89,105.85) .. (377.89,104.54) .. controls (377.89,103.23) and (376.83,102.17) .. (375.51,102.17) .. controls (374.2,102.17) and (373.14,103.23) .. (373.14,104.54) -- cycle ;
\draw  [color={rgb, 255:red, 0; green, 0; blue, 0 }  ,draw opacity=1 ][fill={rgb, 255:red, 0; green, 0; blue, 0 }  ,fill opacity=1 ] (353.5,137.01) .. controls (353.5,138.32) and (354.56,139.38) .. (355.87,139.38) .. controls (357.18,139.38) and (358.25,138.32) .. (358.25,137.01) .. controls (358.25,135.7) and (357.18,134.63) .. (355.87,134.63) .. controls (354.56,134.63) and (353.5,135.7) .. (353.5,137.01) -- cycle ;
\draw  [color={rgb, 255:red, 0; green, 0; blue, 0 }  ,draw opacity=1 ][fill={rgb, 255:red, 0; green, 0; blue, 0 }  ,fill opacity=1 ] (330.88,174.35) .. controls (330.88,175.66) and (331.94,176.72) .. (333.26,176.72) .. controls (334.57,176.72) and (335.63,175.66) .. (335.63,174.35) .. controls (335.63,173.04) and (334.57,171.97) .. (333.26,171.97) .. controls (331.94,171.97) and (330.88,173.04) .. (330.88,174.35) -- cycle ;
\draw [color={rgb, 255:red, 208; green, 2; blue, 27 }  ,draw opacity=1 ]   (397.4,97.2) -- (373,218.4) ;
\draw  [color={rgb, 255:red, 0; green, 0; blue, 0 }  ,draw opacity=1 ][fill={rgb, 255:red, 0; green, 0; blue, 0 }  ,fill opacity=1 ] (392.56,109.99) .. controls (392.56,111.3) and (393.62,112.37) .. (394.93,112.37) .. controls (396.25,112.37) and (397.31,111.3) .. (397.31,109.99) .. controls (397.31,108.68) and (396.25,107.62) .. (394.93,107.62) .. controls (393.62,107.62) and (392.56,108.68) .. (392.56,109.99) -- cycle ;
\draw  [color={rgb, 255:red, 0; green, 0; blue, 0 }  ,draw opacity=1 ][fill={rgb, 255:red, 0; green, 0; blue, 0 }  ,fill opacity=1 ] (379.97,173.69) .. controls (379.97,175) and (381.03,176.06) .. (382.34,176.06) .. controls (383.65,176.06) and (384.72,175) .. (384.72,173.69) .. controls (384.72,172.37) and (383.65,171.31) .. (382.34,171.31) .. controls (381.03,171.31) and (379.97,172.37) .. (379.97,173.69) -- cycle ;
\draw  [color={rgb, 255:red, 0; green, 0; blue, 0 }  ,draw opacity=1 ][fill={rgb, 255:red, 0; green, 0; blue, 0 }  ,fill opacity=1 ] (372.81,206.96) .. controls (372.81,208.27) and (373.87,209.33) .. (375.18,209.33) .. controls (376.5,209.33) and (377.56,208.27) .. (377.56,206.96) .. controls (377.56,205.64) and (376.5,204.58) .. (375.18,204.58) .. controls (373.87,204.58) and (372.81,205.64) .. (372.81,206.96) -- cycle ;
\draw [color={rgb, 255:red, 208; green, 2; blue, 27 }  ,draw opacity=1 ]   (413,101.31) -- (431,212.31) ;
\draw  [color={rgb, 255:red, 0; green, 0; blue, 0 }  ,draw opacity=1 ][fill={rgb, 255:red, 0; green, 0; blue, 0 }  ,fill opacity=1 ] (423.1,177.77) .. controls (423.1,179.08) and (424.16,180.14) .. (425.47,180.14) .. controls (426.78,180.14) and (427.85,179.08) .. (427.85,177.77) .. controls (427.85,176.46) and (426.78,175.39) .. (425.47,175.39) .. controls (424.16,175.39) and (423.1,176.46) .. (423.1,177.77) -- cycle ;
\draw  [color={rgb, 255:red, 0; green, 0; blue, 0 }  ,draw opacity=1 ][fill={rgb, 255:red, 0; green, 0; blue, 0 }  ,fill opacity=1 ] (420.24,160.91) .. controls (420.24,162.22) and (421.3,163.29) .. (422.61,163.29) .. controls (423.92,163.29) and (424.99,162.22) .. (424.99,160.91) .. controls (424.99,159.6) and (423.92,158.54) .. (422.61,158.54) .. controls (421.3,158.54) and (420.24,159.6) .. (420.24,160.91) -- cycle ;
\draw  [color={rgb, 255:red, 0; green, 0; blue, 0 }  ,draw opacity=1 ][fill={rgb, 255:red, 0; green, 0; blue, 0 }  ,fill opacity=1 ] (414.81,127.2) .. controls (414.81,128.51) and (415.87,129.57) .. (417.18,129.57) .. controls (418.5,129.57) and (419.56,128.51) .. (419.56,127.2) .. controls (419.56,125.88) and (418.5,124.82) .. (417.18,124.82) .. controls (415.87,124.82) and (414.81,125.88) .. (414.81,127.2) -- cycle ;

\end{tikzpicture}
&
&
\\
\textbf{lable}
&
\underline{$\delta_1(5)$}
&
\underline{$\delta_1(4,3)$}
&
\underline{$\delta_1(3^3)$}
&
&
\\
\textbf{degree}
&
$10$
&
$30$
&
$54$
&
&
\\
\hline

\rule{0pt}{2.3cm}
\raisebox{0.75cm}{$\delta=1$, $0<d\le 5$}
&
\tikzset{every picture/.style={line width=0.75pt}} 

\begin{tikzpicture}[x=0.4pt,y=0.4pt,yscale=-1,xscale=1]

\draw   (310.82,74.14) -- (456.57,74.14) -- (456.57,219.89) -- (310.82,219.89) -- cycle ;
\draw [color={rgb, 255:red, 208; green, 2; blue, 27 }  ,draw opacity=1 ]   (440.43,90.43) -- (326.2,204.94) ;
\draw  [color={rgb, 255:red, 0; green, 0; blue, 0 }  ,draw opacity=1 ][fill={rgb, 255:red, 0; green, 0; blue, 0 }  ,fill opacity=1 ] (411.1,117.91) .. controls (411.1,119.22) and (412.16,120.29) .. (413.47,120.29) .. controls (414.78,120.29) and (415.85,119.22) .. (415.85,117.91) .. controls (415.85,116.6) and (414.78,115.54) .. (413.47,115.54) .. controls (412.16,115.54) and (411.1,116.6) .. (411.1,117.91) -- cycle ;
\draw  [color={rgb, 255:red, 0; green, 0; blue, 0 }  ,draw opacity=1 ][fill={rgb, 255:red, 0; green, 0; blue, 0 }  ,fill opacity=1 ] (421.72,107) .. controls (421.72,108.31) and (422.79,109.37) .. (424.1,109.37) .. controls (425.41,109.37) and (426.47,108.31) .. (426.47,107) .. controls (426.47,105.68) and (425.41,104.62) .. (424.1,104.62) .. controls (422.79,104.62) and (421.72,105.68) .. (421.72,107) -- cycle ;
\draw  [color={rgb, 255:red, 0; green, 0; blue, 0 }  ,draw opacity=1 ][fill={rgb, 255:red, 0; green, 0; blue, 0 }  ,fill opacity=1 ] (428.58,99.8) .. controls (428.58,101.11) and (429.64,102.17) .. (430.96,102.17) .. controls (432.27,102.17) and (433.33,101.11) .. (433.33,99.8) .. controls (433.33,98.48) and (432.27,97.42) .. (430.96,97.42) .. controls (429.64,97.42) and (428.58,98.48) .. (428.58,99.8) -- cycle ;
\draw  [color={rgb, 255:red, 0; green, 0; blue, 0 }  ,draw opacity=1 ][fill={rgb, 255:red, 0; green, 0; blue, 0 }  ,fill opacity=1 ] (401.67,127.62) .. controls (401.67,128.94) and (402.73,130) .. (404.04,130) .. controls (405.35,130) and (406.42,128.94) .. (406.42,127.62) .. controls (406.42,126.31) and (405.35,125.25) .. (404.04,125.25) .. controls (402.73,125.25) and (401.67,126.31) .. (401.67,127.62) -- cycle ;
\draw  [color={rgb, 255:red, 0; green, 0; blue, 0 }  ,draw opacity=1 ][fill={rgb, 255:red, 0; green, 0; blue, 0 }  ,fill opacity=1 ] (386.52,142.48) .. controls (386.52,143.79) and (387.59,144.86) .. (388.9,144.86) .. controls (390.21,144.86) and (391.27,143.79) .. (391.27,142.48) .. controls (391.27,141.17) and (390.21,140.11) .. (388.9,140.11) .. controls (387.59,140.11) and (386.52,141.17) .. (386.52,142.48) -- cycle ;
\draw  [color={rgb, 255:red, 0; green, 0; blue, 0 }  ,draw opacity=1 ][fill={rgb, 255:red, 0; green, 0; blue, 0 }  ,fill opacity=1 ] (331.95,197.05) .. controls (331.95,198.37) and (333.02,199.43) .. (334.33,199.43) .. controls (335.64,199.43) and (336.7,198.37) .. (336.7,197.05) .. controls (336.7,195.74) and (335.64,194.68) .. (334.33,194.68) .. controls (333.02,194.68) and (331.95,195.74) .. (331.95,197.05) -- cycle ;
\draw  [color={rgb, 255:red, 0; green, 0; blue, 0 }  ,draw opacity=1 ][fill={rgb, 255:red, 0; green, 0; blue, 0 }  ,fill opacity=1 ] (348.24,181.05) .. controls (348.24,182.37) and (349.3,183.43) .. (350.61,183.43) .. controls (351.92,183.43) and (352.99,182.37) .. (352.99,181.05) .. controls (352.99,179.74) and (351.92,178.68) .. (350.61,178.68) .. controls (349.3,178.68) and (348.24,179.74) .. (348.24,181.05) -- cycle ;
\draw  [color={rgb, 255:red, 0; green, 0; blue, 0 }  ,draw opacity=1 ][fill={rgb, 255:red, 0; green, 0; blue, 0 }  ,fill opacity=1 ] (367.67,161.34) .. controls (367.67,162.65) and (368.73,163.71) .. (370.04,163.71) .. controls (371.35,163.71) and (372.42,162.65) .. (372.42,161.34) .. controls (372.42,160.03) and (371.35,158.96) .. (370.04,158.96) .. controls (368.73,158.96) and (367.67,160.03) .. (367.67,161.34) -- cycle ;

\end{tikzpicture}
&
\tikzset{every picture/.style={line width=0.75pt}} 

\begin{tikzpicture}[x=0.4pt,y=0.4pt,yscale=-1,xscale=1]

\draw   (311.82,80.14) -- (457.57,80.14) -- (457.57,225.89) -- (311.82,225.89) -- cycle ;
\draw [color={rgb, 255:red, 208; green, 2; blue, 27 }  ,draw opacity=1 ]   (441.43,96.43) -- (327.2,210.94) ;
\draw  [color={rgb, 255:red, 0; green, 0; blue, 0 }  ,draw opacity=1 ][fill={rgb, 255:red, 0; green, 0; blue, 0 }  ,fill opacity=1 ] (412.1,123.91) .. controls (412.1,125.22) and (413.16,126.29) .. (414.47,126.29) .. controls (415.78,126.29) and (416.85,125.22) .. (416.85,123.91) .. controls (416.85,122.6) and (415.78,121.54) .. (414.47,121.54) .. controls (413.16,121.54) and (412.1,122.6) .. (412.1,123.91) -- cycle ;
\draw  [color={rgb, 255:red, 0; green, 0; blue, 0 }  ,draw opacity=1 ][fill={rgb, 255:red, 0; green, 0; blue, 0 }  ,fill opacity=1 ] (422.72,113) .. controls (422.72,114.31) and (423.79,115.37) .. (425.1,115.37) .. controls (426.41,115.37) and (427.47,114.31) .. (427.47,113) .. controls (427.47,111.68) and (426.41,110.62) .. (425.1,110.62) .. controls (423.79,110.62) and (422.72,111.68) .. (422.72,113) -- cycle ;
\draw  [color={rgb, 255:red, 0; green, 0; blue, 0 }  ,draw opacity=1 ][fill={rgb, 255:red, 0; green, 0; blue, 0 }  ,fill opacity=1 ] (429.58,105.8) .. controls (429.58,107.11) and (430.64,108.17) .. (431.96,108.17) .. controls (433.27,108.17) and (434.33,107.11) .. (434.33,105.8) .. controls (434.33,104.48) and (433.27,103.42) .. (431.96,103.42) .. controls (430.64,103.42) and (429.58,104.48) .. (429.58,105.8) -- cycle ;
\draw  [color={rgb, 255:red, 0; green, 0; blue, 0 }  ,draw opacity=1 ][fill={rgb, 255:red, 0; green, 0; blue, 0 }  ,fill opacity=1 ] (402.67,133.63) .. controls (402.67,134.94) and (403.73,136) .. (405.04,136) .. controls (406.35,136) and (407.42,134.94) .. (407.42,133.63) .. controls (407.42,132.31) and (406.35,131.25) .. (405.04,131.25) .. controls (403.73,131.25) and (402.67,132.31) .. (402.67,133.63) -- cycle ;
\draw  [color={rgb, 255:red, 0; green, 0; blue, 0 }  ,draw opacity=1 ][fill={rgb, 255:red, 0; green, 0; blue, 0 }  ,fill opacity=1 ] (387.52,148.48) .. controls (387.52,149.79) and (388.59,150.86) .. (389.9,150.86) .. controls (391.21,150.86) and (392.27,149.79) .. (392.27,148.48) .. controls (392.27,147.17) and (391.21,146.11) .. (389.9,146.11) .. controls (388.59,146.11) and (387.52,147.17) .. (387.52,148.48) -- cycle ;
\draw  [color={rgb, 255:red, 0; green, 0; blue, 0 }  ,draw opacity=1 ][fill={rgb, 255:red, 0; green, 0; blue, 0 }  ,fill opacity=1 ] (332.95,203.05) .. controls (332.95,204.37) and (334.02,205.43) .. (335.33,205.43) .. controls (336.64,205.43) and (337.7,204.37) .. (337.7,203.05) .. controls (337.7,201.74) and (336.64,200.68) .. (335.33,200.68) .. controls (334.02,200.68) and (332.95,201.74) .. (332.95,203.05) -- cycle ;
\draw  [color={rgb, 255:red, 0; green, 0; blue, 0 }  ,draw opacity=1 ][fill={rgb, 255:red, 0; green, 0; blue, 0 }  ,fill opacity=1 ] (340.84,195.05) .. controls (340.84,196.37) and (341.9,197.43) .. (343.21,197.43) .. controls (344.52,197.43) and (345.59,196.37) .. (345.59,195.05) .. controls (345.59,193.74) and (344.52,192.68) .. (343.21,192.68) .. controls (341.9,192.68) and (340.84,193.74) .. (340.84,195.05) -- cycle ;
\draw  [color={rgb, 255:red, 0; green, 0; blue, 0 }  ,draw opacity=1 ][fill={rgb, 255:red, 0; green, 0; blue, 0 }  ,fill opacity=1 ] (375.87,160.14) .. controls (375.87,161.45) and (376.93,162.51) .. (378.24,162.51) .. controls (379.55,162.51) and (380.62,161.45) .. (380.62,160.14) .. controls (380.62,158.83) and (379.55,157.76) .. (378.24,157.76) .. controls (376.93,157.76) and (375.87,158.83) .. (375.87,160.14) -- cycle ;
\draw  [color={rgb, 255:red, 0; green, 0; blue, 0 }  ,draw opacity=1 ][fill={rgb, 255:red, 0; green, 0; blue, 0 }  ,fill opacity=1 ] (354.81,181.45) .. controls (354.81,182.77) and (355.88,183.83) .. (357.19,183.83) .. controls (358.5,183.83) and (359.56,182.77) .. (359.56,181.45) .. controls (359.56,180.14) and (358.5,179.08) .. (357.19,179.08) .. controls (355.88,179.08) and (354.81,180.14) .. (354.81,181.45) -- cycle ;
\draw  [color={rgb, 255:red, 0; green, 0; blue, 0 }  ,draw opacity=1 ][fill={rgb, 255:red, 0; green, 0; blue, 0 }  ,fill opacity=1 ] (367.92,168.6) .. controls (367.92,169.91) and (368.99,170.97) .. (370.3,170.97) .. controls (371.61,170.97) and (372.67,169.91) .. (372.67,168.6) .. controls (372.67,167.28) and (371.61,166.22) .. (370.3,166.22) .. controls (368.99,166.22) and (367.92,167.28) .. (367.92,168.6) -- cycle ;

\end{tikzpicture}
&
\tikzset{every picture/.style={line width=0.75pt}} 

\begin{tikzpicture}[x=0.4pt,y=0.4pt,yscale=-1,xscale=1]

\draw   (311.82,80.14) -- (457.57,80.14) -- (457.57,225.89) -- (311.82,225.89) -- cycle ;
\draw [color={rgb, 255:red, 208; green, 2; blue, 27 }  ,draw opacity=1 ]   (441.43,96.43) -- (327.2,210.94) ;
\draw  [color={rgb, 255:red, 0; green, 0; blue, 0 }  ,draw opacity=1 ][fill={rgb, 255:red, 0; green, 0; blue, 0 }  ,fill opacity=1 ] (412.1,123.91) .. controls (412.1,125.22) and (413.16,126.29) .. (414.47,126.29) .. controls (415.78,126.29) and (416.85,125.22) .. (416.85,123.91) .. controls (416.85,122.6) and (415.78,121.54) .. (414.47,121.54) .. controls (413.16,121.54) and (412.1,122.6) .. (412.1,123.91) -- cycle ;
\draw  [color={rgb, 255:red, 0; green, 0; blue, 0 }  ,draw opacity=1 ][fill={rgb, 255:red, 0; green, 0; blue, 0 }  ,fill opacity=1 ] (422.72,113) .. controls (422.72,114.31) and (423.79,115.37) .. (425.1,115.37) .. controls (426.41,115.37) and (427.47,114.31) .. (427.47,113) .. controls (427.47,111.68) and (426.41,110.62) .. (425.1,110.62) .. controls (423.79,110.62) and (422.72,111.68) .. (422.72,113) -- cycle ;
\draw  [color={rgb, 255:red, 0; green, 0; blue, 0 }  ,draw opacity=1 ][fill={rgb, 255:red, 0; green, 0; blue, 0 }  ,fill opacity=1 ] (429.58,105.8) .. controls (429.58,107.11) and (430.64,108.17) .. (431.96,108.17) .. controls (433.27,108.17) and (434.33,107.11) .. (434.33,105.8) .. controls (434.33,104.48) and (433.27,103.42) .. (431.96,103.42) .. controls (430.64,103.42) and (429.58,104.48) .. (429.58,105.8) -- cycle ;
\draw  [color={rgb, 255:red, 0; green, 0; blue, 0 }  ,draw opacity=1 ][fill={rgb, 255:red, 0; green, 0; blue, 0 }  ,fill opacity=1 ] (402.67,133.63) .. controls (402.67,134.94) and (403.73,136) .. (405.04,136) .. controls (406.35,136) and (407.42,134.94) .. (407.42,133.63) .. controls (407.42,132.31) and (406.35,131.25) .. (405.04,131.25) .. controls (403.73,131.25) and (402.67,132.31) .. (402.67,133.63) -- cycle ;
\draw  [color={rgb, 255:red, 0; green, 0; blue, 0 }  ,draw opacity=1 ][fill={rgb, 255:red, 0; green, 0; blue, 0 }  ,fill opacity=1 ] (387.52,148.48) .. controls (387.52,149.79) and (388.59,150.86) .. (389.9,150.86) .. controls (391.21,150.86) and (392.27,149.79) .. (392.27,148.48) .. controls (392.27,147.17) and (391.21,146.11) .. (389.9,146.11) .. controls (388.59,146.11) and (387.52,147.17) .. (387.52,148.48) -- cycle ;
\draw  [color={rgb, 255:red, 0; green, 0; blue, 0 }  ,draw opacity=1 ][fill={rgb, 255:red, 0; green, 0; blue, 0 }  ,fill opacity=1 ] (332.95,203.05) .. controls (332.95,204.37) and (334.02,205.43) .. (335.33,205.43) .. controls (336.64,205.43) and (337.7,204.37) .. (337.7,203.05) .. controls (337.7,201.74) and (336.64,200.68) .. (335.33,200.68) .. controls (334.02,200.68) and (332.95,201.74) .. (332.95,203.05) -- cycle ;
\draw  [color={rgb, 255:red, 0; green, 0; blue, 0 }  ,draw opacity=1 ][fill={rgb, 255:red, 0; green, 0; blue, 0 }  ,fill opacity=1 ] (340.84,195.05) .. controls (340.84,196.37) and (341.9,197.43) .. (343.21,197.43) .. controls (344.52,197.43) and (345.59,196.37) .. (345.59,195.05) .. controls (345.59,193.74) and (344.52,192.68) .. (343.21,192.68) .. controls (341.9,192.68) and (340.84,193.74) .. (340.84,195.05) -- cycle ;
\draw  [color={rgb, 255:red, 0; green, 0; blue, 0 }  ,draw opacity=1 ][fill={rgb, 255:red, 0; green, 0; blue, 0 }  ,fill opacity=1 ] (375.87,160.14) .. controls (375.87,161.45) and (376.93,162.51) .. (378.24,162.51) .. controls (379.55,162.51) and (380.62,161.45) .. (380.62,160.14) .. controls (380.62,158.83) and (379.55,157.76) .. (378.24,157.76) .. controls (376.93,157.76) and (375.87,158.83) .. (375.87,160.14) -- cycle ;
\draw  [color={rgb, 255:red, 0; green, 0; blue, 0 }  ,draw opacity=1 ][fill={rgb, 255:red, 0; green, 0; blue, 0 }  ,fill opacity=1 ] (349.67,186.31) .. controls (349.67,187.62) and (350.73,188.69) .. (352.05,188.69) .. controls (353.36,188.69) and (354.42,187.62) .. (354.42,186.31) .. controls (354.42,185) and (353.36,183.94) .. (352.05,183.94) .. controls (350.73,183.94) and (349.67,185) .. (349.67,186.31) -- cycle ;
\draw  [color={rgb, 255:red, 0; green, 0; blue, 0 }  ,draw opacity=1 ][fill={rgb, 255:red, 0; green, 0; blue, 0 }  ,fill opacity=1 ] (367.92,168.6) .. controls (367.92,169.91) and (368.99,170.97) .. (370.3,170.97) .. controls (371.61,170.97) and (372.67,169.91) .. (372.67,168.6) .. controls (372.67,167.28) and (371.61,166.22) .. (370.3,166.22) .. controls (368.99,166.22) and (367.92,167.28) .. (367.92,168.6) -- cycle ;
\draw  [color={rgb, 255:red, 0; green, 0; blue, 0 }  ,draw opacity=1 ][fill={rgb, 255:red, 0; green, 0; blue, 0 }  ,fill opacity=1 ] (395.1,140.88) .. controls (395.1,142.19) and (396.16,143.26) .. (397.47,143.26) .. controls (398.79,143.26) and (399.85,142.19) .. (399.85,140.88) .. controls (399.85,139.57) and (398.79,138.51) .. (397.47,138.51) .. controls (396.16,138.51) and (395.1,139.57) .. (395.1,140.88) -- cycle ;
\draw  [color={rgb, 255:red, 0; green, 0; blue, 0 }  ,draw opacity=1 ][fill={rgb, 255:red, 0; green, 0; blue, 0 }  ,fill opacity=1 ] (358.53,177.74) .. controls (358.53,179.05) and (359.59,180.11) .. (360.9,180.11) .. controls (362.21,180.11) and (363.28,179.05) .. (363.28,177.74) .. controls (363.28,176.43) and (362.21,175.36) .. (360.9,175.36) .. controls (359.59,175.36) and (358.53,176.43) .. (358.53,177.74) -- cycle ;

\end{tikzpicture}
&
\tikzset{every picture/.style={line width=0.75pt}} 

\begin{tikzpicture}[x=0.4pt,y=0.4pt,yscale=-1,xscale=1]

\draw   (311.82,80.14) -- (457.57,80.14) -- (457.57,225.89) -- (311.82,225.89) -- cycle ;
\draw [color={rgb, 255:red, 208; green, 2; blue, 27 }  ,draw opacity=1 ]   (441.43,96.43) -- (327.2,210.94) ;
\draw  [color={rgb, 255:red, 0; green, 0; blue, 0 }  ,draw opacity=1 ][fill={rgb, 255:red, 0; green, 0; blue, 0 }  ,fill opacity=1 ] (408.95,127.05) .. controls (408.95,128.37) and (410.02,129.43) .. (411.33,129.43) .. controls (412.64,129.43) and (413.7,128.37) .. (413.7,127.05) .. controls (413.7,125.74) and (412.64,124.68) .. (411.33,124.68) .. controls (410.02,124.68) and (408.95,125.74) .. (408.95,127.05) -- cycle ;
\draw  [color={rgb, 255:red, 0; green, 0; blue, 0 }  ,draw opacity=1 ][fill={rgb, 255:red, 0; green, 0; blue, 0 }  ,fill opacity=1 ] (422.72,113) .. controls (422.72,114.31) and (423.79,115.37) .. (425.1,115.37) .. controls (426.41,115.37) and (427.47,114.31) .. (427.47,113) .. controls (427.47,111.68) and (426.41,110.62) .. (425.1,110.62) .. controls (423.79,110.62) and (422.72,111.68) .. (422.72,113) -- cycle ;
\draw  [color={rgb, 255:red, 0; green, 0; blue, 0 }  ,draw opacity=1 ][fill={rgb, 255:red, 0; green, 0; blue, 0 }  ,fill opacity=1 ] (429.58,105.8) .. controls (429.58,107.11) and (430.64,108.17) .. (431.96,108.17) .. controls (433.27,108.17) and (434.33,107.11) .. (434.33,105.8) .. controls (434.33,104.48) and (433.27,103.42) .. (431.96,103.42) .. controls (430.64,103.42) and (429.58,104.48) .. (429.58,105.8) -- cycle ;
\draw  [color={rgb, 255:red, 0; green, 0; blue, 0 }  ,draw opacity=1 ][fill={rgb, 255:red, 0; green, 0; blue, 0 }  ,fill opacity=1 ] (402.67,133.63) .. controls (402.67,134.94) and (403.73,136) .. (405.04,136) .. controls (406.35,136) and (407.42,134.94) .. (407.42,133.63) .. controls (407.42,132.31) and (406.35,131.25) .. (405.04,131.25) .. controls (403.73,131.25) and (402.67,132.31) .. (402.67,133.63) -- cycle ;
\draw  [color={rgb, 255:red, 0; green, 0; blue, 0 }  ,draw opacity=1 ][fill={rgb, 255:red, 0; green, 0; blue, 0 }  ,fill opacity=1 ] (381.94,153.69) .. controls (381.94,155) and (383,156.06) .. (384.31,156.06) .. controls (385.63,156.06) and (386.69,155) .. (386.69,153.69) .. controls (386.69,152.37) and (385.63,151.31) .. (384.31,151.31) .. controls (383,151.31) and (381.94,152.37) .. (381.94,153.69) -- cycle ;
\draw  [color={rgb, 255:red, 0; green, 0; blue, 0 }  ,draw opacity=1 ][fill={rgb, 255:red, 0; green, 0; blue, 0 }  ,fill opacity=1 ] (332.95,203.05) .. controls (332.95,204.37) and (334.02,205.43) .. (335.33,205.43) .. controls (336.64,205.43) and (337.7,204.37) .. (337.7,203.05) .. controls (337.7,201.74) and (336.64,200.68) .. (335.33,200.68) .. controls (334.02,200.68) and (332.95,201.74) .. (332.95,203.05) -- cycle ;
\draw  [color={rgb, 255:red, 0; green, 0; blue, 0 }  ,draw opacity=1 ][fill={rgb, 255:red, 0; green, 0; blue, 0 }  ,fill opacity=1 ] (340.84,195.05) .. controls (340.84,196.37) and (341.9,197.43) .. (343.21,197.43) .. controls (344.52,197.43) and (345.59,196.37) .. (345.59,195.05) .. controls (345.59,193.74) and (344.52,192.68) .. (343.21,192.68) .. controls (341.9,192.68) and (340.84,193.74) .. (340.84,195.05) -- cycle ;
\draw  [color={rgb, 255:red, 0; green, 0; blue, 0 }  ,draw opacity=1 ][fill={rgb, 255:red, 0; green, 0; blue, 0 }  ,fill opacity=1 ] (375.87,160.14) .. controls (375.87,161.45) and (376.93,162.51) .. (378.24,162.51) .. controls (379.55,162.51) and (380.62,161.45) .. (380.62,160.14) .. controls (380.62,158.83) and (379.55,157.76) .. (378.24,157.76) .. controls (376.93,157.76) and (375.87,158.83) .. (375.87,160.14) -- cycle ;
\draw  [color={rgb, 255:red, 0; green, 0; blue, 0 }  ,draw opacity=1 ][fill={rgb, 255:red, 0; green, 0; blue, 0 }  ,fill opacity=1 ] (349.67,186.31) .. controls (349.67,187.62) and (350.73,188.69) .. (352.05,188.69) .. controls (353.36,188.69) and (354.42,187.62) .. (354.42,186.31) .. controls (354.42,185) and (353.36,183.94) .. (352.05,183.94) .. controls (350.73,183.94) and (349.67,185) .. (349.67,186.31) -- cycle ;
\draw  [color={rgb, 255:red, 0; green, 0; blue, 0 }  ,draw opacity=1 ][fill={rgb, 255:red, 0; green, 0; blue, 0 }  ,fill opacity=1 ] (367.92,168.6) .. controls (367.92,169.91) and (368.99,170.97) .. (370.3,170.97) .. controls (371.61,170.97) and (372.67,169.91) .. (372.67,168.6) .. controls (372.67,167.28) and (371.61,166.22) .. (370.3,166.22) .. controls (368.99,166.22) and (367.92,167.28) .. (367.92,168.6) -- cycle ;
\draw  [color={rgb, 255:red, 0; green, 0; blue, 0 }  ,draw opacity=1 ][fill={rgb, 255:red, 0; green, 0; blue, 0 }  ,fill opacity=1 ] (395.1,140.88) .. controls (395.1,142.19) and (396.16,143.26) .. (397.47,143.26) .. controls (398.79,143.26) and (399.85,142.19) .. (399.85,140.88) .. controls (399.85,139.57) and (398.79,138.51) .. (397.47,138.51) .. controls (396.16,138.51) and (395.1,139.57) .. (395.1,140.88) -- cycle ;
\draw  [color={rgb, 255:red, 0; green, 0; blue, 0 }  ,draw opacity=1 ][fill={rgb, 255:red, 0; green, 0; blue, 0 }  ,fill opacity=1 ] (358.53,177.74) .. controls (358.53,179.05) and (359.59,180.11) .. (360.9,180.11) .. controls (362.21,180.11) and (363.28,179.05) .. (363.28,177.74) .. controls (363.28,176.43) and (362.21,175.36) .. (360.9,175.36) .. controls (359.59,175.36) and (358.53,176.43) .. (358.53,177.74) -- cycle ;
\draw  [color={rgb, 255:red, 0; green, 0; blue, 0 }  ,draw opacity=1 ][fill={rgb, 255:red, 0; green, 0; blue, 0 }  ,fill opacity=1 ] (415.52,120.2) .. controls (415.52,121.51) and (416.59,122.57) .. (417.9,122.57) .. controls (419.21,122.57) and (420.27,121.51) .. (420.27,120.2) .. controls (420.27,118.88) and (419.21,117.82) .. (417.9,117.82) .. controls (416.59,117.82) and (415.52,118.88) .. (415.52,120.2) -- cycle ;
\draw  [color={rgb, 255:red, 0; green, 0; blue, 0 }  ,draw opacity=1 ][fill={rgb, 255:red, 0; green, 0; blue, 0 }  ,fill opacity=1 ] (388.1,147.91) .. controls (388.1,149.22) and (389.16,150.29) .. (390.47,150.29) .. controls (391.78,150.29) and (392.85,149.22) .. (392.85,147.91) .. controls (392.85,146.6) and (391.78,145.54) .. (390.47,145.54) .. controls (389.16,145.54) and (388.1,146.6) .. (388.1,147.91) -- cycle ;

\end{tikzpicture}
&
\tikzset{every picture/.style={line width=0.75pt}} 

\begin{tikzpicture}[x=0.4pt,y=0.4pt,yscale=-1,xscale=1]

\draw   (311.82,80.14) -- (457.57,80.14) -- (457.57,225.89) -- (311.82,225.89) -- cycle ;
\draw [color={rgb, 255:red, 208; green, 2; blue, 27 }  ,draw opacity=1 ]   (441.43,96.43) -- (327.2,210.94) ;
\draw  [color={rgb, 255:red, 0; green, 0; blue, 0 }  ,draw opacity=1 ][fill={rgb, 255:red, 0; green, 0; blue, 0 }  ,fill opacity=1 ] (408.95,127.05) .. controls (408.95,128.37) and (410.02,129.43) .. (411.33,129.43) .. controls (412.64,129.43) and (413.7,128.37) .. (413.7,127.05) .. controls (413.7,125.74) and (412.64,124.68) .. (411.33,124.68) .. controls (410.02,124.68) and (408.95,125.74) .. (408.95,127.05) -- cycle ;
\draw  [color={rgb, 255:red, 0; green, 0; blue, 0 }  ,draw opacity=1 ][fill={rgb, 255:red, 0; green, 0; blue, 0 }  ,fill opacity=1 ] (421.87,113.57) .. controls (421.87,114.88) and (422.93,115.94) .. (424.24,115.94) .. controls (425.55,115.94) and (426.62,114.88) .. (426.62,113.57) .. controls (426.62,112.26) and (425.55,111.19) .. (424.24,111.19) .. controls (422.93,111.19) and (421.87,112.26) .. (421.87,113.57) -- cycle ;
\draw  [color={rgb, 255:red, 0; green, 0; blue, 0 }  ,draw opacity=1 ][fill={rgb, 255:red, 0; green, 0; blue, 0 }  ,fill opacity=1 ] (427.58,108.37) .. controls (427.58,109.68) and (428.64,110.74) .. (429.96,110.74) .. controls (431.27,110.74) and (432.33,109.68) .. (432.33,108.37) .. controls (432.33,107.06) and (431.27,105.99) .. (429.96,105.99) .. controls (428.64,105.99) and (427.58,107.06) .. (427.58,108.37) -- cycle ;
\draw  [color={rgb, 255:red, 0; green, 0; blue, 0 }  ,draw opacity=1 ][fill={rgb, 255:red, 0; green, 0; blue, 0 }  ,fill opacity=1 ] (402.67,133.63) .. controls (402.67,134.94) and (403.73,136) .. (405.04,136) .. controls (406.35,136) and (407.42,134.94) .. (407.42,133.63) .. controls (407.42,132.31) and (406.35,131.25) .. (405.04,131.25) .. controls (403.73,131.25) and (402.67,132.31) .. (402.67,133.63) -- cycle ;
\draw  [color={rgb, 255:red, 0; green, 0; blue, 0 }  ,draw opacity=1 ][fill={rgb, 255:red, 0; green, 0; blue, 0 }  ,fill opacity=1 ] (381.94,153.69) .. controls (381.94,155) and (383,156.06) .. (384.31,156.06) .. controls (385.63,156.06) and (386.69,155) .. (386.69,153.69) .. controls (386.69,152.37) and (385.63,151.31) .. (384.31,151.31) .. controls (383,151.31) and (381.94,152.37) .. (381.94,153.69) -- cycle ;
\draw  [color={rgb, 255:red, 0; green, 0; blue, 0 }  ,draw opacity=1 ][fill={rgb, 255:red, 0; green, 0; blue, 0 }  ,fill opacity=1 ] (332.95,203.05) .. controls (332.95,204.37) and (334.02,205.43) .. (335.33,205.43) .. controls (336.64,205.43) and (337.7,204.37) .. (337.7,203.05) .. controls (337.7,201.74) and (336.64,200.68) .. (335.33,200.68) .. controls (334.02,200.68) and (332.95,201.74) .. (332.95,203.05) -- cycle ;
\draw  [color={rgb, 255:red, 0; green, 0; blue, 0 }  ,draw opacity=1 ][fill={rgb, 255:red, 0; green, 0; blue, 0 }  ,fill opacity=1 ] (340.84,195.05) .. controls (340.84,196.37) and (341.9,197.43) .. (343.21,197.43) .. controls (344.52,197.43) and (345.59,196.37) .. (345.59,195.05) .. controls (345.59,193.74) and (344.52,192.68) .. (343.21,192.68) .. controls (341.9,192.68) and (340.84,193.74) .. (340.84,195.05) -- cycle ;
\draw  [color={rgb, 255:red, 0; green, 0; blue, 0 }  ,draw opacity=1 ][fill={rgb, 255:red, 0; green, 0; blue, 0 }  ,fill opacity=1 ] (375.87,160.14) .. controls (375.87,161.45) and (376.93,162.51) .. (378.24,162.51) .. controls (379.55,162.51) and (380.62,161.45) .. (380.62,160.14) .. controls (380.62,158.83) and (379.55,157.76) .. (378.24,157.76) .. controls (376.93,157.76) and (375.87,158.83) .. (375.87,160.14) -- cycle ;
\draw  [color={rgb, 255:red, 0; green, 0; blue, 0 }  ,draw opacity=1 ][fill={rgb, 255:red, 0; green, 0; blue, 0 }  ,fill opacity=1 ] (349.67,186.31) .. controls (349.67,187.62) and (350.73,188.69) .. (352.05,188.69) .. controls (353.36,188.69) and (354.42,187.62) .. (354.42,186.31) .. controls (354.42,185) and (353.36,183.94) .. (352.05,183.94) .. controls (350.73,183.94) and (349.67,185) .. (349.67,186.31) -- cycle ;
\draw  [color={rgb, 255:red, 0; green, 0; blue, 0 }  ,draw opacity=1 ][fill={rgb, 255:red, 0; green, 0; blue, 0 }  ,fill opacity=1 ] (367.92,168.6) .. controls (367.92,169.91) and (368.99,170.97) .. (370.3,170.97) .. controls (371.61,170.97) and (372.67,169.91) .. (372.67,168.6) .. controls (372.67,167.28) and (371.61,166.22) .. (370.3,166.22) .. controls (368.99,166.22) and (367.92,167.28) .. (367.92,168.6) -- cycle ;
\draw  [color={rgb, 255:red, 0; green, 0; blue, 0 }  ,draw opacity=1 ][fill={rgb, 255:red, 0; green, 0; blue, 0 }  ,fill opacity=1 ] (395.1,140.88) .. controls (395.1,142.19) and (396.16,143.26) .. (397.47,143.26) .. controls (398.79,143.26) and (399.85,142.19) .. (399.85,140.88) .. controls (399.85,139.57) and (398.79,138.51) .. (397.47,138.51) .. controls (396.16,138.51) and (395.1,139.57) .. (395.1,140.88) -- cycle ;
\draw  [color={rgb, 255:red, 0; green, 0; blue, 0 }  ,draw opacity=1 ][fill={rgb, 255:red, 0; green, 0; blue, 0 }  ,fill opacity=1 ] (355.67,179.74) .. controls (355.67,181.05) and (356.73,182.11) .. (358.05,182.11) .. controls (359.36,182.11) and (360.42,181.05) .. (360.42,179.74) .. controls (360.42,178.43) and (359.36,177.36) .. (358.05,177.36) .. controls (356.73,177.36) and (355.67,178.43) .. (355.67,179.74) -- cycle ;
\draw  [color={rgb, 255:red, 0; green, 0; blue, 0 }  ,draw opacity=1 ][fill={rgb, 255:red, 0; green, 0; blue, 0 }  ,fill opacity=1 ] (415.52,120.2) .. controls (415.52,121.51) and (416.59,122.57) .. (417.9,122.57) .. controls (419.21,122.57) and (420.27,121.51) .. (420.27,120.2) .. controls (420.27,118.88) and (419.21,117.82) .. (417.9,117.82) .. controls (416.59,117.82) and (415.52,118.88) .. (415.52,120.2) -- cycle ;
\draw  [color={rgb, 255:red, 0; green, 0; blue, 0 }  ,draw opacity=1 ][fill={rgb, 255:red, 0; green, 0; blue, 0 }  ,fill opacity=1 ] (388.1,147.91) .. controls (388.1,149.22) and (389.16,150.29) .. (390.47,150.29) .. controls (391.78,150.29) and (392.85,149.22) .. (392.85,147.91) .. controls (392.85,146.6) and (391.78,145.54) .. (390.47,145.54) .. controls (389.16,145.54) and (388.1,146.6) .. (388.1,147.91) -- cycle ;
\draw  [color={rgb, 255:red, 0; green, 0; blue, 0 }  ,draw opacity=1 ][fill={rgb, 255:red, 0; green, 0; blue, 0 }  ,fill opacity=1 ] (434.15,101.57) .. controls (434.15,102.88) and (435.22,103.94) .. (436.53,103.94) .. controls (437.84,103.94) and (438.9,102.88) .. (438.9,101.57) .. controls (438.9,100.26) and (437.84,99.19) .. (436.53,99.19) .. controls (435.22,99.19) and (434.15,100.26) .. (434.15,101.57) -- cycle ;
\draw  [color={rgb, 255:red, 0; green, 0; blue, 0 }  ,draw opacity=1 ][fill={rgb, 255:red, 0; green, 0; blue, 0 }  ,fill opacity=1 ] (362.15,173.85) .. controls (362.15,175.17) and (363.22,176.23) .. (364.53,176.23) .. controls (365.84,176.23) and (366.9,175.17) .. (366.9,173.85) .. controls (366.9,172.54) and (365.84,171.48) .. (364.53,171.48) .. controls (363.22,171.48) and (362.15,172.54) .. (362.15,173.85) -- cycle ;

\end{tikzpicture}
\\
\textbf{lable}
&
$d_1\delta_1(8)$
&
$d_2\delta_1(10)$
&
$d_3\delta_1(12)$
&
$d_4\delta_1(14)$
&
$d_5\delta_1(16)$
\\
\textbf{degree}
&
$20$
&
$104$
&
$320$
&
$760$
&
$1540$\\ \hline
\end{tabular}
}
\end{center}


    \caption{All minimal problems discovered in Section \ref{sec:lines}. Each row contains \emph{all} minimal problems in the given setting (except for $\delta=0, d=3$). The lines represent 3D lines, while the points are sampled on their image curves. The label identifies a problem as $dD\delta M(N)$, where $D=\deg(C)$, $M=\deg(A)$, and $N=n_1^{e_1},\ldots,n_k^{e_k}$ means that $e_i$ lines have $n_i$ points sampled on their image curves. $P$ means that the lines are parallel and $PC$ means that they are parallel and coplanar. For the underlined labels, we evaluated minimal solvers in Section \ref{sec:experiments}.}
    \label{tab:minimal}
\end{table}

Following the strategy described above, we find all minimal problems in a given setting by finding the balanced problems first.
Here, we explain the general principle on how to enumerate all balanced SfM problems in Sec. \ref{sec:lines}. 
There, the DoF of the parameters to be recovered is of the form $\Gamma + \Lambda \ell$, where $\Gamma$ is the DoF of the camera plus potential joint information of the $\ell$ lines (e.g., in case of parallel or coplanar lines) and $\Lambda$ is the individual DoF of each world line (both after modding out all global symmetries). 
Often $\Lambda$ is $4$, but in Sec. \ref{sec:pure_rotation} and \ref{sec:pureTranslation}, it is also $2$ or $1$.
Choosing $N_i$ points on the image curve of the $i$-th line yields $\sum_{i=1}^l N_i$ constraints on structure and motion.  
Note that each $N_i$ has to be chosen larger than $\Lambda$ as the points on the $i$-th image curve otherwise do not constrain the camera parameters.
Furthermore, $N_i$ should be at most the dimension of the space of all possible image curves (determined in Prop. \ref{prop:D0ImageCurveVariety}, \ref{prop:pureTranslationDominant} and \ref{prop:delta2VaryingCam}) as measuring more points per curve would only measure noise. 
All in all, the balanced problems are those satisfying
\begin{gather}
\label{eq:pointsinlinesDoF}
    \begin{split}
    \Gamma + \Lambda \ell = \textstyle{\sum_{i=1}^l N_i}, \\
    \Lambda < N_i \leq \dim \mathrm{im}(\mathcal{P}_{d,\delta} \times \mathrm{Gr}(1, \PP^3) \to \mathcal{H}_{1+d+2\delta}) \text{ for all } i = 1, \ldots, \ell.
    \end{split}
\end{gather}
From this, we can easily enumerate all balanced problems, and then check minimality for each of them. For the cases studied in Sec. \ref{sec:lines}, this results in Table~\ref{tab:minimal}.

\begin{example}[$d=1,\delta=2$] \label{ex:minProblsD1Delta2}
    We discuss minimal problems arising from measuring points on image curves obtained from a single generic RS camera in $\mathcal{P}_{1,2}$ observing $\ell$ lines.
The image curves have degree $6$, i.e., they live in the $12$-dimensional space $\mathcal{H}_6$.
By Proposition \ref{prop:delta2VaryingCam}, all such curves form an $11$-dimensional subvariety.
In other words, they form a hypersurface in  $\mathcal{H}_6$ that is defined by a single equation in the curve coefficients, analogously to the equation in Section~\ref{sec:longEquation}.

From the image of a single line (i.e., $\ell=1$), we cannot recover the part of the camera motion in the same direction as the line. 
However, by Proposition \ref{prop:delta2VaryingCam}, we can reconstruct the remaining motion parameters (up to global scale, translation, rotation). Thus, using the notation in \eqref{eq:pointsinlinesDoF}, we have $\Gamma = 2d+3\delta-1=7$ and $\Lambda=4$.
Hence,  choosing $N=11$ points on the image curve yields a balanced problem, which we verified with GB to be  minimal of degree 1964  (note that this coincides with the degree of the hypersurface in $\mathcal{H}_6$).

Given images of more lines (i.e., $\ell > 1$), the line direction ambiguity is in general no longer a problem.
Therefore, we have $\Gamma = 3d+3\delta-1=3+6-1=8$  and $\Lambda = 4$, and so  
\eqref{eq:pointsinlinesDoF} becomes in this case $8+4\ell = \sum_{i=1}^\ell N_i$, $5\leq N_i\leq 11$. 
All possible values for $N_i$ (up to reordering the lines) are:
\begin{center}
\begin{tabular}{|c|l|}
\hline 
    $\;\ell = 2\;$ & $\;(11,5),(10,6),(9,7),\mathbf{(8,8)}$ \\
    \hline
    $\ell = 3$ & $\;(10,5,5),(9,6,5),\mathbf{(8,7,5)},\mathbf{(8,6,6)}, \mathbf{(7,7,6)}$ \\
    \hline 
    $\ell = 4$ & $\;(9,5,5,5),\mathbf{(8,6,5,5)},\mathbf{(7,7,5,5)},\mathbf{(7,6,6,5)}, \mathbf{(6,6,6,6)}\;$ \\
    \hline 
    $\ell = 5$ & $\;\mathbf{(8,5,5,5,5)},\mathbf{(7,6,5,5,5)}, \mathbf{(6,6,6,5,5)}$ \\
    \hline 
    $\ell = 6$ & $\;\mathbf{(7,5,5,5,5,5)}, \mathbf{(6,6,5,5,5,5)}$ \\
    \hline 
    $\ell = 7$ & $\;\mathbf{(6,5,5,5,5,5,5)}$ \\
    \hline 
    $\ell = 8$ & $\;\mathbf{(5,5,5,5,5,5,5,5)}$ \\
    \hline 
\end{tabular} 
\end{center}\hfill$\diamondsuit$
\end{example}

\begin{example}[$d=3,\delta=0$] \label{ex:minProblsD3Delta0}
     We again use the notation in \eqref{eq:pointsinlinesDoF}.
     After modding out  global rotation, translation and scaling, we have $\Gamma = 3\delta+3d-1=8$ (assuming $\ell>1$) and $\Lambda=4$. Thus, we obtained the same balanced equation as above: $8+4\ell=\sum_{i=1}^\ell N_i$. The lower bound of $5$ for the $N_i$ remains unchanged; however, by Prop. \ref{prop:pureTranslationDominant}, the upper bound is the dimension of $\mathcal{H}_4$, which is  $8$ by Prop.~\ref{prop:H}. Thus,  the balanced problems are the bold ones in the table in Example \ref{ex:minProblsD1Delta2}.
\hfill$\diamondsuit$
\end{example}

\section{Proofs}\label{sec:Proofs}

Here, we provide proofs for the assertions made in the main paper.
Some of our proofs are computational, relying on symbolic computations in \texttt{Macaulay2} \cite{M2} and \texttt{Oscar} \cite{OSCAR,OSCAR-book} or numerical computations with \texttt{Julia} \cite{Julia-2017} using the \texttt{HomotopyContinuation.jl} package \cite{HomotopyContinuation.jl}. All the code mentioned in this section can be found in \href{https://github.com/sofiaemd/Single-View-Rolling-Shutter-SfM}{https://github.com/sofiaemd/Single-View-Rolling-Shutter-SfM}.

\begin{proof} [of Theorem \ref{thm:order}]
    To compute the order, we work with complex and projective spaces. In particular, we use homogeneous coordinates $x = [v:t] \in \PP^1$ for the position of the scanline, such that the usual affine coordinates lie in the affine chart where $t=1$ and the point $(1:0)$ represents infinity.
Then, we interpret the camera position $C(x)$ in the projective space $\PP^3$ by homogenizing  its polynomial coordinate functions, i.e.
$C = (a,b,c,t^d)^\top \in (\RR[v,t]_d)^4$.
We homogenize $A(x)$ in the same way.
The camera matrix for the scanline at position $x$ is now
$P(x) = R(x) \begin{bmatrix} t^d I | -\tilde{C}(x)  \end{bmatrix}$, where $\tilde{C} = (a,b,c)^\top$.

The order of the camera equals the degree of the (now homogeneous) polynomial $\Sigma^\vee(x) \cdot  X = r(x) P(x) X$. 
Since $X$ is a (generic) constant, the degree is equal to the degree of the map $\PP^1 \to (\PP^{3})^\ast, x \mapsto \Sigma^\vee(x) =  r(x) P(x)$. That degree is at most  $\deg r + \deg C + \deg R = 1+d+2\delta$.

    It remains to prove that the 3 coordinate functions of $\Sigma^\vee(x)$ do not have any common factors, since, in that case, the actual degree of the map $\Sigma^\vee$ would be lower. Since the set of $(C,A)$ for which $\Sigma^\vee(x)$ has common factors is a subvariety of $\mathcal{P}_{d,\delta}$, it suffices to find a single example where this does not happen. For this, we set $b = t^d$, $c = a$, $\beta = t^\delta$, $\gamma = \alpha$.
    After this substitution, the second entry of $\Sigma^\vee(x)$ is equal to $\Sigma_2^\vee(x) = -4 \alpha t^{d+\delta} v $. The last element of $\Sigma^\vee(x)$ is  $\Sigma_0^\vee(x) = 2a\alpha^2v + 4\alpha t^{d+\delta}v - 2a t^{2 \delta} v - 2a\alpha^2t - 2at^{2 \delta+1}$. For generic polynomials  $a \in \RR[v,t]_d$ and $\alpha \in  \RR[v,t]_\delta$, we have that $v$ does not divide $\Sigma_0^\vee(x)$, $t$ does not divide $\Sigma_0^\vee(x)$, and none of the complex linear factors of $\alpha$ divides $\Sigma_0^\vee(x)$. Thus, $\Sigma_2^\vee(x)$ and $\Sigma_0^\vee(x)$ are generically coprime, and the order of the RS camera is indeed $1+d+2\delta$.
    \qed
\end{proof}

\begin{proof} [of Theorem \ref{thm:imageCurve}]
We continue the proof from the main paper.
     As in the proof of Theorem \ref{thm:order} above, it is sufficient to find a single example, where $u_1,u_2,u_0$ do not have common factors (since having common factors is a Zariski closed property). For this, we take a line $L$, whose Plücker coordinates are $q = (0,1,0)$ and $\Delta=(1,0,0)$ and set the motion parameters to $b=t^d$, $c = a$, $\beta =t^\delta$, $\gamma = \alpha$.
    In that case,  $u_2$ and $u_0 = t \cdot \tilde u$ become
    \begin{align*}
        u_2 = -2 \alpha^2 t^d v - 2 t^{d+2 \delta} v + 4 a \alpha  t^{\delta+1} - 4  \alpha  t^{d+\delta+1}, \\
        \tilde u = 2 a \alpha^2  - 2 t^d \alpha^2  - 2 a t^{2\delta}  + 2  t^{d+2\delta}  = 2  (\alpha - t^\delta) (\alpha + t^\delta) (a - t^d).
    \end{align*}
    We can see that $t$ is not a factor of $\tilde u$ for generic choices of $a \in \mathbb{R}[v,t]_d$ and $\alpha \in \mathbb{R}[v,t]_\delta$, and thus $t$ does not divide $u_1 = v \cdot \tilde u$ generically.
    Hence, the only way how $u_1,u_2,u_0$ can have a common factor is if one of the complex linear factors of $\tilde u$ divides $u_2$.
    We prove now that this does not happen generically.
    To see this, we begin with considering a complex linear factor $\lambda$ of $\alpha-t^\delta$, i.e., $\alpha-t^\delta = \lambda \cdot \tilde \alpha$ for some polynomial $\tilde \alpha$ of degree $\delta-1$.
    If $\lambda$ were to divide $u_2$, then $u_2 \equiv 0 \mod \lambda$, but since we also have $\alpha\equiv t^\delta \mod \lambda$, we obtain
    \begin{align*}
        0 \equiv u_2 \equiv   (-4t^d  v + 4a  t - 4t^{d+1}) t^{2\delta} \mod \lambda.
    \end{align*}
    However, $\lambda \neq t$  (as otherwise it would also be a factor of $\alpha = t^\delta + \lambda \cdot \tilde \alpha$ and so $A(x)$ would have been of degree lower than $\delta$) and for generic choices of $a \in \mathbb{R}[v,t]_d$, $\lambda$ is also not a factor of $t^d v-at+t^{d+1}$. 
    Thus, generically, $\lambda$ cannot divide $u_2$.

    Similarly, if a linear factor $\lambda$ of $\alpha+t^\delta$ or $a-t^d$ would divide $u_2$, we can exploit the relation $\alpha \equiv -t^\delta \mod \lambda$ or $a \equiv t^d \mod \lambda$, respectively, to derive
    \begin{align*}
        &0 \equiv u_2 \equiv 
         (-4t^d v - 4a  t + 4  t^{d+1}) t^{2\delta} 
        \mod \lambda \quad \text{ or }\\
        &0 \equiv u_2 \equiv
        (-2 \alpha^2  - 2 t^{2 \delta} ) t^d v
        \mod \lambda, 
    \end{align*}
    respectively, both of which are impossible for generic choices of $a \in \mathbb{R}[v,t]_d$ and  $\alpha \in \mathbb{R}[v,t]_\delta$.
    This shows that $u_1,u_2,u_0$ do not have any common factors generically. 
    \qed
\end{proof}

\begin{proof}[of Proposition \ref{prop:spanImLine}]
To verify that the linear span is full dimensional, we picked $\ell(2+2d+4\delta)$ generic cameras and $\ell$ generic world lines for each of them. Then we computed the matrix whose rows are the coefficients of the curves obtained when taking a picture of these lines through the respective cameras and checked that its determinant does not vanish. See submitted code. \qed
\end{proof}

\begin{proof}[of Proposition \ref{prop:D0ImageCurvePlane}]
  The image of a world line whose span with the camera center $[0\, 0\, 0]^\top$ is described by the coefficient vector $q\in \mathbb{P}^2$ is a curve parametrized by $u(x) = r(x) \times (R(x) q)$. 
    Since this expression is linear in $q $, it is clear that the set of all image curves is a projective subspace of $\mathcal{H}_{1+2\delta}$ of dimension at most two. 
    We now prove that this dimension equals two. 
    As we know from the proof of Theorem \ref{thm:imageCurve} that the coordinate functions $[u_1:u_2:u_0]$ of $u$ do not share a common factor, it is sufficient to investigate the coefficients of $u_0(x)$.
    In the affine chart $t=1$, we have $r(x) = [1:0:-x]^\top$ and so $u_0(x)$ is the second row of $R(x)$ times $q$.
    By modding out the global rotation, we may assume that $R(0)$ is the identity matrix.
    Writing $q = [q_1:q_2:1]^\top$ and $R(x) = q2r(A(x))$ with $A(x) = (\alpha(x),\beta(x),\gamma(x))$ and $\alpha(x) = \sum_{j=1}^\delta \alpha_j x^j,\;\beta(x) = \sum_{j=1}^\delta \beta_j x^j,\;\gamma(x) = \sum_{j=1}^\delta \gamma_j x^j$, we see that
    \begin{equation} \label{eq:u0Prop4}
        u_0(x) = (q_2 + 2(\gamma_1q_1-\alpha_1)x + \mathrm{h.o.t.}).
    \end{equation}
    Thus, for $\gamma_1 \neq 0$, the linear map that assigns to $q$ the coefficients of $u_0$ is bijective onto its image.

To prove the second statement, we generalize Example \ref{ex:curve_d0del1} as follows: First, we describe how to get the analog of \eqref{eq:matrixRepresentingPlane} as a unique representation for the plane in $\mathcal{H}_{1+2\delta}$. Then, for any fixed $\delta$, we  argue inductively on the coefficients of $A(x)$ to show that the plane uniquely determines the camera. The proof is constructive and thus provides an algorithm for computing the camera from the plane: 

We start by spelling out the entries of $u(x):$
\begin{gather*}
    u(x)=\begin{bmatrix}
        1 &
        0 &
        -x
    \end{bmatrix}^\top \times q2r (\alpha(x),\beta(x),\gamma(x))\begin{bmatrix}
        q_1&
        q_2&
        1
    \end{bmatrix}^\top=\\
    \left[\begin{smallmatrix}
        x(2(\alpha(x) \beta(x) + \gamma(x))q_1 +(1 - \alpha(x)^2 + \beta(x)^2 - \gamma(x)^2)q_2 +2(\beta(x) \gamma(x) - \alpha(x))) \\
        \\
-x((1 + \alpha(x)^2 - \beta(x)^2 - \gamma(x)^2)q_1 + 2(\alpha(x) \beta(x) - \gamma(x))q_2 +2(\alpha(x) \gamma(x) + \beta(x)))\\
-(2(\alpha(x) \gamma(x) - \beta(x))q_1+2(\beta(x) \gamma(x) + \alpha(x))q_2 +1 - \alpha(x)^2 - \beta(x)^2 + \gamma(x)^2) \\
\\
2(\alpha(x) \beta(x) + \gamma(x))q_1 +(1 - \alpha(x)^2 + \beta(x)^2 - \gamma(x)^2)q_2 +2(\beta(x) \gamma(x) - \alpha(x))
    \end{smallmatrix}\right].
\end{gather*}
The parametrization \eqref{eq:AffineCurveParam} of the curve is obtained by dividing the second row by the last one, which gives it the form $y=\frac{\sum_{i=0}^{2\delta+1}Z_{i}x^i}{\sum_{i=0}^{2\delta}Y_{i}x^i}$.  
We consider the matrix $M$ that has a column for each coefficient $Z_i, Y_i$, and $3$ rows corresponding to $1,q_1,q_2$: The first row of the matrix contains in each column the terms of $Z_i$ resp. $Y_i$ that do not contain any $q_j$, in the second row the terms with $q_1$, and in the third row the terms with $q_2$. We observe now that $Z_{0}=-1, \;Y_{1}=2\gamma_1q_1-2\alpha_1\text{, and }
        Y_{0}=q_2,$ see \eqref{eq:u0Prop4}.
   These give rise to the following columns of the matrix $M$: {\small\begin{gather*}
        \begin{bmatrix}
            -1\\
            0\\
            0
        \end{bmatrix}, \begin{bmatrix}
            -2\alpha_1\\
            2\gamma_1\\
            0
        \end{bmatrix},\text{ and }\begin{bmatrix}
            0\\
            0\\
            1
        \end{bmatrix}.
    \end{gather*}}
    Whenever $\gamma_1\neq0$, the columns form a full-rank $3\times 3$ submatrix $M$. 
    Thus, multiplying $M$ with the inverse of that $3 \times 3$ matrix 
    gives a normal form for the matrix $M$ containing the standard basis vectors 
    as in \eqref{eq:matrixRepresentingPlane}.
    This normal form is a unique representation of the plane in $\mathcal{H}_{1+2\delta}$. 
    
    From here on, we will proceed by induction on the index of the coefficients of $\alpha(x),\beta(x),$ and $\gamma(x)$. We will show as a base case that $\alpha_1,\beta_1,$ and $\gamma_1$ are uniquely determined by this normal form. The induction step will assume that $\alpha_k,\beta_k, \gamma_k$ are uniquely determined for $k\leq i$ and show that this implies that $\alpha_{i+1},\beta_{i+1},\gamma_{i+1}$ are  uniquely determined as well.

    For the base case, we look at the columns corresponding to $Z_1,Z_{2\delta+1},$ and $Y_{2\delta}$. 
    We only care about the second and third entry of the $Z_1$-column. 
    In the matrix $M$, they encode the terms $-q_1 + 2\beta_1q_1-2\alpha_1q_2$.
    After the normalization, they become the column vector $\begin{bmatrix}
        * & \frac{2\beta_1-1}{2\gamma_1} & -2\alpha_1
    \end{bmatrix}^\top$, where $*$ refers to entries that are irrelevant for our argument. In particular, we observe that $\alpha_1$ is uniquely determined. 
    The only entry we need from $Z_{2\delta+1}$ is the last one, which is the  maximal-degree term containing $q_2$, i.e.,  $-2\alpha_\delta \beta_\delta$. So the $Z_{2\delta+1}$-vector looks like $\begin{bmatrix}
        * &* & -2\alpha_\delta \beta_\delta
    \end{bmatrix}^\top$. Finally, from $Y_{2\delta}$ we only need the second entry, i.e., the maximal-degree term containing $q_1$, which  is $2\alpha_\delta \beta_\delta$. Thus, after normalizing, the $Y_{2\delta}$-vector looks like $\begin{bmatrix}
        *& \frac{2\alpha_\delta \beta_\delta}{2\gamma_1} & *
    \end{bmatrix}^\top$. 
    Combining this with the $Z_{2\delta+1}$-vector, we see that 
    $\gamma_1$  is uniquely determined. Now, the second entry of the $Z_1$-column  determines $\beta_1$ uniquely, finishing the base case.

    For the induction step, we assume now that $\alpha_k,\beta_k$, and $\gamma_k$ are uniquely determined for $1\leq k\leq i$ and show that the same holds for $k=i+1$. To achieve this, we study the columns associated with $Y_{i+1}$ and $Z_{i+1}$. 
    Starting with $Y_{i+1}$, we claim that the second entry involves --  besides a linear appearance of $\gamma_{i+1}$ --  only coefficients $\alpha_k,\beta_k,\gamma_k$ with $k\leq i$.
    The terms in the last entry of $u$ involving $q_1$ are $2(\alpha(x)\beta(x)+\gamma(x))$, so the term of  degree $i+1$ is 
    $2(\sum_{k=1}^i\alpha_{i+1-k}\beta_k + \gamma_{i+1})$ since
    $\alpha_0=\beta_0=0$.
    Thus, after normalizing, the $Y_{i+1}$-column is     $\begin{bmatrix}
        *& \frac{\gamma_{i+1}}{\gamma_1} + \text{l.i.t.} & *
    \end{bmatrix}^\top, $  where l.i.t. stands for lower index terms, and we can recover $\gamma_{i+1}$ from this uniquely.

    Similarly, we study the $q_1$ and $q_2$ entries of the $Z_{i+1}$-column. Since the first summand of $u_2(x)$ is multiplied with $x$, it can only contribute coefficients with index at most $i$ to both entries. 
    Thus, it only contributes values we already computed. 
    We can argue analogously to the above that the products $\alpha(x)\gamma(x)$ and $\beta(x)\gamma(x)$ cannot contain coefficients of index larger than $i$. 
    Thus, the only terms involving $q_1$ or $q_2$ and indices larger than $i$ are
    $2\beta_{i+1}q_1-2\alpha_{i+1}q_2$. Hence, after normalizing,  the $Z_{i+1}$-column is    $\begin{bmatrix}
        *& \frac{\beta_{i+1}}{\gamma_1}+\text{l.i.t.} & -2\alpha_{i+1} + \text{l.i.t.}
    \end{bmatrix}^\top,$ from which we can recover also $\alpha_{i+1}$ and $\beta_{i+1}$ uniquely.
    \qed
\end{proof}

\begin{proof}[of Proposition~\ref{prop:D0ImageCurveVariety}]
The second statement implies the first since cameras in $\P_{0,\delta}$ (up to global rotation and translation) have $3\delta$ DoF and the world plane through the fixed camera center has $2$ DoF. 
To prove the second statement, we consider the space $\mathcal{P}$ of cameras in $\mathcal{P}_{0,\delta}$ centered at the origin and with $R(0) = I$, the space $\mathcal{G}$ of world planes passing through the origin, and the picture-taking map $\mathcal{P} \times \mathcal{G} \to \mathcal{H}_{1+2 \delta}$.
The second statement claims that that map is birational. 
We can check this via the criterion in SM Lemma~\ref{lem:birationalCheck}, i.e.,  by finding a single image curve that uniquely determines the camera in $\mathcal{P}$ and the plane in $\mathcal{G}$, and such that the Jacobian of the picture-taking map at that point ($=$ the camera $+$ the plane) has full rank. 
We verified this computationally (see attached code).
\qed
\end{proof}

\begin{proof}[of Proposition \ref{prop:Delta0ImageCurveFixedCamera}]
We see from \eqref{eq:lineImageAtTimeX} that the image of a world line depends linearly on its Pl\"ucker coordinates. Writing $q = [q_1,q_2,q_3]^\top$ and $\Delta = [\Delta_1,\Delta_2,\Delta_3]^\top$ for those coordinates and $C(x) = [a(x),b(x),c(x)]^ \top$ for the camera center,   \eqref{eq:lineImageAtTimeX} becomes  {\small \begin{gather*}
    u(x)=\begin{bmatrix}
        x(q_2+\Delta_3a(x)-\Delta_1c(x))\\
        x(-q_1+\Delta_3b(x)-\Delta_2c(x))-q_3+\Delta_2a(x)-\Delta_1b(x)\\
        q_2+\Delta_3a(x)-\Delta_1c(x)
    \end{bmatrix}.
\end{gather*}}
As before, we denote by $Z_i$ and $Y_i$ the coefficients of the numerator and  denominator of $\frac{u_2}{u_0}$, respectively. Fixing the global translation by $C(0) = [0,0,0]^\top$, we see that $Z_0=-q_3$. So in the affine chart $q_3=1$, the image curve is already in a normal form and lives in an affine subspace of $\mathcal{H}_{1+d}$. Furthermore, we have {\small\begin{equation}\label{eq:puretranscoeff}
\begin{aligned}
&Z_1 = -q_1+\Delta_2a_1-\Delta_1b_1,\\
&Z_i = \Delta_3b_{i-1}-\Delta_2c_{i-1}+\Delta_2a_i-\Delta_1b_i\quad &(2\le i\le d),\\
&Z_{d+1} = \Delta_3b_d-\Delta_2c_d,\\
&Y_0 = q_2,\\
&Y_i = \Delta_3a_i-\Delta_1c_i \quad (1\le i\le d).
\end{aligned}
\end{equation}}
    For $d>1$ and generic $a_i,b_i$, and $c_i$, the six coefficients $Y_0,Y_1,Y_2,Z_0,Z_1,Z_2$ yield a full-rank linear map $\PP^5 \to \PP^5$
    \begin{align*}
\left[ \begin{smallmatrix}
            q \\ \Delta
        \end{smallmatrix} \right] \mapsto        \left[ \begin{smallmatrix}
          0 & 1 & 0 & 0 & 0 & 0 \\
          0 & 0 & 0 & -c_{1} & 0 & a_{1} \\
          0 & 0 & 0 & -c_{2} & 0 &a_{2} \\
          0 & 0 & -1 & 0 & 0 & 0 \\
          -1 & 0 & 0 & -b_{1} & a_{1} & 0 \\
          0 & 0 & 0 & -b_{2} & a_{2}-c_{1} & b_{1}
        \end{smallmatrix} \right] \cdot \left[ \begin{smallmatrix}
            q \\ \Delta
        \end{smallmatrix} \right].
    \end{align*}
    Thus, the camera induces a full-rank linear map  $\Gamma: \mathbb{P}^5 \to \mathcal{H}_{1+d}$, whose restriction to the Pl\"ucker quadric $q_1 \Delta_1+q_2\Delta_2+q_3\Delta_3 = 0$ is a linear isomorphism between world lines and their image curves. This proves the first assertion of Proposition \ref{prop:Delta0ImageCurveFixedCamera}.
    
    Next, we show that the image of the Grassmanian $\mathrm{Gr}(1, \PP^3)$ under $\Gamma$ determines the camera uniquely.
    In fact, we prove that 6 generic points on $\Gamma(\mathrm{Gr}(1, \PP^3))$ already determine the camera.
    To see this, we take 6 generic lines and observe that $Y_0^j,Y_1^j,Y_d^j, Z_1^j,$ and $Z_{d+1}^j$ for $j=1, \ldots, 6$  uniquely determine both $a_1,b_1,c_1,a_d,b_d,c_d$ and the parameters of the lines (see code in the SM) in the affine chart where $q_3^j=1$ and $c_d=1$. Note that, since $Z_0=-q_3$, any line producing a curve with $Z_0=-1$ indeed has to be in the affine chart $q_3=1$.  Our code additionally checks that there are no further solutions with $c_d=0$ and that the Jacobian of the map taking the parameters to the coefficients $Y_i^j,Z_i^j$ has indeed full rank generically.
    Then, we can apply Lemma \ref{lem:birationalCheck} to conclude that the image curves of 6 generic lines uniquely determine the lines and the $a_1,b_1,c_1,a_d,b_d,c_d$. To derive the uniqueness of the remaining camera parameters, we use $Y_i^1$ and $Y_i^2$ for $i=1, \ldots, d$ to obtain the equation\begin{gather*}
    \begin{bmatrix}
        \Delta^1_3 & \quad-\Delta_1^1\\
        \Delta^2_3 & \quad-\Delta_1^2
    \end{bmatrix}\cdot \begin{bmatrix}
        a_i\\
        c_i
    \end{bmatrix}=\begin{bmatrix}
        Y_i^1\\
        Y_i^2
    \end{bmatrix}.
\end{gather*}  
For generic $\Delta^j$, the matrix is invertible and thus $a_i,c_i$ are uniquely determined from the image measurements $Y_i^j$. Given now all parameters but the $b_i$, we can easily reconstruct those from $Z_i^1$.

    For $d=1$, we have $a_2=b_2=c_{2}=0$, and so the five coefficients $Y_0,Y_1,Z_0,Z_1,$ $Z_2$ yield the  linear map $\Gamma: \mathbb{P}^5 \to \mathcal{H}_2$ in \eqref{eq:linearCamMapD1Delta0}.
    If $c_{1}\neq 0$, its kernel is spanned by $\boldsymbol{c} := [0\,0\,0\,a_{1}\,b_{1}\,c_{1}]^\top$.
    This point lies on the Pl\"ucker quadric (it represents the world line along which the camera is moving).
    Therefore, the linear map restricted to the Pl\"ucker quadric is one-to-one almost everywhere. 
    In fact, for every point $(u_0(x) = Y_0+Y_1x, u_2(x) = Z_0+Z_1x+Z_2x^2)$ in $\mathcal{H}_2$, its preimage under the linear map is a projective line $\Gamma^{-1}(u) \subseteq \mathbb{P}^5$ that passes through the kernel point $\boldsymbol{c}$.
    If $\Gamma^{-1}(u)$ is not a tangent line of the Pl\"ucker quadric at $\boldsymbol{c}$, 
    then it intersects the Pl\"ucker quadric at exactly one additional point, which means that there is precisely one world line that the camera maps to the image curve parametrized by $u(x)$.
    Otherwise (i.e., if $\Gamma^{-1}(u)$ is a tangent line), the curve parametrized by $u(x)$ cannot be obtained by the camera. This happens when $(Y_0,Y_1,Z_0,Z_1,Z_2)$ lies in the image of the tangent space of the Pl\"ucker quadric at $\boldsymbol{c}$ under the linear map \eqref{eq:linearCamMapD1Delta0}, which happens precisely when 
    \begin{align}
        -b_{1}c_{1}Y_0-a_{1}b_{1}Y_1+c_{1}^2Z_0+a_{1}c_{1}Z_1+a_{1}^2Z_2=0.
        \label{eq:discriminantDelta0D1}
    \end{align}
    Note that \eqref{eq:discriminantDelta0D1} is precisely the denominator $\Xi$ in the triangulation formulas in Example \ref{ex:D1Delta0}.
  \qed
\end{proof}
\begin{proof}[of Proposition \ref{prop:pureTranslationDominant}]
The second claim follows from the first by applying the Fiber-Dimension Theorem (see SM Theorem \ref{thm:fiberDim}) to the picture-taking map $\mathcal{P}_{d,0} \times \mathrm{Gr}(1,\PP^3) \to \mathcal{H}_{1+d}$. In fact, the fiber over a generic image curve has dimension $(3d+6) + 4 - 2(d+1)$, which -- after subtracting $7$ DoF for the global scaling / translation / rotation -- equals $d+1$, as claimed. 

To prove the first statement,     we consider the affine chart in the Plücker coordinates for the line where we normalize $q_3=1$. The Plücker relation can then be used to eliminate $\Delta_3$ via $\Delta_3=-q_1\Delta_1-q_2\Delta_2$. 
To understand the dimension of the subvariety of $\mathcal{H}_{1+d}$ consisting of curves that are the image of some line through a camera in $\mathcal{P}_{0,d}$, we use SM Lemma \ref{lem:jacobianCheck} and study the Jacobian of the map that takes a camera and a line and maps it to its image curve. The map is again the one that takes the parameters and maps to the coefficients of the curve in  \eqref{eq:puretranscoeff}. As we have more variables than coefficients, it suffices to consider selected directional derivatives corresponding to a non-zero minor. Thus, we will consider the following partial derivatives: $\frac{\partial}{\partial a_i}, \frac{\partial}{\partial c_i} $ for $1\leq i\leq d$, $\frac{\partial}{\partial q_2}$, and $\frac{\partial}{\partial \Delta_1}$. We order them as follows: at first alternating the $\frac{\partial}{\partial a_i}$ and $\frac{\partial}{\partial c_i}$ and finally the last two. Then, we obtain the following  (partial) gradients:{\small \begin{gather*}
    \nabla Z_1=\begin{bmatrix}
        \Delta_2& 0 &\dots & 0 & * & *
    \end{bmatrix},\\
    \nabla Z_i=\begin{bmatrix}
        0 & \dots & 0 & -\Delta_2 & \Delta_2 & 0& \dots & 0 & * &*
    \end{bmatrix}  \text{ for }2\leq i\leq d,\\
    \nabla Z_{d+1}=\begin{bmatrix}
        0& \dots &0&-\Delta_2 & *& \frac{\partial}{\partial \Delta_1}Z_{d+1}
    \end{bmatrix},\\
    \nabla Y_0=\begin{bmatrix}
        0& \dots & 0 & 1 & 0
    \end{bmatrix},\\
    \nabla Y_i=\begin{bmatrix}
        0& \dots & 0 & \Delta_3 & -\Delta_1&0 & \dots & 0 & *&*
    \end{bmatrix},
\end{gather*}}
where $*$ refers to entries that are irrelevant for our argument. 
Combining all these, we obtain the partial Jacobian \begin{gather*}
    J=\left[\begin{smallmatrix}
         \Delta_2& 0&0&0 &\dots & 0&0 & * & *\\
         0 & -\Delta_2 & \Delta_2 & 0& \dots & 0&0 & * &*\\
         \vdots & \ddots & \ddots & \ddots& -\Delta_2 & \Delta_2 &0& * &*\\
         0& \dots & \dots & \dots & 0 &0 &-\Delta_2 & * & \frac{\partial}{\partial \Delta_1}Z_{d+1}\\
         0 & 0 &0 &0 & \dots & \dots & 0 & 1 & 0\\
         \Delta_3 & -\Delta_1& 0 & 0 & \dots & 0 & 0 & * &*\\
         0 & 0 & \Delta_3 & -\Delta_1&0 &\dots & 0& * &*\\
         0 & \ddots & \ddots&\ddots & \ddots & \Delta_3 & -\Delta_1 & * &\frac{\partial}{\partial \Delta_1}Y_{d}
    \end{smallmatrix}\right].
\end{gather*}
By SM Lemma \ref{lem:jacobianCheck},
it suffices to see that this matrix has a non-zero determinant at some point. We make the following choices, $\Delta_1=1,\Delta_2=1,q_1=0,q_2=0$, and $c_d=1$. All other parameters can be set randomly (or simply to $0$). By the Plücker relation, we get $\Delta_3=0$. We can now use Laplace expansion to compute the absolute value of the determinant of $J$. First, we observe that we can expand the row corresponding to $Y_0$, which just gives a sign that vanishes under the absolute value. Next, we observe that all columns corresponding to partial derivatives $\frac{\partial}{\partial a_i}$ contain only one non-zero entry corresponding to $\Delta_2=1$ (since $\Delta_3=0$), so we can expand on those columns without changing the absolute value. Since the previous expansions removed the $-\Delta_2$ entry of the columns corresponding to $\frac{\partial}{\partial c_i}$ for $i<d$, these have now only one non-zero entry, namely $-\Delta_1=-1$. Thus, after further expansion on those columns, we get \begin{gather*}
    |\det(J)|= \left\lvert \det\left(\begin{bmatrix}
        -\Delta_2 & \frac{\partial}{\partial \Delta_1}Z_{d+1}\\
        -\Delta_1 & \frac{\partial}{\partial \Delta_1}Y_{d}
    \end{bmatrix}\right)\right\rvert.
\end{gather*}
To finish the proof, we compute the remaining partial derivatives and afterwards the latter determinant: $\frac{\partial}{\partial \Delta_1}Z_{d+1}=\frac{\partial}{\partial \Delta_1}(-q_1\Delta_1-q_2\Delta_2)b_d-\Delta_2c_d= -q_1b_d$ and $\frac{\partial}{\partial \Delta_1}Y_{d}=\frac{\partial}{\partial \Delta_1}a_d(-q_1\Delta_1-q_2\Delta_2)-c_d\Delta_1=-q_1a_d-c_d$. Thus, the absolute value of the Jacobian at the chosen point is \begin{gather*}
    |\det(J)|=|\Delta_2q_1a_d+\Delta_2c_d-\Delta_1q_1b_d|=|0+1+0|=1.
\end{gather*}\qed
\end{proof}

\begin{proof}[of Proposition \ref{prop:ImageCurveFixedCamera}] 
This proof proceeds similarly to the proof of Proposition \ref{prop:Delta0ImageCurveFixedCamera}.
As before, we write $C(x) = (a(x), b(x), c(x))$, $A(x)= (\alpha(x),\beta(x),\gamma(x))$, and assume -- due to global translation and rotation -- that $C(0) = 0$ and $A(0) = 0$.
Further, we again denote by $[\Delta: q]$ the Plücker coordinates of a world line and record the image curve as $u_2(x) = \sum_{i=0}^{1+d+2\delta} Z_i x^i$ and $u_0(x) = \sum_{i=0}^{d+2\delta} Y_i x^i$.
For a generic choice of coefficients $a_i ,b_i, c_i, \alpha_i,\beta_i$ and $\gamma_i$, the six coefficients $Y_0,Y_1,Y_{2\delta+d},Z_0,Z_1$ and $Z_{2\delta+d+1}$ yield a full-rank linear map $\mathbb{P}^5\to\mathbb{P}^5$
    \begin{align*}
\left[ \begin{smallmatrix}
            q \\ \Delta
        \end{smallmatrix} \right] \mapsto        \left[ \begin{smallmatrix}
          0 & 1 & 0 & 0 & 0 & 0 \\
          2\gamma_1 & 0 & -2\alpha_1 & -c_{1} & 0 & a_{1} \\
          0 & 0 & 0 & c_d\nu+2\beta_\delta\gamma_\delta b_d & 2\beta_\delta(\alpha_\delta c_d-\gamma_\delta a_d) & -a_d\nu-2\alpha_\delta\beta_\delta b_d \\
          0 & 0 & -1 & 0 & 0 & 0 \\
          -1+2\beta_1 & -2\alpha_1 & 0 & -b_{1} & a_{1} & 0 \\
          0 & 0 & 0 & 2\alpha_\delta(\beta_\delta c_d-\gamma_\delta b_d) & c_d\mu+2\alpha_\delta\gamma_\delta a_d & -b_d\mu-2\alpha_\delta\beta_\delta a_d
        \end{smallmatrix} \right] \cdot \left[ \begin{smallmatrix}
            q \\ \Delta
        \end{smallmatrix} \right],
    \end{align*}
    where $\nu := \alpha_\delta^2-\beta^2_\delta+\gamma_\delta^2$ and $\mu:=-\alpha_\delta^2+\beta_\delta^2+\gamma_\delta^2$. 
    Thus, the linear map $\Gamma:\mathbb{P}^5\to\mathcal{H}_{2\delta+d+1}$ sending the Plücker coordinates to all coefficients $Y_i,Z_i$ has full rank, and so its restriction to the Plücker quadric is a linear isomorphism, proving the first statement in Proposition \ref{prop:ImageCurveFixedCamera}.

    To prove that the subvariety $\Gamma(\mathrm{Gr}(1,\PP^3)) $ uniquely determines the camera for $(d,\delta)\in \{(1,1),(1,2),(2,1)\}$, it suffices to show that two points in the subvariety uniquely determine the camera. We check this computationally (see submitted code) by studying the two-line picture-taking map $\mathcal{P}_{d,\delta}\times \mathrm{Gr}(1, \mathbb{P}^3)^2\to \mathcal{H}_{1+d+2\delta}^2$ that takes a camera and two world lines and maps them to the coefficients of  the two resulting image curves.
We proceed in three steps in order to derive the desired result from SM Lemma \ref{lem:birationalCheck}: 1) fix a camera and two world lines, 2) solve a polynomial system to see that we can recover the chosen camera and lines from the image curves,  and 3) check that the two-line picture-taking map has full-rank at the chosen point in its domain.

    For the first step, we choose a sufficiently generic camera and world lines $((C,A),L_1,L_2)\in\mathcal{P}_{d,\delta}\times \mathrm{Gr}(1, \mathbb{P}^3)^2$. For convenience in later calculations, we pick it so that $a_d=b_d=c_d=1$ 
    as well as $q_3=1$ in both lines.

    Next, to implement the map as a polynomial system, we consider $8$ variables for the two lines and $3\delta + 3d - 1$ variables for the camera parameters (by setting $a_d=1$ to get rid of the global scaling).
    We then compute the polynomials in those variables that define the curve coefficients for the image of both lines, i.e. the polynomials defining the $Z_i$'s and $Y_i$'s for both curves.

    For the second step, we construct a square polynomial subsystem of the previous one, consisting of $2\delta+d+3$ polynomials coming from the first curve and $\delta+2d+4$ from the second curve. We then use the monodromy technique (which is based on numerical homotopy continuation; see e.g. \cite{DBLP:journals/pami/DuffKLP24}) to find all choices of variables that yield the same values as $((C,A),L_1,L_2)$ when evaluating the polynomials in the square system.
    We finally check which of those solutions make all coefficients $Y_i,Z_i$ from the original system equal to the ones of the first step, and conclude that the only possible choice is $((C,A),L_1,L_2)$. We also check that this is true in the alternative charts where $b_d=1$ or $c_d=1$.

    For the third step, we compute that the Jacobian of the two-line picture-taking map at the chosen point $((C,A),L_1,L_2)$ has full rank. Then, by SM Lemma \ref{lem:birationalCheck}, we obtain that the generic fiber of that map is a point, i.e., generically, the image of two  lines uniquely determines the camera.
    \qed
\end{proof}

\begin{proof}[of Proposition \ref{prop:delta1VaryingCam}]
   In this proof, we write   $Z_i,Y_i$ for the coefficients of the image curve $y=\frac{-\sum_{i=0}^{d+3} Z_i x^i}{\sum_{i=0}^{d+2} Y_i x^i}$.
    We modify the picture-taking map $\mathcal{P}_{d,1}\times \mathrm{Gr}(1, \mathbb{P}^3)\to \mathcal{H}_{d+3}$ by modding out global scaling, rotation, translation and  line direction ambiguity as follows:
    We consider the subset $\mathcal{P} \subseteq \mathcal{P}_{d,1}$ of cameras $(C,A) = ((a,b,c),(\alpha,\beta,\gamma))$ with $C(0)=0$, $A(0)=0$, $c_d = 1$, and $b_i=0$ for all $i=1,\ldots,d$. 
    Writing $[\Delta:q]$ for the Plücker coordinates of 3D lines, we work with the affine chart $\mathcal{G}$ of $\mathrm{Gr}(1,\PP^3)$ where $q_3=1$. 
    Restricting the picture-taking map to these spaces yields the map  $\rho: \mathcal{P} \times \mathcal{G} \to \mathcal{H}$, where $\mathcal{H}$ is the affine chart of $\mathcal{H}_{d+3}$ where $Z_0 = -1$.

    In the following, we prove that almost every curve in $\mathcal{H}$ is in the image of $\rho$. This has two consequences: 1) It shows the first claim in Prop. \ref{prop:delta1VaryingCam}. 
    2) Since  $\dim (\mathcal{P} \times \mathcal{G}) = 2(d+3) = \dim \mathcal{H}$, the Fiber-Dimension Theorem \ref{thm:fiberDim}  yields that the generic fiber of $\rho$ is zero-dimensional, i.e., a finite non-empty set, proving the second claim in Prop. \ref{prop:delta1VaryingCam}. The number of solutions are computed in the submitted code.

    To prove that almost every curve in $\mathcal{H}$ is in the image of $\rho$, it suffices to show that the dimension of the image of $\rho$ equals $\dim \mathcal{H}$. 
    For that, we use the Jacobian check in  SM Lemma \ref{lem:jacobianCheck}. 
        For $d \in \{ 1,2,3\}$, we compute the rank of the Jacobian of $\rho$ in the submitted code.
        For $d \geq 4$, 
    we begin by spelling out the map $\rho$:
    \small
    \begin{align*}
        Y_0 &= q_2,\quad
        Y_1 = \Delta_3a_1-\Delta_1c_1+2\gamma q_1-2\alpha,\\
        Y_2 &= \Delta_3a_2-\Delta_1c_2+2\gamma\Delta_2c_1+2\alpha\Delta_2a_1+\nu q_2+2\alpha\beta q_1+2\beta\gamma,\\
        Y_i &= \Delta_3a_i-\Delta_1c_i+2\gamma\Delta_2 c_{i-1}+2\alpha\Delta_2a_{i-1}
        +\nu(\Delta_3a_{i-2}-\Delta_1 c_{i-2})\\&\hspace{1.5cm}+2\alpha\beta\Delta_2 c_{i-2}-2\beta\gamma\Delta_2a_{i-2},\\
        Y_{d+1} &=2\gamma\Delta_2+2\alpha\Delta_2a_d+\nu(\Delta_3a_{d-1}-\Delta_1c_{d-1})+2\alpha\beta\Delta_2c_{d-1}-2\beta\gamma\Delta_2a_{d-1},\\
        Y_{d+2} &= \nu(\Delta_3a_d-\Delta_1)+2\alpha\beta\Delta_2-2\beta\gamma\Delta_2a_d,
    \end{align*}
    \normalsize
    where $\nu := -\alpha^2+\beta^2-\gamma^2$ and $3\leq i\leq d$;  and
    \small
    \begin{align*}
        &Z_1 = \Delta_2a_1+2\beta q_1-2\alpha q_2-q_1,\\
        &Z_2 =\Delta_2a_2+2\beta\Delta_2c_1-2\alpha(\Delta_3 a_1-\Delta_1c_1)+\upsilon
        -2\alpha\gamma q_1-2\beta\gamma q_2-\Delta_2 c_1+2\gamma q_2-2\beta,\\
        &Z_3 = \Delta_2a_3-(1-2\beta)\Delta_2c_2-2\alpha(\Delta_3a_2-\Delta_1c_2)-(\upsilon-2\beta)\Delta_2a_1\\&\hspace{1.2cm}-2\alpha\gamma\Delta_2c_1-(2\beta\gamma-2\gamma)(\Delta_3a_1-\Delta_1c_1)+\mu q_1-2\alpha\beta q_2-2\alpha\gamma,\\
        &Z_i = \Delta_2a_i-(1-2\beta)\Delta_2c_{i-1}-2\alpha(\Delta_3a_{i-1}-\Delta_1c_{i-1})-(\upsilon-2\beta)\Delta_2a_{i-2}\\&\hspace{1.2cm}
        -2\alpha\gamma\Delta_2c_{i-2}-(2\beta\gamma-2\gamma)(\Delta_3a_{i-2}-\Delta_1c_{i-2})+\mu\Delta_2c_{i-3}\\&\hspace{1.2cm}
        -2\alpha\beta (\Delta_3a_{i-3}-\Delta_1c_{i-3})+2\alpha\gamma\Delta_2a_{i-3},\displaybreak\\
        &Z_{d+1} =-(1-2\beta)\Delta_2-2\alpha(\Delta_3a_{d}-\Delta_1)-(\upsilon-2\beta)\Delta_2a_{d-1}
        -2\alpha\gamma\Delta_2c_{d-1}+\mu\Delta_2c_{d-2}\\&\hspace{1.5cm}-(2\beta\gamma-2\gamma)(\Delta_3a_{d-1}-\Delta_1c_{d-1})
        -2\alpha\beta (\Delta_3a_{d-2}-\Delta_1c_{d-2})+2\alpha\gamma\Delta_2a_{d-2},\\
        &Z_{d+2} =-(\upsilon-2\beta)\Delta_2a_d-2\alpha\gamma\Delta_2-(2\beta\gamma-2\gamma)(\Delta_3a_d-\Delta_1)\\&\hspace{1.5cm}+\mu\Delta_2c_{d-1}-2\alpha\beta (\Delta_3a_{d-1}-\Delta_1c_{d-1})+2\alpha\gamma\Delta_2a_{d-1},\\
        &Z_{d+3} = -2\alpha\beta(\Delta_3 a_d-\Delta_1 )+\Delta_2(\mu+2\alpha\gamma a_d),
    \end{align*}
    \normalsize
    where $\upsilon := \alpha^2+\beta^2-\gamma^2$, $\mu :=-\alpha^2+\beta^2+\gamma^2$ and $4\leq i\leq d$. By SM Lemma~\ref{lem:jacobianCheck}, we only need to verify that the Jacobian has nonzero determinant at some point. This chosen point will be the camera with rotation parameters $\alpha=0$, $\beta=\frac{1}{2}$ and $\gamma=0$,  center parameters $a_i=0$ for $1\leq i\leq d$,  $c_1=\frac{1}{2}$ and $c_i=0$ for $2\leq i\leq d-1$, and the line with parameters $q_1=\Delta_1=0$, $\Delta_2=q_2=1$, $\Delta_3=-1$.

    Notice that for any $1\leq i\leq d-1$, we have that $\frac{\partial}{\partial c_i}Z_j=0$ if $j\notin \{i+1,i+2,i+3\}$ and
    \begin{gather*}
    \frac{\partial}{\partial c_i}Z_{i+1}=2\beta\Delta_2+2\alpha\Delta_1-\Delta_2,\;\frac{\partial}{\partial c_i}Z_{i+2}=-2\alpha\gamma\Delta_2+2(\beta-1)\gamma\Delta_1,\\
    \frac{\partial}{\partial c_i}Z_{i+3}=\mu\Delta_2+2\alpha\beta\Delta_1.
    \end{gather*}
    Also, $\frac{\partial}{\partial c_i}Y_j=0$ if $j\notin \{i,i+1,i+2\}$ and
    \begin{gather*}
    \frac{\partial}{\partial c_i}Y_{i}=-\Delta_1,\;\frac{\partial}{\partial c_i}Y_{i+1}=2\gamma\Delta_2,\;\frac{\partial}{\partial c_i}Y_{i+2}=-\nu\Delta_1+2\alpha\beta\Delta_2.
    \end{gather*}
    In the chosen point all these partial derivatives become $0$ except for $\frac{\partial}{\partial c_i}Z_{i+3}=-\frac{1}{4}$. Then we can do a Laplace expansion of the Jacobian on all its columns corresponding to $\frac{\partial}{\partial c_i}$ for $1 \leq i \leq d-1$, which also removes all rows corresponding to $Z_4, \ldots, Z_{d+2}$.
    It now suffices to show that the resulting $(d+7)\times (d+7)$ matrix has nonzero determinant.
    Its rows correspond to the functions $Z_{d+3},Z_3,Z_2,Z_1,$ $Y_{d+2},\dots,Y_0$, and its columns to the variables $a_1,\dots, a_d,\alpha,\beta,\gamma,q_1,q_2,\Delta_1,\Delta_2$ (note that in our affine chart $q_3=1$, we have $\Delta_3=-\Delta_1q_1-\Delta_2q_2$).

    Next, we consider the partial derivatives of the remaining equations with respect to $q_1$. At the chosen point, we have that $\frac{\partial }{\partial q_1}\Delta_3 =0$, and so we see that $\frac{\partial}{\partial q_1}Z_i=0$ for every $i\neq3$, $\frac{\partial}{\partial q_1}Z_3=\frac{1}{4}\neq0$  and $\frac{\partial}{\partial q_1}Y_i=0$ for every $i$.
    Thus,  we can do a Laplace expansion in the column corresponding to $q_1$, which also removes the $Z_3$-row.
    In addition, we can  do a Laplace expansion on the row corresponding to $Y_0$, which also removes the $q_2$-column.
    We are now left with a $(d+5)\times (d+5)$ matrix and need to show that its determinant is nonzero.
    
At the chosen point, if we order the rows following $Z_{d+3},Z_2,Z_{1},Y_{d+2},\dots, Y_1$ and the columns following $a_1,\dots, a_d,\alpha,\beta,\gamma,\Delta_1,\Delta_2$, we obtain the matrix:
    \begin{gather*}
        \left[\begin{smallmatrix}
            0 & 0 & 0 & 0 & 0& \dots& 0 &0 & 1 & 0 & 0 &\frac{1}{4}\\
            0 & 1 & 0 & 0 & 0 &\dots & 0 & 0 & 0 & 1 & 0 & 0\\
            1 & 0 &0 &0 &0&\dots & 0 & -2 & 0 & 0 & 0 & 0\\
            0 & 0&0 &0&0 &\dots &-\frac{1}{4}& 1 & 0 & 0&-\frac{1}{4}&0\\
            0 & 0 & \dots & \dots &\dots &-\frac{1}{4} & 0& 0 & 0 & 2 & 0&0\\
            0 & 0 &\dots &\dots & -\frac{1}{4} & 0 & -1& 0 & 0 & 0 & -1 & 0\\
            0 & 0 &\dots  & -\frac{1}{4} & 0 & -1& 0 & 0 & 0 & 0 & 0 & 0\\
            0 & \iddots& \iddots &\iddots& \iddots & \iddots & 0 & 0 & 0 & 0 & 0 & 0\\
            0 & -\frac{1}{4}&0&-1& 0 & \dots &\dots & 0 & 0 & 0& 0 & 0\\
            -\frac{1}{4}& 0 & -1& 0 &\dots &\dots &\dots  &\frac{1}{2} & 0 & 0 & -\frac{1}{8}&0\\
            0 & -1 & 0 & 0 &\dots &\dots &\dots & 0 & 1 & 2 & 0 & 0\\
            -1 & 0& 0&0&\dots &\dots &\dots  & -2 & 0 & 0 & -\frac{1}{2}&0
        \end{smallmatrix}\right]
    \end{gather*}
    We can expand on the last column and then in the forth to last column to obtain

    \begin{gather*}
        \left[\begin{smallmatrix}
            0 & 1 & 0 & 0 & 0 &\dots & 0 & 0 &1&0\\
            1 & 0 &0 &0 &0&\dots & 0 & -2&0 &0\\
            0 & 0&0 &0&0 &\dots &-\frac{1}{4}& 1& 0&-\frac{1}{4}\\
            0 & 0 & \dots & \dots &\dots &-\frac{1}{4} & 0& 0&2 & 0\\
            0 & 0 &\dots &\dots & -\frac{1}{4} & 0 & -1& 0 &0&-1\\
            0 & 0 &\dots  & -\frac{1}{4} & 0 & -1& 0 &0&0&0\\
            0 & \iddots& \iddots &\iddots& \iddots & \iddots & 0 &0 &0 &0\\
            0 & -\frac{1}{4}&0&-1& 0 & \dots &\dots & 0& 0& 0\\
            -\frac{1}{4}& 0 & -1& 0 &\dots &\dots &\dots  &\frac{1}{2}&0& -\frac{1}{8}\\
            -1 & 0& 0&0&\dots &\dots &\dots  &-2&0& -\frac{1}{2}
        \end{smallmatrix}\right]
    \end{gather*}
To compute the (absolute value of the)  determinant, we start by performing some basic row operations. We subtract $\frac{1}{4}$ times the last row from the penultimate one,  we add the last row to the second row, and we add $\frac{1}{4}$ times the first row to the third to last. These operations do not change the determinant and we obtain the following matrix\begin{gather*}
    \left[\begin{smallmatrix}
            0 & 1 & 0 & 0 & 0 &\dots & 0 & 0 &1&0\\
            0 & 0 &0 &0 &0&\dots & 0 & -4&0 &-\frac{1}{2}\\
            0 & 0&0 &0&0 &\dots &-\frac{1}{4}& 1& 0&-\frac{1}{4}\\
            0 & 0 & \dots & \dots &\dots &-\frac{1}{4} & 0& 0&2 & 0\\
            0 & 0 &\dots &\dots & -\frac{1}{4} & 0 & -1& 0 &0&-1\\
            0 & 0 &\dots  & -\frac{1}{4} & 0 & -1& 0 &0&0&0\\
            0 & \iddots& \iddots &\iddots& \iddots & \iddots & 0 &0 &0 &0\\
            0 & 0&0&-1& 0 & \dots &\dots & 0& \frac{1}{4}& 0\\
            0& 0 & -1& 0 &\dots &\dots &\dots  &1&0& 0\\
            -1 & 0& 0&0&\dots &\dots &\dots  &-2&0& -\frac{1}{2}
        \end{smallmatrix}\right]
\end{gather*} 
Applying Laplace expansion to the first two columns yields
\begin{gather*}\left[\begin{smallmatrix}
            0 &0 &0&\dots & 0 & -4&0 &-\frac{1}{2}\\
            0 &0&0 &\dots &-\frac{1}{4}& 1& 0&-\frac{1}{4}\\
            0 & \dots &\dots &-\frac{1}{4} & 0& 0&2 & 0\\
            0 & \dots & -\frac{1}{4} & 0 & -1& 0 &0&-1\\
            \dots  & -\frac{1}{4} & 0 & -1& 0 &0&0&0\\
            0 & \iddots& \iddots &\iddots& \iddots & \iddots & 0 &0 \\
            -\frac{1}{4}& 0& -1 & 0 &\dots &\dots &\dots &0\\
            0&-1& 0 & \dots &\dots & 0& \frac{1}{4}& 0\\
             -1& 0 &\dots &\dots &\dots  &1&0& 0\\
        \end{smallmatrix}\right]
\end{gather*}
Now, we will perform the following two steps iteratively: 1) subtract $\frac{1}{4}$ times the last row from the third to last row and also $\frac{1}{4}$ of the penultimate row from the fourth to last row; 2) Laplace expand the first two columns. We see inductively that after applying these steps $i$-times, we obtain 
\begin{gather*}
    \left[\begin{smallmatrix}
            0 &0 &0&\dots & 0 & -4&0 &-\frac{1}{2}\\
            0 &0&0 &\dots &-\frac{1}{4}& 1& 0&-\frac{1}{4}\\
            0 & \dots &\dots &-\frac{1}{4} & 0& 0&2 & 0\\
            0 & \dots & -\frac{1}{4} & 0 & -1& 0 &0&-1\\
            \dots  & -\frac{1}{4} & 0 & -1& 0 &0&0&0\\
            0 & \iddots& \iddots &\iddots& \iddots & \iddots & 0 &0 \\
            -\frac{1}{4}& 0& -1 & 0 &\dots &\dots &\dots &0\\
            0&-1& 0 & \dots &\dots & 0& \frac{1}{4}\left(-\frac{1}{4}\right)^i& 0\\
             -1& 0 &\dots &\dots &\dots  &\left(-\frac{1}{4}\right)^i&0& 0\\
        \end{smallmatrix}\right].
\end{gather*}
Hence, depending on whether the size of the matrix was odd or even, we reduced our problem to computing the determinant of either of the following 2 matrices: 
{\small\begin{gather*}
    \begin{bmatrix}
            0 & 0 & -4&0 &-\frac{1}{2}\\
            0 &-\frac{1}{4}& 1& 0&-\frac{1}{4}\\
             -\frac{1}{4} & 0& 0&2 & 0\\
             0 & -1& 0 &\frac{1}{4}\left(-\frac{1}{4}\right)^i&-1\\
             -1& 0 &\left(-\frac{1}{4}\right)^i&0&0         
        \end{bmatrix},  \begin{bmatrix}
             0 & -4&0 &-\frac{1}{2}\\
           -\frac{1}{4}& 1& 0&-\frac{1}{4}\\
             0& 0&2+\frac{1}{4}\left(-\frac{1}{4}\right)^i & 0\\
             -1& \left(-\frac{1}{4}\right)^i &0&-1       
        \end{bmatrix}.
\end{gather*}}
Starting with the first case, we perform the two steps again and obtain now {\small\begin{gather*}
   \det \begin{bmatrix}
        -4 & 0 & -\frac{1}{2}\\
        1 & \frac{1}{4}\left(-\frac{1}{4}\right)^{i+1}& 0\\
        \left(-\frac{1}{4}\right)^{i+1}& 2 & 0
    \end{bmatrix}=-\frac{1}{2}\det\begin{bmatrix}
        1 & \frac{1}{4}\left(-\frac{1}{4}\right)^{i+1}\\
        \left(-\frac{1}{4}\right)^{i+1}& 2
    \end{bmatrix}\\
    =-\frac{1}{2}\left(2-\frac{1}{4}\left(-\frac{1}{4}\right)^{2i+2}\right),
\end{gather*}}
which cannot be zero for $i\geq 0$.
In the case of the second matrix, we expand along the third row and so the problem reduces to showing that {\small\begin{gather*}
    0\neq \det\begin{bmatrix}
        0 & -4& -\frac{1}{2}\\
        -\frac{1}{4}& 1 & -\frac{1}{4}\\
        -1 &\left(-\frac{1}{4}\right)^i& -1
    \end{bmatrix}= \det \begin{bmatrix}
        0 & -4& -\frac{1}{2}\\
        0& 1+\left(-\frac{1}{4}\right)^{i+1} & 0\\
        -1 &\left(-\frac{1}{4}\right)^i& -1
    \end{bmatrix}.
\end{gather*}}
After expanding along the second row, it thus suffices to see that $\det\begin{bmatrix}
    0 & -\frac{1}{2}\\
    -1 & -1
\end{bmatrix}\neq 0,$ which is obviously true.
\qed
\end{proof}

\begin{proof}[of Proposition \ref{prop:delta2VaryingCam}]
    We study the picture-taking map $\rho: \mathcal{P}_{d,\delta}\times \mathrm{Gr}(1, \mathbb{P}^3)\to \mathcal{H}_{1+d+2\delta}$ that takes a camera and a line and computes the corresponding image curve.
    From Proposition \ref{prop:parallel}, we know that, for generic $((C,A), L) \in \mathcal{P}_{d,\delta}\times \mathrm{Gr}(1, \mathbb{P}^3)$, the fiber $\rho^{-1}(\rho((C,A), L))$ contains all $((C',A'), L')$ that are obtained from $((C,A), L)$ by global scaling, translation, rotation, and the line direction ambiguity. We show in the following that the fiber does not contain anything else. This then concludes the proof of Proposition \ref{prop:delta2VaryingCam} by the Fiber-Dimension Theorem \ref{thm:fiberDim}.

    We fix the known ambiguities in the fiber as follows:
    When the scanline is at position $x=0$, the camera is assumed to be at the origin with identity rotation. If the line direction has a nonzero first coordinate, we can also fix the parameters of the first coordinate of the camera center (e.g., setting $a(x)=0)$ to get rid of line direction ambiguity. Finally, if $b_d\neq 0$, we can fix the global scaling by assuming $b_d=1$.

    We check computationally (see submitted code) that the  generic image curve $\rho((C,A), L)$ uniquely determines the remaining camera parameters and the world line by applying SM Lemma \ref{lem:birationalCheck} in a similar manner as in the proof of Proposition \ref{prop:ImageCurveFixedCamera}: 
    We construct the polynomial system that defines the picture-taking map (i.e. the polynomials that define the $Y_i$'s and $Z_i$'s). Then, we use monodromy to find the solutions to a square subsystem of these polynomials (i.e.,  of  $3\delta+2d+3$ many polynomials),  yielding a finite amount of solutions. Afterwards, we verify that only one of those solutions satisfies the whole system (this solutions is of course $((C,A), L)$). Finally, we compute the Jacobian of the map at $((C,A), L)$ and verify that it has full rank. 
    \qed
\end{proof}

 \begin{proof}[of Lemma~\ref{lem:pointRel}]
        By \eqref{eq:orderPoint}, the  $x_i^j$ are exactly the points in time when a picture of point $j$ is taken, i.e., they all have to satisfy the equation {\small\begin{gather*}
           0= \begin{bmatrix}
      1 & 0 & -x
  \end{bmatrix}\cdot R(x) \cdot \begin{bmatrix}
                1 & 0 & 0 & -a(x)\\
                0 & 1 & 0& -b(x)\\
                0& 0& 1 & -c(x)
            \end{bmatrix}\cdot \begin{bmatrix}
                X_1^j\\
                X_2^j\\
                X_3^j\\
                1
            \end{bmatrix}
            =\begin{bmatrix}
      1 & 0 & -x
  \end{bmatrix}\cdot R(x)\cdot \begin{bmatrix}
                X_1^j-a(x)\\
                X_2^j-b(x)\\
                X_3^j-c(x)
            \end{bmatrix}.
        \end{gather*}}
        As in the proof of Theorem \ref{thm:order}, we know that this is a polynomial of degree $1+d+2\delta$ in $x$. Thus, the $x_i^j$ are indeed all roots of this polynomial for a fixed $j$. Moreover, we see that $X_1^j,X_2^j,X_3^j$ appear only in monomials of degree at most $1+2\delta$. Since $d\geq 2$, we thus know that they do not appear in the coefficient of $x^{d+2\delta}$ as $d+2\delta> 1+2\delta$. Hence, this coefficient is independent of $j$, but (after dividing by the leading coefficient) the second highest coefficient of a univariate polynomial is precisely the sum of its roots, which proves the claim.
        \qed
    \end{proof}

    \begin{proof}[of Proposition \ref{prop:spanImagePoints}]
    Similarly to the proof of Proposition \ref{prop:spanImLine}, we checked this numerically by sampling points and estimating the rank of the resulting matrix by inspecting the singular values. See code in SM.
    \qed
\end{proof}

\begin{proof}[of Theorem \ref{thm:point_minimal_problems}]
A minimal problem must be balanced, i.e., the number $3p+3d+3\delta-1$  of unknowns (or $3\delta+2p$ in the case $d=0$) must equal the number \eqref{eq:ptConstraintsNumber} of constraints.
 In the case $d=0$, this means $3\delta+2p=2p(1+2\delta)=2p+4p\delta$; however,  for $p\geq1$, the right-hand side is strictly larger than the left. So there is no balanced problem with $d=0$, and we assume $d>0$ from now on.

Here, we first consider the case $p=1$. Then, the two numbers in \eqref{eq:ptConstraintsNumber} agree and we get the balanced equation\begin{gather*}
      3d+3\delta +2=2(1+d+2\delta)\iff d=\delta,
  \end{gather*} which yields the third kind of problem in Theorem~\ref{thm:point_minimal_problems}.
  
  Next, we assume  $p \geq 2$ and $d = 1$. In this case, the balanced equation is \begin{gather*}
      p-1=(2p-3)+(4p-3)\delta.
  \end{gather*} 
  As $2\leq p$, both $2p-3$ and $4p-3$ are positive. Furthermore, $p-1< 4p-3$ for any such $p$ and thus we get $\delta=0$. 
  Now, the only solution is  $p=2$, giving the second problem in Theorem~\ref{thm:point_minimal_problems}.
  
  To complete the classification of balanced problems, it remains to consider the case $p \geq 2$ and $d \geq 2$. This time, the balanced equation can be written as  \begin{gather*}
      2(p-1)=(2p-3)d+(4p-3)\delta.
  \end{gather*} Again, it holds that $4p-3>2(p-1)$ and thus $\delta=0$. Since $d\geq 2$, we get that $2(p-1)=(2p-3)d\geq 4p-6$, which yields $2\geq p$. Hence, $p=2$ is the only possibility, which implies $d=2$, i.e, the first problem in Theorem~\ref{thm:point_minimal_problems}. 
  
  In the following, we always assume that the constant rotation matrix is the identity. To prove minimality of degree $1$ for the first two problems in Theorem~\ref{thm:point_minimal_problems}, we study the corresponding equation system.
   We assume that $\delta=0, p=2, d \in \{ 1,2\}$.
  By \eqref{eq:orderPoint}, we know that $(x,y)$ is an image of a world point $X$ if and only if $x$ satisfies  $r(x)\cdot P(x)\cdot X=0$,  where $r(x)=\begin{bmatrix}
      1 & 0 & -x
  \end{bmatrix}$. If $(P(x)X)_3\neq 0$, then the second coordinate satisfies $y=(P(x)X)_2/(P(x)X)_3$, or equivalently $(P(x)X)_2=y\cdot (P(x)X)_3$.
 We show that for a generic tuple of two dimensional vectors $(\tilde x_{i,1},\tilde y_{i,1}), \ldots, $ $(\tilde x_{i,d+1},\tilde y_{i,d+1})$ in the codomain of the picture-taking relation there exist a unique camera and unique world point giving rise to the vectors as images. The genericity of these vectors implies in particular that we may assume that the $\tilde x_{i,1}, \ldots, \tilde x_{i,d+1}$ are pairwise distinct. 
 We now consider all solutions $X^1, X^2, P(x)$ to the equations 
 \begin{align} \label{eq:solveForCamsAndPts}
     r(\tilde x_{i,j})\cdot P(\tilde x_{i,j}) \cdot X^i = 0 \quad \text{and} \quad
     (P(\tilde x_{i,j}) X^i)_2 = \tilde y_{i,j} \cdot(P(\tilde x_{i,j})X^i)_3
 \end{align}
 for all $i=1,2$ and $j=1,\ldots,d+1$.
 We claim that all such solutions must live in the affine chart  $X^i_4 =1$.
 Otherwise, if $X^i_4$ were zero, then we would have that $P(x)X^i=\begin{bmatrix}
      X_1^i&X_2^i&X_3^i
  \end{bmatrix}^\top$ (since $\delta=0$) and this expression does not depend on $x$ anymore. In particular, the  first coordinate of each corresponding image point is  $\frac{X_1^i}{X_3^i}$ (note that $X_3^i$ cannot be zero as  the generic image point we started with was not at infinity). However, this contradicts our assumption that $\tilde x_{i,1} \neq \tilde x_{i,2}$.
  
  The equations \eqref{eq:solveForCamsAndPts} can be written as
  \begin{gather*}
X_1^i-a(\tilde x_{i,j})-\tilde x_{i,j}X^i_3+\tilde x_{i,j}c(\tilde x_{i,j})=0, \,  \tilde y_{i,j}\cdot (X_3^i-c(\tilde x_{i,j}))-X_2^i+b(\tilde x_{i,j})=0.
  \end{gather*}
 These equations are linear in the unknowns, which are the coefficients of $a(x),$ $b(x),c(x)$ and the entries of the $X^i$. Assuming $a_0=b_0=c_0=0$ (by modding out the global translation)
 and combining all those linear equations in the case $d=1$, we obtain the matrix\begin{gather*}
      \left[\begin{smallmatrix}
        -\tilde x_{1,1} & 0 & \tilde x_{1,1}^2 & 1 & 0 & -\tilde x_{1,1} & 0 & 0 &0\\
        0 & \tilde x_{1,1} & -\tilde y_{1,1} \tilde x_{1,1} & 0 & -1 & \tilde y_{1,1}& 0 & 0 &0\\
        -\tilde x_{1,2} & 0 & \tilde x_{1,2}^2 & 1 & 0 & -\tilde x_{1,2} & 0 & 0 &0\\
        0 & \tilde x_{1,2} & -\tilde y_{1,2} \tilde x_{1,2} & 0 & -1 & \tilde y_{1,2}& 0 & 0 &0\\
        - \tilde x_{2,1} & 0 & \tilde x_{2,1}^2 & 0 & 0 & 0 & 1 & 0 &-\tilde x_{2,1}\\
        0 & \tilde x_{2,1} & -\tilde y_{2,1} \tilde x_{2,1} & 0 & 0 & 0& 0 & -1 &\tilde y_{2,1}\\
        -\tilde x_{2,2} & 0 & \tilde x_{2,2}^2 & 0 & 0 & 0 & 1 & 0 &-\tilde x_{2,2}\\
        0 & \tilde x_{2,2} & -\tilde y_{2,2} \tilde x_{2,2} & 0 & 0 & 0& 0 & -1 &\tilde y_{2,2}
      \end{smallmatrix}\right]\cdot \left[\begin{smallmatrix}
          a_1\\
          b_1\\
          c_1\\
          X^1_1\\
          X^1_2\\
          X^1_3\\
          X^2_1\\
          X^2_2\\
          X^2_3\\
      \end{smallmatrix}\right]=0.
  \end{gather*}
  Since we can still mod out the global scaling,  we can view the solution space to this equation as a projective space.
  This matrix is of full rank $8$, i.e., its kernel is 
  one-dimensional, which corresponds  exactly to a single projective point. This proves that there is generically a unique solution to the second SfM problem in Theorem \ref{thm:point_minimal_problems}. For the first SfM problem in that theorem (i.e., $d=2$), we can argue similarly by additionally using the linear relation from Lemma \ref{lem:pointRel}. We get:\[
     \left[ \begin{smallmatrix}
          -\tilde x_{1,1} & -\tilde  x_{1,1}^2 & 0 &0 & \tilde  x_{1,1}^2& \tilde  x_{1,1}^3& 1 & 0 & -\tilde  x_{1,1} & 0 & 0 &0\\
          0 & 0 & \tilde  x_{1,1}& \tilde  x_{1,1}^2 & -\tilde y_{1,1}\tilde x_{1,1} & -\tilde y_{1,1}\tilde x_{1,1}^2 & 0 & -1 & \tilde y_{1,1}& 0 & 0 &0\\
          -\tilde x_{1,2} & -\tilde x_{1,2}^2 & 0 &0 & \tilde x_{1,2}^2& \tilde x_{1,2}^3& 1 & 0 & -\tilde x_{1,2} & 0 & 0 &0\\
          0 & 0 & \tilde x_{1,2}& \tilde x_{1,2}^2 & -\tilde y_{1,2}\tilde x_{1,2} & -\tilde y_{1,2}\tilde x_{1,2}^2 & 0 & -1 & \tilde y_{1,2}& 0 & 0 &0\\
          -\tilde x_{1,3} & -\tilde x_{1,3}^2 & 0 &0 & \tilde x_{1,3}^2& \tilde x_{1,3}^3& 1 & 0 & -\tilde x_{1,3} & 0 & 0 &0\\
          0 & 0 & \tilde x_{1,3}& \tilde x_{1,3}^2 & -\tilde y_{1,3}\tilde x_{1,3} & -\tilde y_{1,3}\tilde x_{1,3}^2 & 0 & -1 & \tilde y_{1,3}& 0 & 0 &0\\
          -\tilde x_{2,1} & -\tilde x_{2,1}^2 & 0 &0 & \tilde x_{2,1}^2& \tilde x_{2,1}^3& 0 & 0 & 0 & 1 & 0 &-\tilde x_{2,1}\\
          0 & 0 & \tilde x_{2,1}& \tilde x_{2,1}^2 & -\tilde y_{2,1}\tilde x_{2,1} & -\tilde y_{2,1}\tilde x_{2,1}^2& 0 & 0 & 0& 0 & -1 &\tilde y_{2,1}\\
          -\tilde x_{2,2} & -\tilde x_{2,2}^2 & 0 &0 & \tilde x_{2,2}^2& \tilde x_{2,2}^3& 0 & 0 & 0 & 1 & 0 &-\tilde x_{2,2}\\
          0 & 0 & \tilde x_{2,2}& \tilde x_{2,2}^2 & -\tilde y_{2,2}\tilde x_{2,2} & -\tilde y_{2,2}\tilde x_{2,2}^2& 0 & 0 & 0& 0 & -1 &\tilde y_{2,2}\\
          0 & 0 & z& z^2 & -\tilde y_{2,3}z & -\tilde y_{2,3}z^2& 0 & 0 & 0& 0 & -1 &\tilde y_{2,3}
      \end{smallmatrix}\right]\cdot \left[ \begin{smallmatrix}
          a_1\\
          a_2\\
          b_1\\
          b_2\\
          c_1\\
          c_2\\
          X^1_1\\
          X^1_2\\
          X^1_3\\
          X^2_1\\
          X^2_2\\
          X^2_3\\
      \end{smallmatrix}\right]=0,
\]
where $z:=\tilde x_{1,1}+\tilde x_{1,2}+\tilde x_{1,3}-\tilde x_{2,1}-\tilde x_{2,2}$. The fact that the last row does no longer follow the pattern is due to Lemma \ref{lem:pointRel}, as $\tilde x_{2,3}=z$.
We can again show that the matrix is full-rank to obtain a single solution in the projective space arising from the global scaling ambiguity.

  Lastly, we have to describe how to test the minimality for the balanced problems with $p=1$. 
  For a GS camera, this can be done by applying SM Lemma~\ref{lem:jacobianCheck} to the picture-taking map (as in \cite{DBLP:journals/pami/DuffKLP24} and as explained at the end of Sec.~\ref{sec:AG}).
  More concretely, it is enough to find a single point in the domain where the Jacobian of the picture-taking map has full rank.
  
  In the case of an RS camera, this approach is not directly applicable, since the act of taking a picture is not as straightforward to model as an algebraic map;  rather, it is more naturally seen as the relation of a world point with its image points. More explicitly, the image of $p$ points taken by an RS camera $P\in \mathcal{P}_{d,\delta}$ can be seen as a projection of the following algebraic set.{\small\begin{align*}
      \mathcal{V}'&:=V(\lbrace r(x_{i,j}) P(x_{i,j}) X^i, y_{i,j}(P(x_{i,j}) X^i)_3-(P(x_{i,j}) X^i)_2\mid i\in [p], j\in [o]\rbrace) \\
      &\subseteq \mathcal{P}_{d,\delta}\times (\mathbb{K}^3)^p\times \left((\mathbb{K}^2)^o\right)^p,
  \end{align*}}

  \noindent
  where $o:=1+d+2\delta$ and $V(S)$ denotes the vanishing locus of a set $S$ of polynomial equations. The first equation means that the point $X^i$ is observed at time $x_{i,j}$ and the second equation ensures that the second coordinate of the image point has indeed the correct value. We then consider the Zariski closure $\mathcal{V}$ of the open subset of $\mathcal{V}$' that additionally satisfies  $x_{i,j}\neq x_{i,j'}$ for all $j\neq j'$ and also $(P(x_{i,j}) X^i)_3\neq 0$. Now, the image variety of points is the projection of $\mathcal{V}$ onto the variables $x_{i,j}$ and $y_{i,j}$. By Chevalley's theorem (Section 6, Theorem 6, \cite{matsumura1980commutative}), this is  a constructible set; in particular, it is dense if and only if it has full dimension. The latter condition can  be checked analytically since a set that contains an open $\varepsilon$-ball around some point has to be full-dimensional. 
  
  The advantage of this approach is that, even though our model for the act of taking pictures is just a relation (instead of a map), we can still show that---in a Euclidean neighborhood of a generic point---there are indeed functions that give us the roots of the equation for varying parameters. This can be proven using the implicit function theorem from multivariate calculus.

  To apply the implicit function theorem, we have to specify some function $
      F: \mathbb{C}^{n+m}\to \mathbb{C}^m$.
  In our setting, we  have $n=3d+3\delta +2$ representing  camera and world point parameters (modulo global rotation, translation and scaling) and $m=o=1+d+2\delta$. In the $i$-the dimension, this function is given by\begin{gather*} 
      F_i(a,b,c,\alpha,\beta,\gamma,X,x_1,\dots,x_{o})=r(x_{i}) P_{a,b,c,\alpha,\beta,\gamma}(x_{i}) X,
  \end{gather*}
  where $P_{a,b,c,\alpha,\beta,\gamma}(x_{i})$ is the RS camera parametrized by the coefficient vectors $a,b,c,\alpha,\beta,\gamma$. The condition for the implicit function theorem is that the Jacobian matrix containing only the $x_i$ derivatives is invertible. Since $F_i$ does not depend on $x_j$ for $j\neq i$, this Jacobian will be a diagonal matrix and the diagonal entries are the derivatives of the $F_i$ viewed  as one-variate polynomials. Hence, this matrix will be non-zero at a collection of $x_i$ making $F$ vanish if and only if the one-variate polynomial has no multiple roots. We will check this at one specific instance. Namely, let $\alpha_\delta=1=c_d$ and otherwise $\alpha_i=\beta_j=\gamma_j=a_j=b_j=c_j=0$ for all other camera parameters.
  Moreover, we choose $X=[0,6,7.5,1]^\top$. Then, $F_i$ becomes{\small\begin{gather*}
      r(t)\begin{bmatrix}
          1+t^{2d} & 0 & 0\\
          0 & 1-t^{2d}& -2t^d\\
          0& 2t^d& 1-t^{2d}
      \end{bmatrix}\begin{bmatrix}
          1 & 0 & 0 & 0\\
          0 & 1 & 0 & 0\\
          0 & 0 & 1 & -t^d
      \end{bmatrix}\begin{bmatrix}
          0 \\
          6\\
          7.5\\
          1
      \end{bmatrix}= \\
      \begin{bmatrix}
          1+t^{2d} & -2t^{d+1} & t^{2d+1}-t
      \end{bmatrix}\begin{bmatrix}
          0\\
          6\\
          7.5-t^d
      \end{bmatrix}= -11t^{d+1}+7.5t^{2d+1}-7.5t-t^{3d+1}.
  \end{gather*}}
  We see that $t=0$ is one solution, and after multiplying the rest with $-1$ and substituting $z=t^d$ we obtain the polynomial $z^3-7.5z^2+11z+7.5$ whose roots are precisely $-\frac{1}{2},3,$ and $5$. Hence, the roots of the given polynomial are the $d$-th roots of these and $0$. Therefore, all its roots are distinct. Hence, in a Euclidean neighborhood of our chosen camera and point, there exist differentiable functions $v_0,\dots v_{3d}$ which are the roots of  \eqref{eq:orderPoint}.

  With the help of these functions, we can describe the act of taking a picture of a point close to $X$ and with a camera close to our chosen $P$ as a differentiable map; namely, \begin{gather}\label{eq:MapJacobian}
      (a,b,c,\alpha,\beta,\gamma,X)\mapsto \begin{bmatrix}
          v_i(a,b,c,\alpha,\beta,\gamma,X)\\
          \frac{(P_{(a,b,c,\alpha,\beta,\gamma,X)}(v_i(a,b,c,\alpha,\beta,\gamma,X))X)_2}{(P_{(a,b,c,\alpha,\beta,\gamma,X)}(v_i(a,b,c,\alpha,\beta,\gamma,X))X)_3}
      \end{bmatrix}_{i=0}^{3d}.
  \end{gather}
  Since the $v_i$ are differentiable functions and also $P$ is differentiable in the camera parameters, this function itself is differentiable and thus, by the inverse function theorem, we know that its image contains an open ball if the Jacobian is invertible at some point. We will check computationally for $d\leq 5$ that this is the case for the point chosen above.

  First, we look at the second entries, as these are seemingly more complicated. Let $z$ now be any of the appearing variables (we assume in the code that we mod out the global scaling by setting $c_d=1$).
  We abbreviate moreover $(a,b,c,\alpha,\beta,\gamma,X)$ with $\overline{\textbf{z}}$ and obtain for the derivative in direction $z${\small\begin{gather*}
      \frac{\partial}{\partial z}\frac{(P_{\overline{\textbf{z}}}(v_i(\overline{\textbf{z}}))X)_2}{(P_{\overline{\textbf{z}}}(v_i(\overline{\textbf{z}}))X)_3}=\\
      \frac{(\frac{\partial}{\partial z}(P_{\overline{\textbf{z}}}(v_i(\overline{\textbf{z}}))X)_2)(P_{\overline{\textbf{z}}}(v_i(\overline{\textbf{z}}))X)_3-(P_{\overline{\textbf{z}}}(v_i(\overline{\textbf{z}}))X)_2\frac{\partial}{\partial z}((P_{\overline{\textbf{z}}}(v_i(\overline{\textbf{z}}))X)_3)}{((P_{\overline{\textbf{z}}}(v_i(\overline{\textbf{z}}))X)_3)^2}.  
  \end{gather*}}
We observe that $(P_{\overline{\textbf{z}}}(v_i(\overline{\textbf{z}})X)_3)^2$ divides the whole row of the Jacobian. Since rescaling a row with a non-zero number does not change the invertibility of a matrix, we may just omit the denominator. Next, we will try to simplify $\frac{\partial}{\partial z}(P_{\overline{\textbf{z}}}(v_i(\overline{\textbf{z}}))X)_3$ and $\frac{\partial}{\partial z}(P_{\overline{\textbf{z}}}(v_i(\overline{\textbf{z}}))X)_2$. By the definition, we can view these as polynomials in $v_i(\overline{\textbf{z}})$ with coefficients which we will call $\lambda_j(\overline{\textbf{z}}) $ for now, i.e., {\small\begin{gather*}
      \frac{\partial}{\partial z}(P_{\overline{\textbf{z}}}(v_i(\overline{\textbf{z}}))X)_2=\frac{\partial}{\partial z}\sum_{j=0}^{3d}\lambda_j(\overline{\textbf{z}})v_i(\overline{\textbf{z}})^j=\sum_{j=0}^{3d}\frac{\partial}{\partial z}(\lambda_j(\overline{\textbf{z}}))v_i(\overline{\textbf{z}})^j+\lambda_j(\overline{\textbf{z}})\frac{\partial}{\partial z}v_i(\overline{\textbf{z}})^j=\\
      \sum_{j=0}^{3d}\frac{\partial}{\partial z}(\lambda_j(\overline{\textbf{z}}))v_i(\overline{\textbf{z}})^j+\sum_{j=0}^{3d}\lambda_j(\overline{\textbf{z}})jv_i(\overline{\textbf{z}})^{j-1}\frac{\partial}{\partial z}v_i(\overline{\textbf{z}})=\\
      \frac{\partial}{\partial z}(P_{\overline{\textbf{z}}}(t)X)_2|_{t=v_i(\overline{\textbf{z}})}+\left(\sum_{j=0}^{3d}\lambda_j(\overline{\textbf{z}})jv_i(\overline{\textbf{z}})^{j-1}\right)\frac{\partial}{\partial z}v_i(\overline{\textbf{z}}).
  \end{gather*}}
  Now we observe that $\sum_{j=0}^{3d}\lambda_j(\overline{\textbf{z}})jv_i(\overline{\textbf{z}})^{j-1}$ does not depend on the choice of the variable $z$ and thus is the same for the whole row. So all the second summands together give rise to a scaled version of the gradient of $v_i$. Since $v_i$ is another one of the functions we differentiate (as seen in  \eqref{eq:MapJacobian}), we have this gradient as another row and thus we can add a multiple of it without changing the determinant and invertibility of the matrix. Hence, we can drop the second summand. Arguing similarly for $\frac{\partial}{\partial z}(P_{\overline{\textbf{z}}}(v_i(\overline{\textbf{z}}))X)_3$, we get that the Jacobian can be brought to a form where the second entries in the description of the map in  \eqref{eq:MapJacobian} give rise to entries of the form \begin{gather*}
      \frac{\partial}{\partial z}(P_{\overline{\textbf{z}}}(t)X)_2|_{t=v_i(\overline{\textbf{z}})}(P_{\overline{\textbf{z}}}(v_i(\overline{\textbf{z}}))X)_3-\frac{\partial}{\partial z}(P_{\overline{\textbf{z}}}(t)X)_3|_{t=v_i(\overline{\textbf{z}})}(P_{\overline{\textbf{z}}}(v_i(\overline{\textbf{z}}))X)_2.
  \end{gather*}
  Applying the differential operator entry-wise to a vector, we observe that we can also write $\frac{\partial}{\partial z}(P_{\overline{\textbf{z}}}(t)X)_2|_{t=v_i(\overline{\textbf{z}})}=e_2^\top\left(\frac{\partial}{\partial z}P_{\overline{\textbf{z}}}(t)X\right)|_{t=v_i(\overline{\textbf{z}})}$. Thus, the computation of these entries boils down to computing the Jacobian in the camera and world point parameters for $P_{\overline{\textbf{z}}}(t)X$ (i.e., all partial derivatives except $t$) and then evaluating at a given point and potentially multiplying it with a unit vector. We will denote this Jacobian with $\mathcal{J}(\overline{\textbf{z}},t)$

  We will now show that the gradients of the $v_i$ are up to scaling also of such a form. The Jacobain matrix of the implicit functions is also known from the implicit function theorem. Namely,\begin{gather*}
      \left[\frac{\partial v_i}{\partial z_j}\right]_{i,j}=\left[\frac{\partial F_i}{\partial x_j}\right]_{i,j}^{-1}\cdot \left[\left(\frac{\partial F_i}{\partial z_j}\right)(\overline{\textbf{z}},v_i(\overline{\textbf{z}}))\right]_{i,j}.
  \end{gather*}
  As remarked before, the Jacobian with respect to the $x_i$ is a diagonal matrix, which means that it simply rescales a row of the Jacobian we care about. Thus, we can drop this scaling without changing the rank of the Jacobian. We remind here that $F_i=(r(t)P_{\overline{\textbf{z}}}(t)X)|_{t=x_i}$. Thus, since $r(t)$ does not depend on the camera parameters or the world point, we can also write this in terms of $\mathcal{J}$, namely $\nabla v_i=([1,0,-t]\mathcal{J})|_{t=v_i(\overline{\textbf{z}})}$.

  Hence, we see that to compute the Jacobian of our map, it suffices to compute $\mathcal{J}$. Moreover, since $\mathcal{J}$ is the Jacobian of a set of multivariate polynomials, this can be done in a computer. For the sake of completeness, we will however do it by hand here.{\small\begin{gather*}
      \frac{\partial P_{\overline{\textbf{z}}}(t)X }{\partial \alpha_i}=\frac{\partial R(\alpha(t),\beta(t),\gamma(t))}{\partial \alpha_i}\begin{bmatrix}
          1 & 0 & 0 & -a(t)\\
          0& 1& 0& -b(t)\\
          0 & 0& 1 & -c(t)
      \end{bmatrix}X.
  \end{gather*}}

  \noindent
  where we can expand now the derivative of the Cayley parametrization to be \begin{gather*}
      \frac{\partial R(\alpha(t),\beta(t),\gamma(t))}{\partial \alpha_i}=\frac{\partial}{\partial \alpha_i}\left[\begin{smallmatrix}
          1 + \alpha^2 - \beta^2 - \gamma^2 &
2(\alpha \beta - \gamma) &
2(\alpha \gamma + \beta) \\
2(\alpha \beta + \gamma) &
1 - \alpha^2 + \beta^2 - \gamma^2 &
2(\beta \gamma - \alpha) \\
2(\alpha \gamma - \beta) &
2(\beta \gamma + \alpha) &
1 - \alpha^2 - \beta^2 + \gamma^2
      \end{smallmatrix}\right]\\
      =\left[\begin{smallmatrix}
          \frac{\partial}{\partial\alpha_i} \alpha^2  &
2\beta\frac{\partial}{\partial\alpha_i}\alpha   &
2\gamma\frac{\partial}{\partial\alpha_i}\alpha  \\
2\beta\frac{\partial}{\partial\alpha_i}\alpha  &
 - \frac{\partial}{\partial\alpha_i}\alpha^2 &
 - 2\frac{\partial}{\partial\alpha_i}\alpha \\
2\gamma\frac{\partial}{\partial\alpha_i}\alpha  &
2\frac{\partial}{\partial\alpha_i} \alpha &
 -\frac{\partial}{\partial\alpha_i} \alpha^2 
      \end{smallmatrix}\right]=\left[\begin{smallmatrix}
          2\alpha\frac{\partial}{\partial\alpha_i} \alpha  &
2\beta t^i  &
2\gamma t^i \\
2\beta t^i  &
 - 2\alpha\frac{\partial}{\partial\alpha_i}\alpha &
 - 2t^i \\
2\gamma t^i  &
2t^i &
 -2\alpha\frac{\partial}{\partial\alpha_i} \alpha 
      \end{smallmatrix}\right]=2t^i\left[\begin{smallmatrix}
          \alpha &
\beta  &
\gamma  \\
\beta   &
 - \alpha &
 - 1 \\
\gamma   &
1 &
 -\alpha 
      \end{smallmatrix}\right].
  \end{gather*}
  Similarly, we obtain{\small\begin{gather*}
      \frac{\partial P_{\overline{\textbf{z}}}(t)X }{\partial \beta_i}=2t^i\begin{bmatrix}
                 -\beta(t) &
\alpha(t)  &
1  \\
\alpha(t)   &
 \beta(t) &
 \gamma(t) \\
-1   &
\gamma(t) &
 -\beta(t)
      \end{bmatrix}\begin{bmatrix}
          1 & 0 & 0 & -a(t)\\
          0& 1& 0& -b(t)\\
          0 & 0& 1 & -c(t)
      \end{bmatrix}X\text{ and}\\  
   \frac{\partial P_{\overline{\textbf{z}}}(t)X }{\partial \gamma_i}=   2t^i\begin{bmatrix}
                 -\gamma(t) &
-1  &
\alpha(t)  \\
1   &
 -\gamma(t) &
 \beta(t) \\
\alpha(t)   &
\beta(t) &
 \gamma(t)
      \end{bmatrix}\begin{bmatrix}
          1 & 0 & 0 & -a(t)\\
          0& 1& 0& -b(t)\\
          0 & 0& 1 & -c(t)
      \end{bmatrix}X.
  \end{gather*}}
  For the derivatives w.r.t. the translation parameters, we can argue analogously and obtain{\small\begin{gather*}
      \frac{\partial P_{\overline{\textbf{z}}}(t)X }{\partial a_i}= R(\alpha(t),\beta(t),\gamma(t))\frac{\partial}{\partial a_i}\begin{bmatrix}
          1 & 0 & 0 & -a(t)\\
          0& 1& 0& -b(t)\\
          0 & 0& 1 & -c(t)
      \end{bmatrix}X\\
      =R(\alpha(t),\beta(t),\gamma(t))\begin{bmatrix}
          0 & 0 & 0 & -t^i\\
          0& 0& 0& 0\\
          0 & 0& 0 & 0
      \end{bmatrix}X=R(\alpha(t),\beta(t),\gamma(t))\begin{bmatrix}
          -t^i\\
          0\\
          0
      \end{bmatrix},\\
      \frac{\partial P_{\overline{\textbf{z}}}(t)X }{\partial b_i}=R(\alpha(t),\beta(t),\gamma(t))\begin{bmatrix}
          0\\
          -t^i\\
          0
      \end{bmatrix}\text{ and}\\
      \frac{\partial P_{\overline{\textbf{z}}}(t)X }{\partial c_i}=R(\alpha(t),\beta(t),\gamma(t))\begin{bmatrix}
          0\\
          0\\
          -t^i
      \end{bmatrix}.
  \end{gather*}}
  Lastly, we look at the derivatives in the direction of the variables of the point \begin{gather*}
      \frac{\partial P_{\overline{\textbf{z}}}(t)X }{\partial X_i}=P_{\overline{\textbf{z}}}(t)\frac{\partial X}{\partial X_i}=P_{\overline{\textbf{z}}}(t)e_i.
  \end{gather*}
  This finishes the computation of $\mathcal{J}$. An implementation of $\mathcal{J}$ in this form can be found in the submitted code. As discussed before, it is then used to compute the Jacobian of the map we are interested in. Then, we check that that Jacobian has full rank at some point, giving an open ball in the image by the inverse function theorem. This finishes the discussion of the minimality of the balanced problems with $p=1$, which was tested for $d\leq 5$. The degrees were computed using the equation system given in \cite{hruby2025single} via homotopy continuation.
  \qed
\end{proof}

\begin{proof}[of Theorem \ref{thm:point_minimal_p2}]
    We start by computing the balanced problems. Since for $d>1$ we have that $1+d+2\delta>2$, the number of constraints we get is always $4p$ and independent of the case distinction in \eqref{eq:ptConstraintsNumber}. So we only have to distinguish the cases $d=0$ and $d>0$ for the number of DoF in the domain. Let us start by considering the case $d>0$: In this case, we have to satisfy the equation $3(d+\delta)-1+3p=4p$, which is equivalent to  $d+\delta=\frac{p+1}{3}$. Thus, all such problems with $p+1$ divisible by $3$ will be balanced. Considering the case $d=0$, we get the equation $3\delta+2p=4p$, so we obtain balanced problems when $\delta$ is even and $p=\frac{3\delta}{2}$.
    
    We continue by showing that the problems of the form $\delta=0,d>0$ and $p=3d-1$ are minimal of degree $1$. We observe that the second case in Theorem \ref{thm:point_minimal_problems} is a special case of this statement and we will thus argue similarly. Using the same argument as in the proof of Theorem \ref{thm:point_minimal_problems}, we may assume that generically all solutions to the reconstruction problem are finite points, i.e., the last coordinates of our points $X^i$ are $1$. Each world point $X^i$ gives rise to $4$ constraints:{\small\begin{gather*}
X_1^i-a(\tilde x_{i,1})-\tilde x_{i,1}X^i_3+\tilde x_{i,1}c(\tilde x_{i,1})=0, \,  \tilde y_{i,1}\cdot (X_3^i-c(\tilde x_{i,1}))-X_2^i+b(\tilde x_{i,1})=0,
\\
X_1^i-a(\tilde x_{i,2})-\tilde x_{i,2}X^i_3+\tilde x_{i,2}c(\tilde x_{i,2})=0, \,  \tilde y_{i,2}\cdot (X_3^i-c(\tilde x_{i,2}))-X_2^i+b(\tilde x_{i,2})=0.
  \end{gather*}}
  Collecting all these in one matrix as in the proof of Theorem \ref{thm:point_minimal_problems} reduces the problem to showing that the resulting $4p\times(4p+1)$ matrix has generically full rank. We do this by considering the submatrix that does not contain the column corresponding to $c_d$. Since the determinant is a polynomial in the $\tilde x_{i,j}$ and $\tilde y_{i,j}$, it will be non-zero for generic values if we can find one such tuple of values where it is non-zero. Thus, we start by setting $\tilde x_{i,2}=0$ for each $i$ such that the resulting matrix becomes\begin{gather*}
      \left[\begin{smallmatrix}
          -\tilde x_{1,1} &\dots & -\tilde x_{1,1}^d& 0& \dots & 0 & \tilde x_{1,1}^2& \dots & \tilde x_{1,1}^{d}& 1 & 0 & -\tilde x_{1,1} & 0 & 0 & 0 & 0& \dots & 0 \\
          0 &\dots & 0& \tilde x_{1,1}& \dots & \tilde x_{1,1}^d & -\tilde y_{1,1} \tilde x_{1,1}& \dots & -\tilde y_{1,1} \tilde x_{1,1}^{d-1} & 0 & -1 & \tilde y_{1,1}& 0 & 0 & 0 & 0& \dots & 0\\
        0 & \dots & 0 &0&\dots & 0& 0&\dots & 0& 1 & 0 &0& 0 & 0 & 0 & 0&\dots & 0\\
        0 & \dots & 0& 0&\dots& 0& 0& \dots &0&0&-1&\tilde y_{1,2}& 0 & 0 & 0 & 0&\dots & 0\\
        -\tilde x_{2,1} &\dots & -\tilde x_{2,1}^d& 0& \dots & 0 & \tilde x_{2,1}^2& \dots & \tilde x_{2,1}^{d}& 0 & 0 & 0& 1 & 0 & -\tilde x_{2,1} & 0 & \dots & 0 \\
          0 &\dots & 0& \tilde x_{2,1}& \dots & \tilde x_{2,1}^d & -\tilde y_{2,1} \tilde x_{2,1}& \dots & -\tilde y_{2,1} \tilde x_{2,1}^{d-1} & 0 & 0 & 0& 0 & -1 & \tilde y_{2,1}& 0 & \dots & 0\\
        0 & \dots & 0 &0&\dots & 0& 0&\dots & 0 & 0 & 0& 0& 1 & 0 &0& 0 &\dots & 0\\
        0 & \dots & 0& 0&\dots& 0& 0& \dots & 0 & 0 & 0&0&0&-1&\tilde y_{2,2}& 0 &\dots & 0\\
        & \vdots & & &\vdots& & & \vdots & & & & \vdots & &  &  & & \ddots & 
      \end{smallmatrix}\right]
  \end{gather*}
  where the vertical and diagonal dots are supposed to represent that we continue the pattern of the blocks consisting of $4$ rows of the matrix. We observe that the third row in each block of $4$ rows contains exactly $1$ non-zero entry and thus we can erase the corresponding row and column by Laplace expansion without doing harm to the invertibility of the matrix. Furthermore, we can subtract the last row from each block of $4$ rows from the second to ensure that the column corresponding to $X^i_2$ has only one non-zero entry left, namely the $-1$ in the last row. Thus, we can Laplace expand that column without changing the invertibility. This results in the following matrix:\begin{gather*}
      \left[\begin{smallmatrix}
          -\tilde x_{1,1} &\dots & -\tilde x_{1,1}^d& 0& \dots & 0 & \tilde x_{1,1}^2& \dots & \tilde x_{1,1}^{d} & -\tilde x_{1,1} & 0 & \dots \\
          0 &\dots & 0& \tilde x_{1,1}& \dots & \tilde x_{1,1}^d & -\tilde y_{1,1} \tilde x_{1,1}& \dots & -\tilde y_{1,1} \tilde x_{1,1}^{d-1} &  \tilde y_{1,1}-\tilde y_{1,2}& 0 & \dots \\
 \tilde x_{2,1} &\dots & -\tilde x_{2,1}^d& 0& \dots & 0 & \tilde x_{2,1}^2& \dots & \tilde x_{2,1}^{d} &0& -\tilde x_{2,1}  &0 \\ 
 0 &\dots & 0& \tilde x_{2,1}& \dots & \tilde x_{2,1}^d & -\tilde y_{2,1} \tilde x_{2,1}& \dots & -\tilde y_{2,1} \tilde x_{2,1}^{d-1} & 0& \tilde y_{2,1}-\tilde y_{2,2}& 0  \\
 & \vdots & & & \vdots & & & \vdots& & \vdots & \ddots & \ddots 
      \end{smallmatrix}\right].
  \end{gather*}
  Now we have blocks consisting of $2$ rows. By our genericity assumption, we have $\tilde x_{i,1}\neq 0$. Thus, we can add  $\frac{\tilde y_{i,1}-\tilde y_{i,2}}{\tilde x_{i,1}}$ times the first row of a block to the second row. Afterwards the columns corresponding to $X_3^i$ contain only one non-zero entry, namely $-\tilde x_{i,1}$. Hence, we can Laplace expand w.r.t. these columns and the determinant will be non-zero iff it was non-zero before. Therefore, we are left with showing that the following matrix is generically invertible:\begin{gather*}
      \left[\begin{smallmatrix}
          (\tilde y_{1,2}-\tilde y_{1,1}) \tilde x_{1,1}^0&\dots & (\tilde y_{1,2}-\tilde y_{1,1}) \tilde x_{1,1}^{d-1}& \tilde x_{1,1}& \dots &\tilde x_{1,1}^d& -\tilde x_{1,1}\tilde y_{1,2}& \dots & -\tilde x_{1,1}^{d-1}\tilde y_{1,2} \\
           (\tilde y_{2,2}-\tilde y_{2,1}) \tilde x_{2,1}^0&\dots & (\tilde y_{2,2}-\tilde y_{2,1}) \tilde x_{2,1}^{d-1}& \tilde x_{2,1}& \dots &\tilde x_{2,1}^d& -\tilde x_{2,1}\tilde y_{2,2}& \dots & -\tilde x_{2,1}^{d-1}\tilde y_{2,2} \\
           & \vdots & && \vdots&&&\vdots&
      \end{smallmatrix}\right].
  \end{gather*}
  Since the determinant of the above matrix is again a polynomial in the variables $\tilde x_{i,j},\tilde y_{i,j}$, it suffices again to find one point where it is invertible; then, it will hold generically and in particular at some point where all $\tilde x_{i,1}\neq 0$. The existence of such a point is equivalent to saying that the {multivariate} functions $y_2-y_1,\dots,(y_2-y_1)x^{d-1},x, \dots, x^d,$ $-xy_2,\dots,-x^{d-1}y_2 $ are linearly independent. Since we are interested in this statement over an infinite field, it is equivalent to saying that these are linearly independent as {multivariate} polynomials. After adding $-x^iy_2$ to $(y_2-y_1)x^{i}$ and eliminating signs, we obtain the 
  distinct monomials $y_1x^i,y_2x^i$ for $1\leq i\leq d-1$ and $x^i$ for $1\leq i\leq d$. As they are distinct monomials, they are linearly independent. Lastly, we have the binomial $y_2-y_1$, which does not contain $x$ and is thus linearly independent from the 
  monomials studied before. Thus, all these are linearly independent and the matrix is generically invertible.

  The minimality of the remaining problems listed in Theorem \ref{thm:point_minimal_p2} were again tested numerically/symbolically  by computing the solution set over a randomly chosen point. For this we used again the equations discussed in \cite{hruby2025single}.\qed
  \end{proof}

\section{Image Curve Constraint for $d=0,\delta=1$}
\label{sec:longEquation}
    The following equation cuts out the Zariski closure of the set of image curves of world lines observed by RS cameras in $\mathcal{P}_{0,1}$, using the notation from Example~\ref{ex:curve_d0del1}, in the affine chart where $Z_0=-1$:

    \noindent 
    \tiny
    {\tiny $4Z_3^2Z_2^4Z_1^4 +4Z_2^6Z_1^4- 8Z_3Z_2^4Z_1^5 +4Z_2^4Z_1^6 +8Z_3^2Z_2^3Z_1^3Y_2Y_1 +16Z_2^5Z_1^3Y_2Y_1 -16Z_3Z_2^3Z_1^4Y_2Y_1 +8Z_2^3Z_1^5Y_2Y_1 +4Z_3^2Z_2^4Z_1^2Y_1^2 -4Z_3^3Z_2^2Z_1^3Y_1^2 -12Z_3Z_2^4Z_1^3Y_1^2 +8Z_3^2Z_2^2Z_1^4Y_1^2 +4Z_2^4Z_1^4Y_1^2 -4Z_3Z_2^2Z_1^5Y_1^2 +4Z_3^2Z_2^2Z_1^2Y_2^2Y_1^2 +24Z_2^4Z_1^2Y_2^2Y_1^2 -8Z_3Z_2^2Z_1^3Y_2^2Y_1^2 +4Z_2^2Z_1^4Y_2^2Y_1^2 +8Z_3^2Z_2^3Z_1Y_2Y_1^3 -4Z_3^3Z_2Z_1^2Y_2Y_1^3 -28Z_3Z_2^3Z_1^2Y_2Y_1^3+\\8Z_3^2Z_2Z_1^3Y_2Y_1^3+8Z_2^3Z_1^3Y_2Y_1^3-4Z_3Z_2Z_1^4Y_2Y_1^3+16Z_2^3Z_1Y_2^3Y_1^3-4Z_3^3Z_2^2Z_1Y_1^4+7Z_3^2Z_2^2Z_1^2Y_1^4-Z_2^4Z_1^2Y_1^4-2Z_3Z_2^2Z_1^3Y_1^4-Z_2^2Z_1^4Y_1^4+4Z_3^2Z_2^2Y_2^2Y_1^4-20Z_3Z_2^2Z_1Y_2^2Y_1^4-Z_3^2Z_1^2Y_2^2Y_1^4+3Z_2^2Z_1^2Y_2^2Y_1^4+2Z_3Z_1^3Y_2^2Y_1^4-Z_1^4Y_2^2Y_1^4+4Z_2^2Y_2^4Y_1^4-4Z_3^3Z_2Y_2Y_1^5+8Z_3^2Z_2Z_1Y_2Y_1^5-2Z_2^3Z_1Y_2Y_1^5-4Z_3Z_2Z_1^2Y_2Y_1^5-4Z_3Z_2Y_2^3Y_1^5-2Z_2Z_1Y_2^3Y_1^5-Z_3^2Z_2^2Y_1^6+2Z_3Z_2^2Z_1Y_1^6-Z_2^2Z_1^2Y_1^6-Z_3^2Y_2^2Y_1^6-Z_2^2Y_2^2Y_1^6+2Z_3Z_1Y_2^2Y_1^6-Z_1^2Y_2^2Y_1^6-Y_2^4Y_1^6+16Z_3^2Z_2^4Z_1^2Y_2Y_0-32Z_3Z_2^4Z_1^3Y_2Y_0+16Z_2^4Z_1^4Y_2Y_0+16Z_3^4Z_2^3Z_1Y_1Y_0-64Z_3^3Z_2^3Z_1^2Y_1Y_0-16Z_3Z_2^5Z_1^2Y_1Y_0+96Z_3^2Z_2^3Z_1^3Y_1Y_0+8Z_2^5Z_1^3Y_1Y_0-64Z_3Z_2^3Z_1^4Y_1Y_0+16Z_2^3Z_1^5Y_1Y_0+32Z_3^2Z_2^3Z_1Y_2^2Y_1Y_0-64Z_3Z_2^3Z_1^2Y_2^2Y_1Y_0+32Z_2^3Z_1^3Y_2^2Y_1Y_0+\\16Z_3^4Z_2^2Y_2Y_1^2Y_0-72Z_3^3Z_2^2Z_1Y_2Y_1^2Y_0-32Z_3Z_2^4Z_1Y_2Y_1^2Y_0+100Z_3^2Z_2^2Z_1^2Y_2Y_1^2Y_0+4Z_2^4Z_1^2Y_2Y_1^2Y_0\\-48Z_3Z_2^2Z_1^3Y_2Y_1^2Y_0+4Z_2^2Z_1^4Y_2Y_1^2Y_0+16Z_3^2Z_2^2Y_2^3Y_1^2Y_0-32Z_3Z_2^2Z_1Y_2^3Y_1^2Y_0+16Z_2^2Z_1^2Y_2^3Y_1^2Y_0-16Z_3^5Z_2Y_1^3Y_0-8Z_3^3Z_2^3Y_1^3Y_0+60Z_3^4Z_2Z_1Y_1^3Y_0+36Z_3^2Z_2^3Z_1Y_1^3Y_0-80Z_3^3Z_2Z_1^2Y_1^3Y_0-24Z_3Z_2^3Z_1^2Y_1^3Y_0+44Z_3^2Z_2Z_1^3Y_1^3Y_0+8Z_2^3Z_1^3Y_1^3Y_0-8Z_3Z_2Z_1^4Y_1^3Y_0-16Z_3^3Z_2Y_2^2Y_1^3Y_0-16Z_3Z_2^3Y_2^2Y_1^3Y_0+32Z_3^2Z_2Z_1Y_2^2Y_1^3Y_0-16Z_2^3Z_1Y_2^2Y_1^3Y_0-16Z_3Z_2Z_1^2Y_2^2Y_1^3Y_0-4Z_3^4Y_2Y_1^4Y_0+4Z_3^2Z_2^2Y_2Y_1^4Y_0+14Z_3^3Z_1Y_2Y_1^4Y_0+30Z_3Z_2^2Z_1Y_2Y_1^4Y_0-16Z_3^2Z_1^2Y_2Y_1^4Y_0-8Z_2^2Z_1^2Y_2Y_1^4Y_0+6Z_3Z_1^3Y_2Y_1^4Y_0-4Z_3^2Y_2^3Y_1^4Y_0-12Z_2^2Y_2^3Y_1^4Y_0+8Z_3Z_1Y_2^3Y_1^4Y_0-4Z_1^2Y_2^3Y_1^4Y_0+6Z_3^3Z_2Y_1^5Y_0+2Z_3Z_2^3Y_1^5Y_0-
14Z_3^2Z_2Z_1Y_1^5Y_0+10Z_3Z_2Z_1^2Y_1^5Y_0-2Z_2Z_1^3Y_1^5Y_0+16Z_3Z_2Y_2^2Y_1^5Y_0+2Z_2Z_1Y_2^2Y_1^5Y_0+2Z_3^2Y_2Y_1^6Y_0+2Z_2^2Y_2Y_1^6Y_0-4Z_3Z_1Y_2Y_1^6Y_0+2Z_1^2Y_2Y_1^6Y_0+4Y_2^3Y_1^6Y_0-16Z_3^4Z_2^4Y_0^2+32Z_3^3Z_2^4Z_1Y_0^2-16Z_3^2Z_2^4Z_1^2Y_0^2-\\64Z_3^3Z_2^3Y_2Y_1Y_0^2+160Z_3^2Z_2^3Z_1Y_2Y_1Y_0^2-
128Z_3Z_2^3Z_1^2Y_2Y_1Y_0^2+32Z_2^3Z_1^3Y_2Y_1Y_0^2+56Z_3^4Z_2^2Y_1^2Y_0^2+8Z_3^2Z_2^4Y_1^2Y_0^2-176Z_3^3Z_2^2Z_1Y_1^2Y_0^2-8Z_3Z_2^4Z_1Y_1^2Y_0^2+208Z_3^2Z_2^2Z_1^2Y_1^2Y_0^2+4Z_2^4Z_1^2Y_1^2Y_0^2-112Z_3Z_2^2Z_1^3Y_1^2Y_0^2+24Z_2^2Z_1^4Y_1^2Y_0^2-32Z_3^2Z_2^2Y_2^2Y_1^2Y_0^2+64Z_3Z_2^2Z_1Y_2^2Y_1^2Y_0^2-32Z_2^2Z_1^2Y_2^2Y_1^2Y_0^2+48Z_3^3Z_2Y_2Y_1^3Y_0^2+
16Z_3Z_2^3Y_2Y_1^3Y_0^2-\\112Z_3^2Z_2Z_1Y_2Y_1^3Y_0^2+80Z_3Z_2Z_1^2Y_2Y_1^3Y_0^2-16Z_2Z_1^3Y_2Y_1^3Y_0^2+3Z_3^4Y_1^4Y_0^2-10Z_3^2Z_2^2Y_1^4Y_0^2-Z_2^4Y_1^4Y_0^2-10Z_3^3Z_1Y_1^4Y_0^2-6Z_3Z_2^2Z_1Y_1^4Y_0^2+11Z_3^2Z_1^2Y_1^4Y_0^2+3Z_2^2Z_1^2Y_1^4Y_0^2-4Z_3Z_1^3Y_1^4Y_0^2+12Z_3^2Y_2^2Y_1^4Y_0^2+12Z_2^2Y_2^2Y_1^4Y_0^2-24Z_3Z_1Y_2^2Y_1^4Y_0^2+12Z_1^2Y_2^2Y_1^4Y_0^2-
20Z_3Z_2Y_2Y_1^5Y_0^2+2Z_2Z_1Y_2Y_1^5Y_0^2-Z_3^2Y_1^6Y_0^2-Z_2^2Y_1^6Y_0^2+2Z_3Z_1Y_1^6Y_0^2-Z_1^2Y_1^6Y_0^2-6Y_2^2Y_1^6Y_0^2+
16Z_3^2Z_2^2Y_2Y_1^2Y_0^3-32Z_3Z_2^2Z_1Y_2Y_1^2Y_0^3+16Z_2^2Z_1^2Y_2Y_1^2Y_0^3-32Z_3^3Z_2Y_1^3Y_0^3+80Z_3^2Z_2Z_1Y_1^3Y_0^3-64Z_3Z_2Z_1^2Y_1^3Y_0^3+16Z_2Z_1^3Y_1^3Y_0^3-12Z_3^2Y_2Y_1^4Y_0^3-4Z_2^2Y_2Y_1^4Y_0^3+24Z_3Z_1Y_2Y_1^4Y_0^3-12Z_1^2Y_2Y_1^4Y_0^3+8Z_3Z_2Y_1^5Y_0^3-2Z_2Z_1Y_1^5Y_0^3+4Y_2Y_1^6Y_0^3+4Z_3^2Y_1^4Y_0^4-8Z_3Z_1Y_1^4Y_0^4+4Z_1^2Y_1^4Y_0^4-Y_1^6Y_0^4-16Z_3^4Z_2^3Z_1^2-
16Z_3^2Z_2^5Z_1^2+64Z_3^3Z_2^3Z_1^3+32Z_3Z_2^5Z_1^3-80Z_3^2Z_2^3Z_1^4+8Z_2^5Z_1^4+32Z_3Z_2^3Z_1^5-16Z_3^4Z_2^2Z_1Y_2Y_1-32Z_3^2Z_2^4Z_1Y_2Y_1+56Z_3^3Z_2^2Z_1^2Y_2Y_1+48Z_3Z_2^4Z_1^2Y_2Y_1-48Z_3^2Z_2^2Z_1^3Y_2Y_1+64Z_2^4Z_1^3Y_2Y_1-8Z_3Z_2^2Z_1^4Y_2Y_1+16Z_2^2Z_1^5Y_2Y_1+16Z_3^5Z_2Z_1Y_1^2+24Z_3^3Z_2^3Z_1Y_1^2-56Z_3^4Z_2Z_1^2Y_1^2-28Z_3^2Z_2^3Z_1^2Y_1^2+4Z_2^5Z_1^2Y_1^2+64Z_3^3Z_2Z_1^3Y_1^2-24Z_3Z_2^3Z_1^3Y_1^2-24Z_3^2Z_2Z_1^4Y_1^2+12Z_2^3Z_1^4Y_1^2-16Z_3^2Z_2^3Y_2^2Y_1^2+8Z_3^2Z_2Z_1^2Y_2^2Y_1^2+
112Z_2^3Z_1^2Y_2^2Y_1^2-16Z_3Z_2Z_1^3Y_2^2Y_1^2+8Z_2Z_1^4Y_2^2Y_1^2+16Z_3^3Z_2^2Y_2Y_1^3+4Z_3^4Z_1Y_2Y_1^3+\\
20Z_3^2Z_2^2Z_1Y_2Y_1^3+8Z_2^4Z_1Y_2Y_1^3-16Z_3^3Z_1^2Y_2Y_1^3-96Z_3Z_2^2Z_1^2Y_2Y_1^3+20Z_3^2Z_1^3Y_2Y_1^3+
24Z_2^2Z_1^3Y_2Y_1^3-8Z_3Z_1^4Y_2Y_1^3-16Z_3Z_2^2Y_2^3Y_1^3+64Z_2^2Z_1Y_2^3Y_1^3+4Z_3^4Z_2Y_1^4+8Z_3^2Z_2^3Y_1^4-22Z_3^3Z_2Z_1Y_1^4-14Z_3Z_2^3Z_1Y_1^4+32Z_3^2Z_2Z_1^2Y_1^4+2Z_2^3Z_1^2Y_1^4-14Z_3Z_2Z_1^3Y_1^4+32Z_3^2Z_2Y_2^2Y_1^4+4Z_2^3Y_2^2Y_1^4-64Z_3Z_2Z_1Y_2^2Y_1^4+6Z_2Z_1^2Y_2^2Y_1^4+8Z_2Y_2^4Y_1^4-2Z_3^3Y_2Y_1^5+4Z_3Z_2^2Y_2Y_1^5+2Z_3^2Z_1Y_2Y_1^5-10Z_2^2Z_1Y_2Y_1^5+2Z_3Z_1^2Y_2Y_1^5-2Z_1^3Y_2Y_1^5+2Z_3Y_2^3Y_1^5-8Z_1Y_2^3Y_1^5-2Z_3^2Z_2Y_1^6+4Z_3Z_2Z_1Y_1^6-2Z_2Z_1^2Y_1^6-2Z_2Y_2^2Y_1^6-32Z_3^4Z_2^3Y_2Y_0+128Z_3^3Z_2^3Z_1Y_2Y_0-128Z_3^2Z_2^3Z_1^2Y_2Y_0+32Z_2^3Z_1^4Y_2Y_0+48Z_3^5Z_2^2Y_1Y_0+32Z_3^3Z_2^4Y_1Y_0-192Z_3^4Z_2^2Z_1Y_1Y_0-80Z_3^2Z_2^4Z_1Y_1Y_0+
288Z_3^3Z_2^2Z_1^2Y_1Y_0-8Z_3Z_2^4Z_1^2Y_1Y_0-192Z_3^2Z_2^2Z_1^3Y_1Y_0+16Z_2^4Z_1^3Y_1Y_0+48Z_3Z_2^2Z_1^4Y_1Y_0-32Z_3^3Z_2^2Y_2^2Y_1Y_0+192Z_3^2Z_2^2Z_1Y_2^2Y_1Y_0-288Z_3Z_2^2Z_1^2Y_2^2Y_1Y_0+128Z_2^2Z_1^3Y_2^2Y_1Y_0+\\
88Z_3^4Z_2Y_2Y_1^2Y_0+64Z_3^2Z_2^3Y_2Y_1^2Y_0-336Z_3^3Z_2Z_1Y_2Y_1^2Y_0-224Z_3Z_2^3Z_1Y_2Y_1^2Y_0+440Z_3^2Z_2Z_1^2Y_2Y_1^2Y_0+\\
48Z_2^3Z_1^2Y_2Y_1^2Y_0-224Z_3Z_2Z_1^3Y_2Y_1^2Y_0+32Z_2Z_1^4Y_2Y_1^2Y_0+32Z_3^2Z_2Y_2^3Y_1^2Y_0-64Z_3Z_2Z_1Y_2^3Y_1^2Y_0+\\
32Z_2Z_1^2Y_2^3Y_1^2Y_0-12Z_3^5Y_1^3Y_0-68Z_3^3Z_2^2Y_1^3Y_0-
8Z_3Z_2^4Y_1^3Y_0+40Z_3^4Z_1Y_1^3Y_0+172Z_3^2Z_2^2Z_1Y_1^3Y_0+4Z_2^4Z_1Y_1^3Y_0-44Z_3^3Z_1^2Y_1^3Y_0-
96Z_3Z_2^2Z_1^2Y_1^3Y_0+16Z_3^2Z_1^3Y_1^3Y_0+28Z_2^2Z_1^3Y_1^3Y_0-16Z_3Z_2^2Y_2^2Y_1^3Y_0-16Z_3^2Z_1Y_2^2Y_1^3Y_0-\\
80Z_2^2Z_1Y_2^2Y_1^3Y_0+32Z_3Z_1^2Y_2^2Y_1^3Y_0-16Z_1^3Y_2^2Y_1^3Y_0-54Z_3^2Z_2Y_2Y_1^4Y_0-6Z_2^3Y_2Y_1^4Y_0+136Z_3Z_2Z_1Y_2Y_1^4Y_0-30Z_2Z_1^2Y_2Y_1^4Y_0-24Z_2Y_2^3Y_1^4Y_0+4Z_3^3Y_1^5Y_0+2Z_3Z_2^2Y_1^5Y_0-8Z_3^2Z_1Y_1^5Y_0+4Z_2^2Z_1Y_1^5Y_0+4Z_3Z_1^2Y_1^5Y_0-\\
2Z_3Y_2^2Y_1^5Y_0+20Z_1Y_2^2Y_1^5Y_0+4Z_2Y_2Y_1^6Y_0-64Z_3^4Z_2^3Y_0^2+128Z_3^3Z_2^3Z_1Y_0^2-64Z_3^2Z_2^3Z_1^2Y_0^2-160Z_3^3Z_2^2Y_2Y_1Y_0^2+384Z_3^2Z_2^2Z_1Y_2Y_1Y_0^2-288Z_3Z_2^2Z_1^2Y_2Y_1Y_0^2+64Z_2^2Z_1^3Y_2Y_1Y_0^2+32Z_3^4Z_2Y_1^2Y_0^2+24Z_3^2Z_2^3Y_1^2Y_0^2-64Z_3^3Z_2Z_1Y_1^2Y_0^2-16Z_3Z_2^3Z_1Y_1^2Y_0^2+32Z_3^2Z_2Z_1^2Y_1^2Y_0^2+8Z_2^3Z_1^2Y_1^2Y_0^2-64Z_3^2Z_2Y_2^2Y_1^2Y_0^2-64Z_2Z_1^2Y_2^2Y_1^2Y_0^2+128Z_3Z_2Z_1Y_2^2Y_1^2Y_0^2+16Z_3^3Y_2Y_1^3Y_0^2+32Z_3Z_2^2Y_2Y_1^3Y_0^2-16Z_3^2Z_1Y_2Y_1^3Y_0^2+16Z_1^3Y_2Y_1^3Y_0^2+16Z_2^2Z_1Y_2Y_1^3Y_0^2-16Z_3Z_1^2Y_2Y_1^3Y_0^2+18Z_3^2Z_2Y_1^4Y_0^2-2Z_2^3Y_1^4Y_0^2-64Z_3Z_2Z_1Y_1^4Y_0^2+20Z_2Z_1^2Y_1^4Y_0^2+24Z_2Y_2^2Y_1^4Y_0^2-2Z_3Y_2Y_1^5Y_0^2-16Z_1Y_2Y_1^5Y_0^2-2Z_2Y_1^6Y_0^2+32Z_3^2Z_2Y_2Y_1^2Y_0^3-64Z_3Z_2Z_1Y_2Y_1^2Y_0^3+32Z_2Z_1^2Y_2Y_1^2Y_0^3-16Z_3^3Y_1^3Y_0^3+32Z_3^2Z_1Y_1^3Y_0^3-16Z_3Z_1^2Y_1^3Y_0^3-8Z_2Y_2Y_1^4Y_0^3+2Z_3Y_1^5Y_0^3+4Z_1Y_1^5Y_0^3-32Z_3^5Z_2^2Z_1-32Z_3^3Z_2^4Z_1+104Z_3^4Z_2^2Z_1^2+8Z_3^2Z_2^4Z_1^2-96Z_3^3Z_2^2Z_1^3+80Z_3Z_2^4Z_1^3+8Z_3^2Z_2^2Z_1^4+4Z_2^4Z_1^4+16Z_3Z_2^2Z_1^5-16Z_3^4Z_2^2Y_2^2+64Z_3^3Z_2^2Z_1Y_2^2-96Z_3^2Z_2^2Z_1^2Y_2^2+64Z_3Z_2^2Z_1^3Y_2^2-16Z_2^2Z_1^4Y_2^2+16Z_3^5Z_2Y_2Y_1-104Z_3^4Z_2Z_1Y_2Y_1-128Z_3^2Z_2^3Z_1Y_2Y_1+240Z_3^3Z_2Z_1^2Y_2Y_1+192Z_3Z_2^3Z_1^2Y_2Y_1-\\
232Z_3^2Z_2Z_1^3Y_2Y_1+96Z_2^3Z_1^3Y_2Y_1+80Z_3Z_2Z_1^4Y_2Y_1-12Z_3^4Z_2^2Y_1^2-16Z_3^2Z_2^4Y_1^2+12Z_3^5Z_1Y_1^2+108Z_3^3Z_2^2Z_1Y_1^2+24Z_3Z_2^4Z_1Y_1^2-40Z_3^4Z_1^2Y_1^2-152Z_3^2Z_2^2Z_1^2Y_1^2+8Z_2^4Z_1^2Y_1^2+44Z_3^3Z_1^3Y_1^2+28Z_3Z_2^2Z_1^3Y_1^2-16Z_3^2Z_1^4Y_1^2+4Z_2^2Z_1^4Y_1^2+4Z_3^4Y_2^2Y_1^2-32Z_3^2Z_2^2Y_2^2Y_1^2-16Z_3^3Z_1Y_2^2Y_1^2-32Z_3Z_2^2Z_1Y_2^2Y_1^2+28Z_3^2Z_1^2Y_2^2Y_1^2+208Z_2^2Z_1^2Y_2^2Y_1^2-24Z_3Z_1^3Y_2^2Y_1^2+8Z_1^4Y_2^2Y_1^2-8Z_3^3Z_2Y_2Y_1^3-16Z_3Z_2^3Y_2Y_1^3+44Z_2^3Z_1Y_2Y_1^3+108Z_3^2Z_2Z_1Y_2Y_1^3-172Z_3Z_2Z_1^2Y_2Y_1^3+36Z_2Z_1^3Y_2Y_1^3-32Z_3Z_2Y_2^3Y_1^3+80Z_2Z_1Y_2^3Y_1^3+3Z_3^4Y_1^4+19Z_3^2Z_2^2Y_1^4-14Z_3^3Z_1Y_1^4-32Z_3Z_2^2Z_1Y_1^4+19Z_3^2Z_1^2Y_1^4+7Z_2^2Z_1^2Y_1^4-8Z_3Z_1^3Y_1^4+15Z_3^2Y_2^2Y_1^4+11Z_2^2Y_2^2Y_1^4-18Z_3Z_1Y_2^2Y_1^4-10Z_1^2Y_2^2Y_1^4+4Y_2^4Y_1^4+8Z_3Z_2Y_2Y_1^5-Z_3^2Y_1^6-14Z_2Z_1Y_2Y_1^5+2Z_3Z_1Y_1^6-Z_1^2Y_1^6-80Z_3^4Z_2^2Y_2Y_0+320Z_3^3Z_2^2Z_1Y_2Y_0-384Z_3^2Z_2^2Z_1^2Y_2Y_0-Y_2^2Y_1^6+128Z_3Z_2^2Z_1^3Y_2Y_0+16Z_2^2Z_1^4Y_2Y_0+32Z_3^5Z_2Y_1Y_0+96Z_3^3Z_2^3Y_1Y_0-128Z_3^4Z_2Z_1Y_1Y_0-232Z_3^2Z_2^3Z_1Y_1Y_0+192Z_3^3Z_2Z_1^2Y_1Y_0+48Z_3Z_2^3Z_1^2Y_1Y_0-128Z_3^2Z_2Z_1^3Y_1Y_0+8Z_2^3Z_1^3Y_1Y_0+32Z_3Z_2Z_1^4Y_1Y_0-64Z_3^3Z_2Y_2^2Y_1Y_0+288Z_3^2Z_2Z_1Y_2^2Y_1Y_0-384Z_3Z_2Z_1^2Y_2^2Y_1Y_0+\\
160Z_2Z_1^3Y_2^2Y_1Y_0+44Z_3^4Y_2Y_1^2Y_0+172Z_3^2Z_2^2Y_2Y_1^2Y_0-152Z_3^3Z_1Y_2Y_1^2Y_0+16Z_1^2Y_2^3Y_1^2Y_0-440Z_3Z_2^2Z_1Y_2Y_1^2Y_0+172Z_3^2Z_1^2Y_2Y_1^2Y_0+100Z_2^2Z_1^2Y_2Y_1^2Y_0-64Z_3Z_1^3Y_2Y_1^2Y_0+16Z_3^2Y_2^3Y_1^2Y_0-32Z_3Z_1Y_2^3Y_1^2Y_0-60Z_3^3Z_2Y_1^3Y_0-20Z_3Z_2^3Y_1^3Y_0+108Z_3^2Z_2Z_1Y_1^3Y_0+8Z_2^3Z_1Y_1^3Y_0-20Z_3Z_2Z_1^2Y_1^3Y_0+8Z_2Z_1^3Y_1^3Y_0+16Z_3Z_2Y_2^2Y_1^3Y_0-\\
112Z_2Z_1Y_2^2Y_1^3Y_0-32Z_3^2Y_2Y_1^4Y_0-16Z_2^2Y_2Y_1^4Y_0+54Z_3Z_1Y_2Y_1^4Y_0+4Z_1^2Y_2Y_1^4Y_0-12Y_2^3Y_1^4Y_0-2Z_3Z_2Y_1^5Y_0+8Z_2Z_1Y_1^5Y_0+2Y_2Y_1^6Y_0-96Z_3^4Z_2^2Y_0^2+192Z_3^3Z_2^2Z_1Y_0^2-96Z_3^2Z_2^2Z_1^2Y_0^2-128Z_3^3Z_2Y_2Y_1Y_0^2+288Z_3^2Z_2Z_1Y_2Y_1Y_0^2-192Z_3Z_2Z_1^2Y_2Y_1Y_0^2+32Z_2Z_1^3Y_2Y_1Y_0^2+28Z_3^2Z_2^2Y_1^2Y_0^2-Z_2^2Y_1^4Y_0^2+16Z_3^3Z_1Y_1^2Y_0^2-8Z_3Z_2^2Z_1Y_1^2Y_0^2-\\
32Z_3^2Z_1^2Y_1^2Y_0^2+4Z_2^2Z_1^2Y_1^2Y_0^2+16Z_3Z_1^3Y_1^2Y_0^2-32Z_3^2Y_2^2Y_1^2Y_0^2+64Z_3Z_1Y_2^2Y_1^2Y_0^2-32Z_1^2Y_2^2Y_1^2Y_0^2+\\
16Z_3Z_2Y_2Y_1^3Y_0^2+32Z_2Z_1Y_2Y_1^3Y_0^2+15Z_3^2Y_1^4Y_0^2-32Z_3Z_1Y_1^4Y_0^2+4Z_1^2Y_1^4Y_0^2+12Y_2^2Y_1^4Y_0^2-Y_1^6Y_0^2+16Z_3^2Y_2Y_1^2Y_0^3+32Z_3^5Z_2Z_1-32Z_3Z_1Y_2Y_1^2Y_0^3+16Z_1^2Y_2Y_1^2Y_0^3-4Y_2Y_1^4Y_0^3-16Z_3^6Z_2-16Z_3^4Z_2^3-64Z_3^3Z_2^3Z_1+16Z_3^4Z_2Z_1^2+96Z_3^2Z_2^3Z_1^2-64Z_3^3Z_2Z_1^3+64Z_3Z_2^3Z_1^3+32Z_3^2Z_2Z_1^4-32Z_3^4Z_2Y_2^2+128Z_3^3Z_2Z_1Y_2^2-192Z_3^2Z_2Z_1^2Y_2^2+128Z_3Z_2Z_1^3Y_2^2-32Z_2Z_1^4Y_2^2+8Z_3^5Y_2Y_1-48Z_3^4Z_1Y_2Y_1-192Z_3^2Z_2^2Z_1Y_2Y_1+104Z_3^3Z_1^2Y_2Y_1+288Z_3Z_2^2Z_1^2Y_2Y_1-96Z_3^2Z_1^3Y_2Y_1+64Z_2^2Z_1^3Y_2Y_1+32Z_3Z_1^4Y_2Y_1-4Z_3^4Z_2Y_1^2-44Z_3^2Z_2^3Y_1^2+72Z_3^3Z_2Z_1Y_1^2+64Z_3Z_2^3Z_1Y_1^2+4Z_2^3Z_1^2Y_1^2-108Z_3^2Z_2Z_1^2Y_1^2+24Z_3Z_2Z_1^3Y_1^2-16Z_3^2Z_2Y_2^2Y_1^2-64Z_3Z_2Z_1Y_2^2Y_1^2+176Z_2Z_1^2Y_2^2Y_1^2-12Z_3^3Y_2Y_1^3-44Z_3Z_2^2Y_2Y_1^3+\\
60Z_3^2Z_1Y_2Y_1^3+80Z_2^2Z_1Y_2Y_1^3-68Z_3Z_1^2Y_2Y_1^3+8Z_1^3Y_2Y_1^3-16Z_3Y_2^3Y_1^3+32Z_1Y_2^3Y_1^3+14Z_3^2Z_2Y_1^4-22Z_3Z_2Z_1Y_1^4+4Z_2Z_1^2Y_1^4+10Z_2Y_2^2Y_1^4+4Z_3Y_2Y_1^5-6Z_1Y_2Y_1^5-64Z_3^4Z_2Y_2Y_0+256Z_3^3Z_2Z_1Y_2Y_0-320Z_3^2Z_2Z_1^2Y_2Y_0+\\
128Z_3Z_2Z_1^3Y_2Y_0+104Z_3^3Z_2^2Y_1Y_0-240Z_3^2Z_2^2Z_1Y_1Y_0+56Z_3Z_2^2Z_1^2Y_1Y_0-32Z_3^3Y_2^2Y_1Y_0+128Z_3^2Z_1Y_2^2Y_1Y_0-160Z_3Z_1^2Y_2^2Y_1Y_0+64Z_1^3Y_2^2Y_1Y_0+152Z_3^2Z_2Y_2Y_1^2Y_0-336Z_3Z_2Z_1Y_2Y_1^2Y_0+72Z_2Z_1^2Y_2Y_1^2Y_0-12Z_3^3Y_1^3Y_0-16Z_3Z_2^2Y_1^3Y_0+8Z_3^2Z_1Y_1^3Y_0+4Z_2^2Z_1Y_1^3Y_0+16Z_3Z_1^2Y_1^3Y_0+16Z_3Y_2^2Y_1^3Y_0-48Z_1Y_2^2Y_1^3Y_0-14Z_2Y_2Y_1^4Y_0-2Z_3Y_1^5Y_0+4Z_1Y_1^5Y_0-64Z_3^4Z_2Y_0^2+128Z_3^3Z_2Z_1Y_0^2-64Z_3^2Z_2Z_1^2Y_0^2-32Z_3^3Y_2Y_1Y_0^2+64Z_3^2Z_1Y_2Y_1Y_0^2-\\
32Z_3Z_1^2Y_2Y_1Y_0^2+16Z_3^2Z_2Y_1^2Y_0^2+16Z_1Y_2Y_1^3Y_0^2-12Z_3^6-44Z_3^4Z_2^2+40Z_3^5Z_1-16Z_3^3Z_2^2Z_1-44Z_3^4Z_1^2+104Z_3^2Z_2^2Z_1^2+16Z_3^3Z_1^3+16Z_3Z_2^2Z_1^3-16Z_3^4Y_2^2+64Z_3^3Z_1Y_2^2-96Z_3^2Z_1^2Y_2^2+64Z_3Z_1^3Y_2^2-16Z_1^4Y_2^2-128Z_3^2Z_2Z_1Y_2Y_1+\\
192Z_3Z_2Z_1^2Y_2Y_1+16Z_2Z_1^3Y_2Y_1+4Z_3^4Y_1^2-40Z_3^2Z_2^2Y_1^2+4Z_3^3Z_1Y_1^2+56Z_3Z_2^2Z_1Y_1^2-12Z_3^2Z_1^2Y_1^2+56Z_1^2Y_2^2Y_1^2-32Z_3Z_1Y_2^2Y_1^2-40Z_3Z_2Y_2Y_1^3+60Z_2Z_1Y_2Y_1^3+3Z_3^2Y_1^4-4Z_3Z_1Y_1^4+3Y_2^2Y_1^4-16Z_3^4Y_2Y_0+64Z_3^3Z_1Y_2Y_0-80Z_3^2Z_1^2Y_2Y_0+32Z_3Z_1^3Y_2Y_0+48Z_3^3Z_2Y_1Y_0-104Z_3^2Z_2Z_1Y_1Y_0+16Z_3Z_2Z_1^2Y_1Y_0+44Z_3^2Y_2Y_1^2Y_0-88Z_3Z_1Y_2Y_1^2Y_0+16Z_1^2Y_2Y_1^2Y_0-4Z_3Z_2Y_1^3Y_0-4Y_2Y_1^4Y_0-16Z_3^4Y_0^2+32Z_3^3Z_1Y_0^2-16Z_3^2Z_1^2Y_0^2+4Z_3^2Y_1^2Y_0^2-40Z_3^4Z_2+32Z_3^3Z_2Z_1+32Z_3^2Z_2Z_1^2-32Z_3^2Z_1Y_2Y_1+48Z_3Z_1^2Y_2Y_1-12Z_3^2Z_2Y_1^2+16Z_3Z_2Z_1Y_1^2-12Z_3Y_2Y_1^3+16Z_1Y_2Y_1^3+8Z_3^3Y_1Y_0-16Z_3^2Z_1Y_1Y_0-12Z_3^4+16Z_3^3Z_1
$}

\normalsize
\section{Details on Experiments} \label{sec:SMexperiments}

Here, we present details about the experiments presented in \Cref{sec:experiments}.

\subsection{Pseudo Ground Truth for $d=0,\delta=1$}
Here, we explain an approach to find the pseudo ground-truth motion $R(x)$ of each image in sequence iPhone 3GS from \cite{DBLP:conf/cvpr/ForssenR10}, which is used in the real-world tests for $d=0,\delta=1$. We exploit the fact that, at the beginning of this sequence, the camera is static, \ie without RS effect. We furthermore assume that the position of the camera center remains constant over the capture of the sequence, and that the motion can be locally approximated with $d=0,\delta=1$.

To find the rotation $R(x)$ of an image $i$, we start with establishing point matches between this image and image $46$ \cite{DBLP:conf/iccv/LindenbergerSP23}, which was captured with a static camera. If point $u=[x \ y \ 1]^\top$ on image $i$ is in a correspondence to a point $u' \in \PP^2$ on image $46$, then it holds that $[u]_\times R(x) R u'=0$,
where $R \in \mathrm{SO}(3)$ is the rotation between frame $0$ and frame $i$ at scanline $x=0$. As $R(x)$ and $R$ have 3 DoF each, and every point correspondence fixes 2 DoF, this problem is balanced if we assume  $3$ point correspondences. This problem is minimal with $64$ solutions, verified with GB.
We find $R(x)$ with RANSAC~\cite{ransac} using this solver, and measure the reprojection error for every correspondence as
\begin{equation} \label{eq:pseudo_gt_error}
    \lVert \pi(u) - \pi(R(x)Ru') \rVert_2,
\end{equation}
where $\pi$ is a map that divides the vector with its last coordinate. We then refine the selected solution with a local optimization, minimizing the sum of squared reprojection errors \eqref{eq:pseudo_gt_error} computed over all inliers. While not as accurate as COLMAP \cite{schoenberger2016sfm}, this does support RS images, and we believe that it is more accurate than single-view solvers.

\subsection{Image Curve Detector}
Now, we describe the detector of image curves used in the experiments (\Cref{sec:experiments}). Examples of results are shown in \Cref{fig:teaser}. This detector is based on \cite{halir1998numerically} and modified to detect curves that arise from projecting 3D lines with a RS camera. The detector  proceeds as follows:

\begin{enumerate}
    \item Find edge points with the Canny detector \cite{DBLP:journals/pami/Canny86a} and for every pixel calculate the distance from the closest edge with distance transform.
    \item Fit curves into these edge points with sequential RANSAC \cite{multibody}.
    \item Within the RANSAC loop, sample 3 edge points and linearly fit a curve in the form $y = Z_2x^2+Z_1x+Z_0$ into the data.
    \item Calculate the score for each curve by iterating through every scanline with $x$ integer, calculating $y=Z_2x^2+Z_1x+Z_0$, and finding the distance $dt$ from the closest curve using the distance transform. If $dt<\tau$, increment the score by $\tau-dt$.
    \item After every RANSAC loop, select the curve with the highest score and refine it by refitting the curve into all its inliers (\ie edge points closer than the threshold $\tau$). Then, remove these inliers from the edge set, and recompute the distance transform.
\end{enumerate}
We repeat this step $20$ times, and thus we can obtain up to $20$ curves. Threshold $\tau$ is set to $2px$.
We tend to sample points close to each other to increase the probability that they will lie on the same curve.
We have also tried fitting the predicted curves \eg \eqref{eq:curveD0Delta1}, but we have discovered that on  real data, the parabola is already sufficiently accurate (see \Cref{fig:teaser} for examples), while the higher degree curves tend to overfit onto multiple curves at the same time.

\subsection{Reprojection error measurement}

Here, we discuss the measurement of the reprojection error, used in the RANSAC scheme in our real-world experiments (\Cref{sec:experiments}).
We consider three cases,  corresponding to the problems we solve in the real-world experiments: 1) $d=0,\delta=1$, 2) $d=1,\delta=0$+parallel, 3) $d=1,\delta=0$+parallel+coplanar. In all three cases, we assume that the projection is given by a collection of points $u_j \in \PP^2, j \in \{1,\ldots,N\}$ lying on the curve.

\medskip \noindent
$\boldsymbol{d=0,\delta=1:}$
In this case, the line is parameterized by the coefficient vector  $q \in \PP^2$ of the plane spanned by the line and the camera center at the origin . We estimate the rotation $R(x)$ with our solver. As $q$ has 2 DoF, we estimate it using 2 points $u_1=[x_1:y_1:1], u_N=[x_N:y_N:1]$. We proceed as follows:
\begin{equation}
    M_1 := [u_1]_{\times} [r(x_1)]_{\times} R(x_1), \quad M_N := [u_N]_{\times} [r(x_N)]_{\times} R(x_N).
\end{equation}
Let $m_1$, $m_N$ be the first rows of the matrices $M_1$, $M_N$. Then, $q$ is obtained as $q = m_1 \times m_n$.
It can be verified that a line parametrized by $q$ will project perfectly onto scanline $x_1$ at $u_1$ and onto scanline $x_N$ at $u_N$.

To measure the reprojection error, we iterate over all points $u_j, j \in \{2,\ldots,N-1\}$, get the projections $u_{j,proj}$ using \eqref{eq:lineImageAtTimeX}, and compute the reprojection error as
\begin{equation}
    \epsilon = \frac{\sum_{j=2}^{N-1} \lVert u_j - u_{j,proj}\lVert_2}{N-2}.
\end{equation}
We consider the line to be an inlier if $\epsilon < 0.4 px$.

\medskip
\noindent
$\boldsymbol{d=1,\delta=0,}$ \textbf{parallel:} Here, the line is parameterized by a vector $L = [0 \ L_{y} \ L_{z}] $ $  \in \RR^3$. We estimate the center motion $C(x)$ and the line direction $\Delta$ with our solver. $L$ has 2 DoF and we estimate it using 2 points $u_1=[x_1:y_1:1], u_N=[x_N:y_N:1]$, by solving the following system of linear equations:
\begin{equation}
        u_1^{\top} [\Delta]_{\times} L = u_1^{\top} [\Delta]_{\times} C(x_1), \quad u_N^{\top} [\Delta]_{\times} L = u_N^{\top} [\Delta]_{\times} C(x_N), \quad [1 \ 0 \ 0] L = 0.
\end{equation}
We measure the error $\epsilon$ as in the previous case. The inlier threshold is $1.0px$.

\medskip
\noindent
$\boldsymbol{d=1,\delta=0,}$ \textbf{parallel and coplanar:} In this case, the line is parameterized with a single scalar $\lambda \in \RR$. We estimate the center motion $C(x)$, the line direction $\Delta$, and $P_d$  with our solver. We further assume $P_0 = [0 \ 0 \ 1]^\top$. $\lambda$ has 1 DoF and we estimate it with a single point $u_c = [x_c : y_c : 1]$ with $c = \lceil N/2\rceil$:
\begin{equation}
    \lambda = \frac{u_c^\top [\Delta]_\times(C(x_c)-P_0)}{  u_c^\top [\Delta]_\times P_d  }.
\end{equation}
We then measure the error $\epsilon$ as in the previous cases. Since the line is now perfectly projected at one instead of two points, its reprojection errors are typically much higher, and thus we use a higher threshold, namely $7px$.

\begin{center}
    We congratulate the reader for getting this far. :)
\end{center}

\end{document}